\documentclass[letterpaper]{article} % DO NOT CHANGE THIS
\usepackage{aaai25}  % DO NOT CHANGE THIS
\usepackage{times}  % DO NOT CHANGE THIS
\usepackage{helvet}  % DO NOT CHANGE THIS
\usepackage{courier}  % DO NOT CHANGE THIS
\usepackage[hyphens]{url}  % DO NOT CHANGE THIS
\usepackage{graphicx} % DO NOT CHANGE THIS
\usepackage{natbib}  % DO NOT CHANGE THIS AND DO NOT ADD ANY OPTIONS TO IT
\usepackage{caption} % DO NOT CHANGE THIS AND DO NOT ADD ANY OPTIONS TO IT
\usepackage{algorithm}
\usepackage{algorithmic}

\usepackage{newfloat}
\usepackage{listings}
\DeclareCaptionStyle{ruled}{labelfont=normalfont,labelsep=colon,strut=off} % DO NOT CHANGE THIS
\floatstyle{ruled}
\newfloat{listing}{tb}{lst}{}
\floatname{listing}{Listing}
\usepackage[utf8]{inputenc}     % allow utf-8 input
\usepackage{url}                % simple URL typesetting
\usepackage{booktabs}           % professional-quality tables
\usepackage{amsfonts}           % blackboard math symbols
\usepackage{nicefrac}           % compact symbols for 1/2, etc.
\usepackage{microtype}          % microtypography
\usepackage[usenames,dvipsnames,svgnames,table,x11names]{xcolor}
\usepackage{multirow, multicol} % create multiple rows/columns in tables
\usepackage{graphicx}           % figures
\usepackage{subcaption}
\usepackage{amsmath}            % maths
\usepackage{amsthm}             % maths
\usepackage{bm}
\usepackage{titletoc}

\newtheorem{theorem}{Theorem}

\newtheorem{lemma}[theorem]{Lemma}
\newtheorem{proposition}[theorem]{Proposition}

\newtheorem{remark}[theorem]{Remark}

\newcommand{\btw}{\widetilde{\boldsymbol{w}}}
\newcommand{\bw}{\boldsymbol{w}}
\newcommand{\be}{\boldsymbol{\epsilon}}
\newcommand{\vmin}{\boldsymbol{v}_{\text{min}}}
\newcommand{\lmin}{\lambda_{\text{min}}}
\newcommand{\lmax}{\lambda_{\text{max}}}
\newcommand{\lratio}{\lambda_{\text{ratio}}}
\newcommand{\Lhead}{\mathcal{L}_{\text{head}}}
\newcommand{\Ltail}{\mathcal{L}_{\text{tail}}}
\newcommand{\ehead}{\boldsymbol{\epsilon}_{\text{head}}}
\newcommand{\etail}{\boldsymbol{\epsilon}_{\text{tail}}}
\newcommand{\rhead}{\rho_{\text{head}}}
\newcommand{\rtail}{\rho_{\text{tail}}}
\newcommand{\expe}{\mathbb{E}}
\newcommand{\li}{\ell_i(\bm{w})}

\title{SSE-SAM: Balancing Head and Tail Classes Gradually through Stage-Wise SAM}
\author {
    Xingyu Lyu\textsuperscript{\rm 1,\rm 2},
    Qianqian Xu\textsuperscript{\rm 1}\thanks{Corresponding authors.},
    Zhiyong Yang\textsuperscript{\rm 2},
    Shaojie Lyu\textsuperscript{\rm 3},
    Qingming Huang\textsuperscript{\rm 2, \rm 1, \rm 4}\footnotemark[1]
}
\affiliations {
    \textsuperscript{\rm 1}Key Lab. of Intelligent Information Processing, Institute of Computing Tech., CAS\\
    \textsuperscript{\rm 2}School of Computer Science and Tech., University of Chinese Academy of Sciences\\
    \textsuperscript{\rm 3}Tencent Corporate\\
    \textsuperscript{\rm 4}BDKM, University of Chinese Academy of Sciences\\
    lvxingyu20@mails.ucas.ac.cn, xuqianqian@ict.ac.cn, yangzhiyong21@ucas.ac.cn, shaojielv@tencent.com, qmhuang@ucas.ac.cn
}
\begin{document}

\maketitle

\begin{abstract}
Real-world datasets often exhibit a long-tailed distribution, where vast majority of classes known as tail classes have only few samples. Traditional methods tend to overfit on these tail classes. Recently, a new approach called Imbalanced SAM (ImbSAM) is proposed to leverage the generalization benefits of Sharpness-Aware Minimization (SAM) for long-tailed distributions. The main strategy is to merely enhance the smoothness of the loss function for tail classes. However, we argue that improving generalization in long-tail scenarios requires a careful balance between head and tail classes. We show that neither SAM nor ImbSAM alone can fully achieve this balance. For SAM, we prove that although it enhances the model's generalization ability by escaping saddle point in the overall loss landscape, it does not effectively address this for tail-class losses. Conversely, while ImbSAM is more effective at avoiding saddle points in tail classes, the head classes are trained insufficiently, resulting in significant performance drops. Based on these insights, we propose Stage-wise Saddle Escaping SAM (SSE-SAM), which uses complementary strengths of ImbSAM and SAM in a phased approach. Initially, SSE-SAM follows the majority sample to avoid saddle points of the head-class loss. During the later phase, it focuses on tail-classes to help them escape saddle points. Our experiments confirm that SSE-SAM has better ability in escaping saddles both on head and tail classes, and shows performance improvements.
\end{abstract}

\begin{links}
    \link{Code}{https://github.com/Zemdalk/SSE-SAM}
    % \link{Datasets}{https://aaai.org/example/datasets}
    % \link{Extended version}{https://aaai.org/example/extended-version}
\end{links}

% -------------------------------------------------------------------

\section{Introduction}
\label{intro}

Over the past few decades, deep neural networks have made significant progress in various fields, largely due to the use of well-curated datasets \cite{deng2009imagenet, krizhevsky2009learning}. However, real-world data often exhibit a long-tail distribution \cite{ouyang2016factors, zhang2021distribution, wei2023towards, li2024long, han2024aucseg, yang2024harnessing}, where a few dominant classes (head classes) hold the majority of samples and most other classes (tail classes) have very few. Modern over-parameterized neural networks tend to overfit on head classes and underperform on tail classes due to inadequate training for the latter \cite{buda2018systematic, van2017devil}.

Several methods have been developed to address this imbalance \cite{li2024size, yang2022optimizing, hou2022adauc}. Most of these approaches use re-weighting techniques, such as LDAM \cite{cao2019learning}, logit-adjustment \cite{menon2020long, wang2024unified}, and Vector Scaling (VS) Loss \cite{kini2021label}, to adjust class weights and focus on tail classes in the loss function. This adjustment significantly enhances model performance in datasets with long-tail distributions. However, these re-weighting techniques can still cause overfitting in tail classes \cite{hawkins2004problem}.

To address overfitting, many strategies involving smoothness regularization have been proposed \cite{kang2019decoupling, alshammari2022long}. Sharpness-Aware Minimization (SAM) \cite{foret2020sharpness}, known for its strong theoretical basis and excellent results, has garnered considerable attention. According to \citet{rangwani2022escaping}, SAM effectively helps models escape saddle points, heading for a flat local minimum. \citet{zhou2023imbsam} introduced Imbalanced SAM (ImbSAM) for long-tail learning, which bypasses the SAM optimization for head classes to prevent their dominance in the optimization process, showing notable performance improvements in long-tail scenarios. However, empirical studies indicate that ImbSAM significantly reduces the generalization ability of head classes compared to SAM. Thus, enhancing long-tail generalization requires a careful balance between head and tail classes.

To find a better solution, we first examine the properties of SAM and ImbSAM. We show that SAM theoretically has a greater component on the maximum negative curvature direction than Stochastic Gradient Descent (SGD), suggesting a stronger saddle escape capability. However, in long-tail distributions, head classes mainly influence the direction and magnitude of perturbations in SAM, often neglecting tail classes. In contrast, empirical studies show that ImbSAM reduces the misleading effects of head classes, helping tail classes escape saddles more effectively. However, it results in insufficient training of head classes.

\begin{figure*}[!t]
    \centering
    \includegraphics[width=0.90 \textwidth]{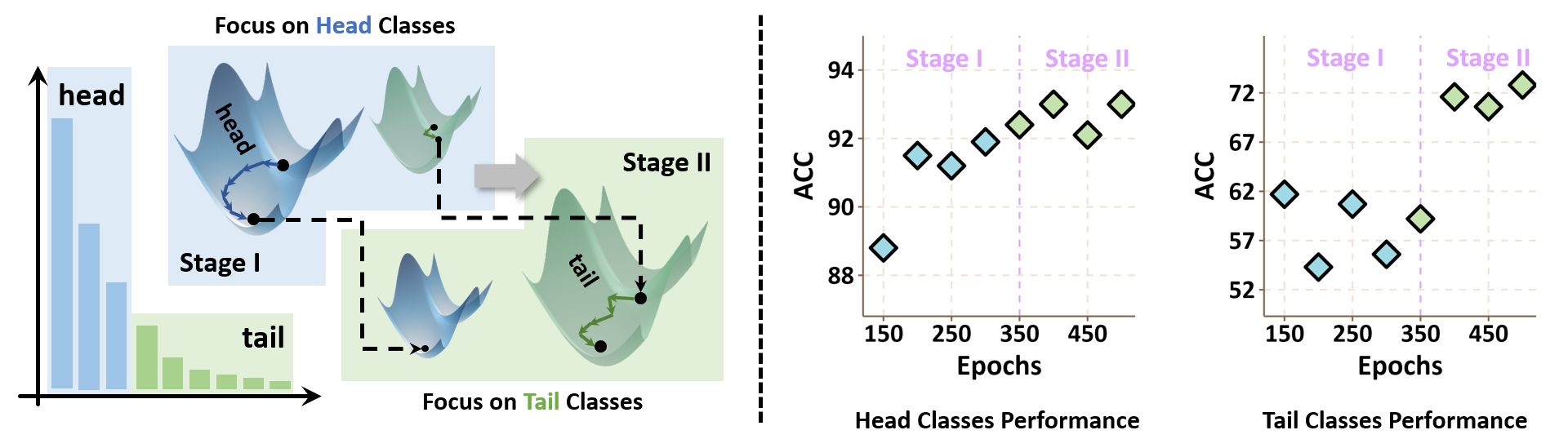}
    \caption{\textbf{Illustration of proposed SSE-SAM.} SSE-SAM splits the training process into two stages. The first stage helps head classes escape saddle points, while the second stage assists tail classes. The graphs on the right show an initial improvement in head class performance during the first stage, followed by a significant enhancement in tail class performance during the second stage. For detailed experimental results, see Sec.\ref{experiments}.}
    \label{fig:illustration}
\end{figure*}

From the analysis, SAM and ImbSAM offer complementary benefits, but it's challenging to escape saddle points in both head and tail classes simultaneously. To tackle this, we introduce a new approach called Stage-wise Saddle Escaping SAM (SSE-SAM), designed to gradually overcome saddle points in both head and tail classes. SSE-SAM divides the training into two phases. Initially, it employs the small loss assumption to help head classes escape saddle points effectively. Later, it shifts attention to tail classes to facilitate their escape from saddle points. Our experiments demonstrate that SSE-SAM is more effective at escaping saddles in both head and tail classes.

Our contributions in this paper are summarized as follows:

\begin{itemize}
    \item We theoretically show that SAM, compared with SGD, has a greater component along the maximum negative direction of the loss landscape. This finding demonstrates the more efficient saddle escaping ability of SAM and lays the foundation for the following analysis.
    \item We analyze the complementary characteristics of SAM and ImbSAM in saddle escaping ability. We observe the distinct effects of aiding escape from saddle points by SAM and ImbSAM, both theoretically and empirically.
    \item We propose a Stage-wise Saddle Escaping SAM (SSE-SAM) algorithm. SSE-SAM splits the training process into two stages of saddle escaping to achieve balance in head and tail classes. Our method achieves remarkable performance across multiple experiments.
\end{itemize}

% -------------------------------------------------------------------

\section{Related Work}
\label{related_work}

% In this section, we briefly discuss related work on long-tail learning methods, smoothness regularization techniques, and studies on escaping saddle points in loss landscapes. For more detailed discussion, please refer to App.\ref{app:related_work}. Finally, we explain the notation conventions in this paper.

In this section, we briefly discuss related work here. More detailed discussion can be found in Appendix. In the end, we explain the notation conventions in this paper.

\subsection{Long-tail Learning}

% Class re-balancing is a prevalent method for training on long-tailed datasets. This approach seeks to mitigate the adverse effects of tail classes through three main strategies: re-sampling, class-sensitive learning, and logit adjustment. Re-sampling involves either oversampling tail classes \cite{chawla2002smote} or undersampling head classes \cite{buda2018systematic} to achieve a more balanced dataset. Class-sensitive learning \cite{lin2017focal, cao2019learning, kini2021label} modifies the learning objectives for different classes to equalize training impact. Logit adjustment \cite{menon2020long, wang2024unified} addresses class imbalance by modifying the prediction logits of neural networks.

Class re-balancing techniques are mainstream methods in long-tail learning \cite{zhang2023deep}, including re-sampling \cite{chawla2002smote, buda2018systematic}, class-sensitive learning \cite{lin2017focal, cao2019learning, kini2021label, du2024probabilistic}, and logit adjustment \cite{menon2020long, wang2024unified}. These strategies aim to provide a more balanced training environment, enhancing the representation and performance of tail classes.

\subsection{Smoothness Regularization Methods}\label{sec:smoothness-reg}

Smoothness regularization algorithms promote convergence to flatter minima, which can significantly improve generalization \cite{keskar2016large, jiang2019fantastic}. A notable method within this domain is Sharpness-Aware Minimization (SAM) \cite{foret2020sharpness}, which has demonstrated notable success in guiding models toward smoother convergence and thus, better generalization. There are abundant theoretical \cite{woodworth2020kernel, andriushchenko2022towards, wen2022does, wen2023sharpness} and practical \cite{du2021efficient, kwon2021asam, mi2022make, zhou2023imbsam} researches surrounding SAM, underscoring its broad applicability.

\subsection{Escaping Saddle Points}
\label{sec:escaping-saddle}

Saddle points are regions in the loss landscape where the function's gradient is zero but the Hessian matrix is indefinite \cite{lee2018introduction}. Models converging to saddle points typically exhibit poor generalization \cite{dauphin2014identifying}. There are extensive theoretical studies on saddle point problem \cite{daneshmand2018escaping, jin2021nonconvex, rangwani2022escaping, hsieh2023riemannian}, and many algorithms have been developed to efficiently escape from saddle points \cite{palaniappan2016stochastic, jin2017escape, staib2019escaping, criscitiello2019efficiently, zhang2021escape, huang2022efficiently}.

\subsection{Notation}

In this paper, $\mathcal{L}$ denotes the loss function, $\boldsymbol{w}$ represents the weights, and $\mathbf{I}$ indicates the identity matrix. Unless otherwise specified , $\lVert \cdot \rVert$ refers to the Euclidean norm $\lVert \cdot \rVert_2$.

% -------------------------------------------------------------------

\section{Preliminary}\label{sec:pre}
\label{preliminary}

As we have mentioned, our method is a two-stage saddle escaping method based on SAM and ImbSAM algorithm. In this section, we first introduce the related concepts and algorithms.

% \textbf{Long-Tail Learning}. In long-tail learning tasks, the training set $S$ is extremely imbalanced. Define imbalance factor $\text{IF} = \frac{\max_k\lvert S_k\rvert}{\min_k\lvert S_k\rvert}$, where $S_k$ denotes the subset labeled $k$ in set $S$. We have that $\text{IF}\gg 1$. We further define $\eta_{\text{thres}}$ as the splitting threshold of head and tail classes. For any $(\boldsymbol{x}_i, y_i)\in S$, split training set $S$ into $S_{\text{head}}$ and $S_{\text{tail}}$ as per Eq. (~\ref{eq:eta-thres}),
% \begin{equation}\label{eq:eta-thres}
% \left\{
% \begin{aligned}
%     & (\boldsymbol{x}_i, y_i)\in S_{\text{head}}, & \text{ if } \lvert S_{y_i}\rvert >\eta_{\text{thres}}, \\
%     & (\boldsymbol{x}_i, y_i)\in S_{\text{tail}}, & \text{ if } \lvert S_{y_i}\rvert \leq\eta_{\text{thres}}. \\
% \end{aligned}
% \right.
% \end{equation}

% Long-tail learning follows the classical classification framework, which trains a neural network to classify inputs into corresponding category using cross-entropy (CE) loss function or its variants. The only difference is that the training dataset is highly imbalanced and shows a long-tail distribution, with few head classes possessing most samples while a large amount of tail classes lacking enough samples.

\subsection{Sharpness-Aware Minimization}

Sharpness-Aware Minimization (SAM) \cite{foret2020sharpness} aims to regularize sharpness of loss landscape by finding the highest loss value in the neighborhood of current weight $\boldsymbol{w}$ controlled by $\rho$, then minimize such value:
\begin{equation}
    \min_{\boldsymbol{w}}\max_{\lVert\boldsymbol{\epsilon}\rVert\leq\rho}\mathcal{L}(\boldsymbol{w}+\boldsymbol{\epsilon}).
\end{equation}

The solution of this minimax problem is approximated by first order Taylor expansion, which gives $\boldsymbol{\epsilon}\approx\rho\frac{\nabla \mathcal{L}(\boldsymbol{w})}{\lVert \nabla \mathcal{L}(\boldsymbol{w})\rVert}$. The SAM loss function is thus defined as follows:
\begin{equation}
    \mathcal{L}^{\text{SAM}}\left(\boldsymbol{w}\right) = \mathcal{L}\left(\boldsymbol{w}+\boldsymbol{\epsilon}\right) =\mathcal{L}\left(\boldsymbol{w}+\rho\frac{\nabla \mathcal{L}(\boldsymbol{w})}{\lVert \nabla \mathcal{L}(\boldsymbol{w})\rVert}\right).
\end{equation}

\subsection{Imbalanced SAM}
Imbalanced SAM (ImbSAM) is a variant of SAM tailored for long-tail distributions \cite{zhou2023imbsam}. It corrects the bias towards dominant classes found in SAM by allowing data from less represented classes to contribute more to optimization during training.

The loss function of SAM can be decomposed as 
\begin{equation}
    \mathcal{L}^{\text{SAM}}(\bw) = \mathcal{L}(\bw+\be) = \Lhead(\bw+\be)+\Ltail(\bw+\be).
\end{equation}
ImbSAM removes the perturbation term in $\Lhead$ to focus the sharpness minimization on tail classes. Consequently, the loss function of ImbSAM is structured as follows:
\begin{equation}\label{eq:imbsam}
\begin{aligned}
    \mathcal{L}^{\text{ImbSAM}}(\bw) & = \Lhead(\bw) + \Ltail(\bw+\etail) \\
    & = \Lhead(\bw) + \Ltail\left(\bw+\rho\frac{\nabla\Ltail(\bw)}{\lVert\nabla\Ltail(\bw)\rVert}\right).
\end{aligned}
\end{equation}

\subsection{Geometric Measurement of Escaping Saddle Points}

We use two methods to measure the ability of models in escaping saddles.

From theoretical aspect, to analyze the saddle escaping process, we have to analyze the direction of maximum negative curvature\footnote{Here "maximum" means the absolute value reaches maximum among all negative curvatures, \textit{i.e.}, the smallest curvature which is negative. This expression is used throughout this paper.} around saddles. Here we briefly summarize the differential geometry theories showing that the direction of the eigenvector corresponding to the smallest eigenvalue of Hessian around saddle point is the direction of maximum negative curvature.

Differential geometry suggests that the principal direction of maximum negative curvature in the loss landscape corresponds to the eigenvector associated with the most negative eigenvalue of the Weingarten matrix \cite{lee2018introduction}  $\mathbf{W}$:
\begin{equation}
    \mathbf{W} = \left(\mathbf{I}-\frac{\nabla \mathcal{L}\cdot \nabla \mathcal{L}^{\top}}{\alpha^2}\right)\frac{\nabla^2\mathcal{L}}{\alpha},
\end{equation}
where $\alpha := \sqrt{1+\lVert\nabla \mathcal{L}\rVert^2}$. Proof of such property can be found in Appendix.

Denote by $\mathcal{H}$ the Hessian of $\mathcal{L}$ at $\boldsymbol{w}_0$, $\nabla^2\mathcal{L}(\boldsymbol{w}_0)$. Also, denote $\lambda_{\text{min}}, \boldsymbol{v}_{\text{min}}$ as the most negative eigenvalue and its corresponding eigenvector of $\mathcal{H}$. Since the gradient $\nabla \mathcal{L}(\boldsymbol{w}_0)$ vanishes near saddle points, $\mathcal{H}$ effectively approximates $\mathbf{W}$. In this sense, $\mathcal{H}$ is a good approximation of $\mathbf{W}$. Therefore $\boldsymbol{v}_{\text{min}}$ can effectively represent the direction of maximum negative curvature of loss landscape.

From empirical aspect, we focus on how well our model escapes from saddles. Following previous works \cite{li2018visualizing, rangwani2022escaping}, we use $\lratio = \lvert \lmax/\lmin\rvert$ to measure the non-convexity of the loss landscape, where high $\lratio$ indicates convergence to points with less negative curvature, therefore away from saddle points.

\subsection{Roadmap}

% This paper is organized as follows: Sec.\ref{sam_analysis} and Sec.\ref{imbsam_analysis} provide a detailed analysis of the strengths and weaknesses of SAM and ImbSAM in escaping saddle points. This examination reveals their complementary dynamics, leading to the development of a stagewise approach in Sec.\ref{methodology}. Sec.\ref{experiments} then presents experimental evidence to validate the efficacy of our proposed SSE-SAM algorithm.

This paper is structured as follows. The next two sections provide a detailed analysis of the strengths and weaknesses of SAM and ImbSAM in escaping saddle points. This examination reveals their complementary dynamics, leading to the development of a stagewise approach, presented in the following section. Experimental evidence to validate the efficacy of our proposed SSE-SAM algorithm is then presented in the subsequent section. The final section of the paper offers concluding remarks. All proofs are deferred to Appendix.

% -------------------------------------------------------------------

\section{Analysis of SAM in Escaping Saddle Points}
\label{sam_analysis}

\subsection{Advantage: SAM Escapes from Saddles More Efficiently Compared with SGD}
As mentioned in previous discussions, the theoretical power of SAM comes partially from its ability to escape from saddle points and heading to a flat local/global minimum. Inspired by this, we first delve deeper into the theoretical properties of SAM's saddle point escaping ability.  

% Moving away from saddles and therefore taking the direction of maximum negative curvature is crucial in escaping saddles\cite{dauphin2014identifying}. Under Correlated Negative Curvature (CNC) assumption\cite{daneshmand2018escaping}, \citet{rangwani2022escaping} prove a theorem showing that SAM technique amplifies the gradient component along the negative curvature by a factor of $(1+\rho\lambda_{\text{min}})^2$. However, this analysis has its limitations at least in three-fold. First, their analysis is based on CNC assumption which assumes the existence of $\gamma >0$ such that $\mathbb{E}[\langle \vmin, \nabla L(\bw)\rangle^2]\geq \gamma$, but $\nabla L(\bw)$ is actually approximately $0$ because $\bw$ is around saddle point, rendering $\gamma$ close to $0$. Thus, the amplification is extremely small. Second, the CNC assumption is actually applies in the analysis of CNC-PGD algorithm instead of SGD to show its effectiveness in escaping saddles\cite{daneshmand2018escaping}. Although both CNC-PGD and SGD utilize the intrinsic noise, CNC-PGD still uses a varying learning rate dependent on $\gamma$, which makes their analysis ineffective in showing the effectiveness of SAM compared with SGD. Third, their analysis does not directly show the way SAM escapes saddles, which renders their analysis less intuitive.
Building upon the foundation laid in the preceding section on geometric measurement of escaping saddle points, we evaluate this capability by examining how close the learned weight is to the subspace associated with the most negative curvature. Moreover, we only need to examine it around the saddle point, where the Weingarten matrix roughly equals to the Hessian, \textit{i.e.}, $\mathbf{W} \approx \nabla^2\mathcal{L}$. The most negative curvature subspace is roughly the subspace spanned by the eigenvector associated with the minimum eigenvalue of the Hessian $\lmin$\footnote{$\lmin$ is mostly negative as shown in the experiments.}. In this sense, a greater component on this subspace implies a more effective direction for descent around the saddle. Therefore, we compare the update directions of SAM and ImbSAM by examining the inner product between these directions and the eigenvector. The result is shown in the following theorem.

\begin{theorem}\label{thm:projection-onto-negative-curvature}
    Let $\mathcal{W} := \left\{\bw_0, \bw_1, \cdots, \bw_t, \btw_1, \cdots, \btw_t\right\}$, where $\bw_0$ is the starting weight, $\bw_t$ denotes the weight obtained after optimizing for $t$ steps using SGD, $\btw_t$ denotes the weight obtained after optimizing for $t$ steps using SAM. All weights in $\mathcal{W}$ are around saddle point. Let $L := \max\limits_{\bw \in \mathcal{W}}\lVert \nabla \mathcal{L}(\bw)\rVert>0$. Given any $\mu\geq 1$, when $\lmin\leq -\frac{L}{\eta\rho}\left(\frac{1+\mu}{2}\eta+\sqrt{\left(\frac{1+\mu}{2}\right)^2\eta^2+\frac{\left(2+\mu\right)\eta\rho}{L}}\right)$, under second order Taylor approximation, we have
    \begin{equation}
    \begin{aligned}
        % \lvert\langle\vmin, \btw_t-\bw_0\rangle\rvert \geq \mu\lvert\langle\vmin, \bw_t-\bw_0\rangle\rvert + \left((\mu^2+\mu)t-\mu^2-2\right)\lvert\langle\vmin, \nabla L(\bw_0)\rangle\rvert.
        \lvert\langle \boldsymbol{v}_{\text{min}}, \widetilde{\boldsymbol{w}}_t-\boldsymbol{w}_0\rangle\rvert\geq\mu\lvert\langle \boldsymbol{v}_{\text{min}}, \boldsymbol{w}_t-\boldsymbol{w}_0\rangle\rvert  \\
        +\left((\mu^2+\mu)t-2\mu^2-\mu\right)\frac{\lvert\langle\vmin, \nabla \mathcal{L}(\bw_0)\rangle\rvert}{\lvert\lmin\rvert}.
    \end{aligned}
    \end{equation}
\end{theorem}

% Proof of Thm.\ref{thm:projection-onto-negative-curvature} is deferred to Appendix.

For all $\mu\geq 1$, when $t\geq 2$, we have that $(\mu^2+\mu)t-2\mu^2-\mu\geq \mu>0$. Therefore, SAM leads to a component on $\vmin$ at least $\mu$ times greater than SGD around the saddle point, showing a better potential to find a descent path toward a smooth and flat local minimum (SAM naturally punishes sharp local minimums through its objective function). Above all, SAM facilities the saddle point escaping ability of the optimization algorithm and leads to better generalization by sharpness regularization. 

% According to Theorem \ref{thm:projection-onto-negative-curvature}, SAM tends to project a greater component onto the direction of $\vmin$ compared to SGD when near the saddle point. This indicates a better potential for SAM to locate a descent path leading toward a smooth and flat local minimum. By inherently discouraging convergence to sharp local minima through its objective function, SAM enhances the model's generalization capabilities. Above all, SAM facilities the saddle point escaping ability of the optimization algorithm and leads to better generalization by sharpness regularization. 

Fig.\ref{fig:illu-thm-1} intuitively illustrates Thm.\ref{thm:projection-onto-negative-curvature}. Starting from a position close to a saddle point, SAM moves substantially more towards the direction of $\vmin$ than SGD, showing better saddle escaping capability.

\begin{figure}
    \centering
    \includegraphics[width=0.35\textwidth]{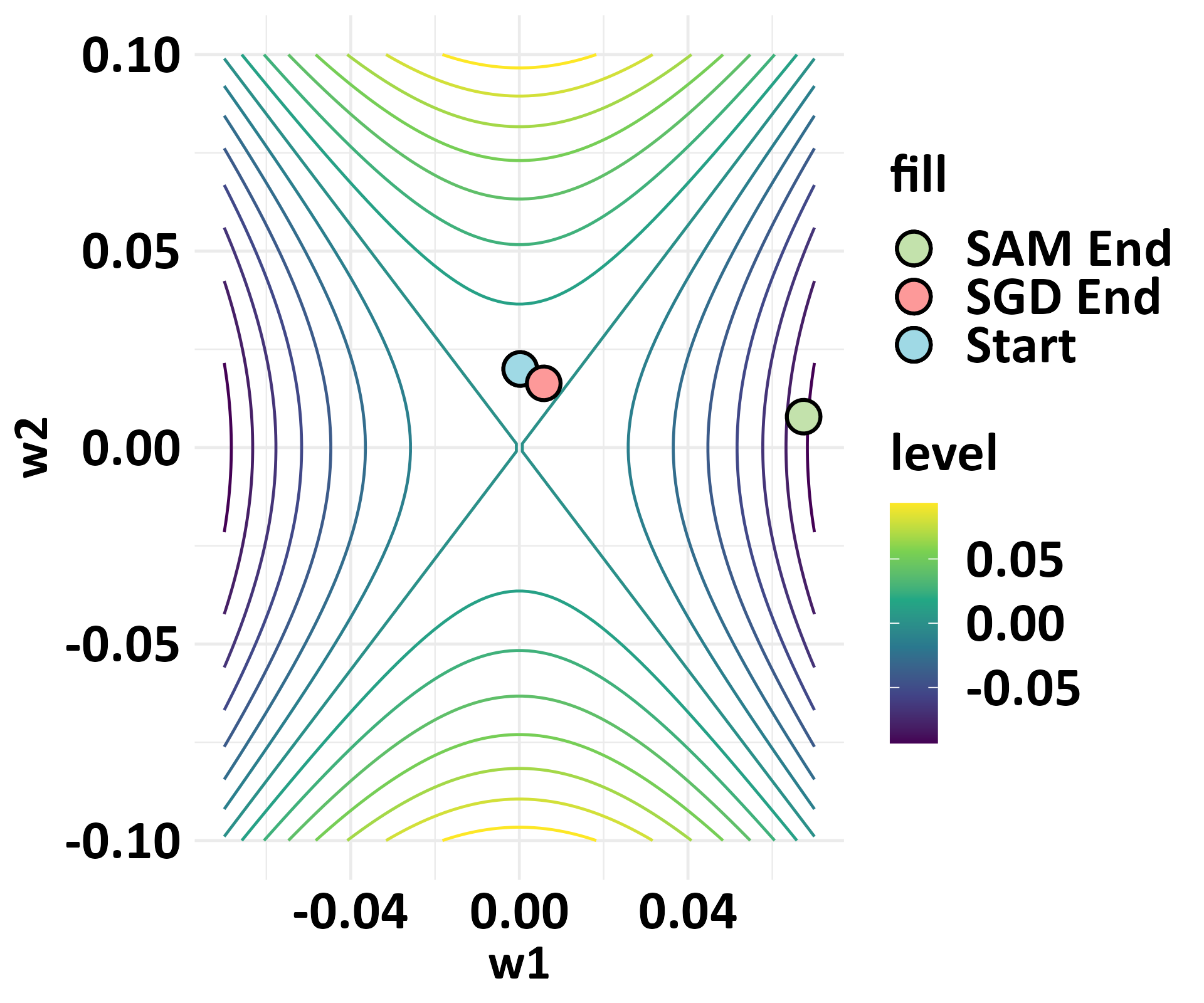}
    \caption{\textbf{Illustration of Thm.\ref{thm:projection-onto-negative-curvature}.}}
    \label{fig:illu-thm-1}
\end{figure}

\begin{remark}
    \citet{rangwani2022escaping} also establish a saddle point escaping theory for SAM. However, their result relies on the  Correlated Negative Curvature (CNC) assumption \cite{daneshmand2018escaping},  which assumes the existence of $\gamma >0$ such that $\mathbb{E}[\langle \vmin, \nabla L(\bw)\rangle^2]\geq \gamma$. If CNC holds, they show that SAM will increase the second moment $\mathbb{E}[\langle \vmin, \nabla L(\bw)\rangle^2]$. However, this assumption is not typically met near saddle points, where $\nabla L(\bm{w})$ approaches zero. Moreover, the CNC assumption is based on a baseline called CNC-PGD algorithm, which does not align with the settings in standard SGD \cite{daneshmand2018escaping}. Compared with this existing work, our result is much more general, since it not only get rid of the CNC assumption but also captures the dynamic trajectory after $t$ steps. As $t$ increases, the different behavior between SAM and SGD becomes increasingly apparent.  
\end{remark}

\subsection{Limitation: Suboptimal Optimization for Tail Classes in Escaping Saddles}
The aforementioned analysis is based on the overall loss landscape. When it comes to long-tail distribution, such theories are no longer validated. Given that head classes vastly outnumber tail classes in such datasets, the shape of the loss surface is predominantly influenced by head classes, effectively overshadowing the characteristics of tail classes which remain underexplored. \citet{zhou2023imbsam} empirically show that  while optimization process of SAM significantly benefits head classes, it leads to suboptimal optimization direction of tail classes due to the overwhelming existence of head classes. 

Delving deeper into this problem, we provide a formal theoretical analysis of such suboptimality.

Decomposing loss function of SAM as $\mathcal{L}^{\text{SAM}}(\bw) = \mathcal{L}(\bw+\be) = \Lhead(\bw+\be)+\Ltail(\bw+\be)$, we have
\begin{equation}\label{eq:decomp-of-epsilon}
\begin{aligned}
    \be  = \underbrace{\rho\cdot \frac{\nabla \Lhead(\bw)}{\lVert\nabla \mathcal{L}(\bw)\rVert}}_{\ehead} +
            \underbrace{\rho\cdot\frac{\nabla \Ltail(\bw)}{\lVert\nabla \mathcal{L}(\bw)\rVert}}_{\etail}.
\end{aligned}
\end{equation}

Because $\lvert S_{\text{head}}\rvert\gg \lvert S_{\text{tail}}\rvert$, we have 
\begin{equation}
    \Lhead = \sum_{(\boldsymbol{x}_i, y_i)\in S_{\text{head}}}\li \gg \Ltail = \sum_{(\boldsymbol{x}_i, y_i)\in S_{\text{tail}}}\li, 
\end{equation}
where $l_i$ is the loss for single sample $(\boldsymbol{x}_i,y_i)$. In fact, we have the following propositions:

\begin{proposition} \label{prop:scale-of-ehead-and-etail}
    Assume that the samples are drawn from i.i.d distribution, let $m=\lvert S_{\text{head}}\rvert, n=\lvert S_{\text{tail}}\rvert$
    then only if $\frac{\xi_{\text{head}}}{\xi_{\text{tail}}} = \Theta(\frac{n^2}{m^2})$, we have 
    \begin{equation}
        \frac{\mathbb{E}_{\text{head}}[\lVert\nabla \mathcal{L}(\bw)\rVert^2]}{\mathbb{E}_{\text{tail}}[\lVert\nabla \mathcal{L}(\bw)\rVert^2]} = \Theta(1).
    \end{equation}
    where 
    \begin{align*}
    &\xi_{\text{head}} = \|\bm{\mu}^{(1)}\|^2 + \sum_{j=1}^d \sigma_j^{(1)^2},~\xi_{\text{tail}} = \|\bm{\mu}^{(2)}\|^2 + \sum_{j=1}^d \sigma_j^{(2)^2},
    \end{align*}
$\bm{\mu}^{(1)},\bm{\mu}^{(2)}$ are the mean of instance gradient $\nabla \li$ for head/tail class, $\sigma_j^{(1)^2},\sigma_j^{(2)^2} $ are the variance of the j-th dimension of the instance gradient $\nabla \li$ for head/tail class.
\end{proposition}
For long-tail distribution,  we have $m^2 \gg n^2$, since it is almost impossible to observe comparable gradients for head and tail distribution. In this sense, we have $\ehead \gg \etail$ the optimization will overlook the update for the tail-class. Furthermore, the following proposition shows the closeness between $\ehead$ and $\epsilon$ in terms of their angles. 
\begin{proposition}\label{prop:direction-of-ehead-and-etail}
    Denote by $\theta_{\text{head}}$ the angle between $\ehead$ and $\boldsymbol{\epsilon}$, $\theta_{\text{tail}}$ the angle between $\etail$ and $\boldsymbol{\epsilon}$, $\psi$ the angle between $\ehead$ and $\etail$. When $\lVert\ehead\rVert\gg\lVert\etail\rVert$, we have
    \begin{equation}
    \begin{aligned}
        &\cos\theta_{\text{head}} \approx 1, \\
        &\cos\theta_{\text{tail}} \approx \cos\psi.
    \end{aligned}
    \end{equation}
\end{proposition}
 Therefore, both in terms of direction and magnitude, $\ehead$ is close to $\boldsymbol{\epsilon}$.  \textbf{The smoothness regularization direction for tail classes is dominated by head classes, leading to suboptimal regularization trajectory for tail classes}. Consequently, this results in less effective generalization enhancements for tail classes compared to head classes in SAM.

% -------------------------------------------------------------------

\section{Analysis of ImbSAM in Escaping Saddles}
\label{imbsam_analysis}

% \subsection{An Extended ImbSAM Algorithm} % include its expansion/extended version

% Inspired by the idea of ImbSAM, we propose an extended version of ImbSAM with the following form:
% \begin{equation}
% \begin{aligned}
%     L^{\text{ImbSAM}}(\bw) & = \Lhead(\bw+\ehead) + \Ltail(\bw+\etail) \\
%                            & = \Lhead\left(\bw+\rhead\frac{\nabla\Lhead(\bw)}{\lVert\nabla L(\bw)\rVert}\right) + \Ltail\left(\bw+\rtail\frac{\nabla\Ltail(\bw)}{\lVert\nabla L(\bw)\rVert}\right).
% \end{aligned}
% \end{equation}
% Here $\rhead$ and $\rtail$ are two hyperparameters.

% Note that we use $\lVert\nabla L(\bw)\rVert$ in denominator of $\ehead$ and $\etail$ instead of $\lVert\nabla\Lhead(\bw)\rVert$ and $\lVert\nabla\Ltail(\bw)\rVert$. Because it follows the form of original decomposition of $\boldsymbol{\epsilon}$ in Eq. (~\ref{eq:decomp-of-epsilon}). Furthermore, it has been observed that the denominator is not the reason of success of SAM\cite{andriushchenko2022towards, rangwani2022escaping}.

\subsection{Advantage: Efficiently Escaping Saddle Points on Tail Classes}

Previous section shows that SAM fails to assist tail classes in escaping saddle points. To address this issue,  ImbSAM propose a simple and effective solution, \textit{i.e.}, removing the perturbation term $\ehead$ in SAM, resulting in a new objective $\mathcal{L}^{\text{ImbSAM}}=\Lhead(\bw) + \Ltail\left(\bw+\rho\frac{\nabla\Ltail(\bw)}{\lVert\nabla\Ltail(\bw)\rVert}\right)$.

% As is described in Section ~\ref{sam_analysis}, SAM is instrumental in aiding models to escape saddle points. Under long-tail distribution, we have that $L^{\text{SAM}}(\bw) \approx \Lhead(\bw+\ehead)+\Ltail(\bw+\ehead)$. ImbSAM eliminates the perturbation term $\boldsymbol{\epsilon}$ in $\Lhead$ and uses $\etail$ in $\Ltail$ to specifically address tail classes. Thus, a logical inference is that ImbSAM effectively assists tail classes in escaping saddle points, thereby enhancing their generalization capabilities.

Previous studies have already shown the effectiveness of ImbSAM in long-tail distribution \cite{zhou2023imbsam}. The benefit of only punishing tail-class sharpness seems straightforward, because it directly eliminate the influence of the head class.  Our empirical study further confirms this inference. As is shown in Fig.\ref{fig:hessian_analysis}, ImbSAM has larger $\lratio$ on tail classes campared with SGD and SAM, rendering it better at escaping saddles on tail classes and improving generalization ability.

\subsection{Limitation: Insufficient Optimization of Head Classes}

% Empirical study shows ImbSAM does achieve better performance compared with SAM\cite{zhou2023imbsam}. However, there still exist problems in ImbSAM. \textit{ImbSAM's performance on head classes is worse than SAM due to insufficient training}, see Table ~\ref{tab:dataset-results}. Even though the significantly better generalization on tail classes contributes largely to the enhancement of overall generalization ability, model still has much space for enhancement of generalization ability of head classes. The extended ImbSAM algorithm tries to add $\rhead$ to avoid insufficient training on head classes, but does not achieve satisfying results. 

Although ImbSAM achieves better performance on tail classes, its effectiveness in head classes still needs validation. To this end, we record the head class performance in Tab.\ref{tab:cifar-results}, showing that ImbSAM underperforms SAM in these classes. This shortfall stems from inadequate optimization of head classes. Due to the elimination of $\ehead$ in head classes, ImbSAM actually has worse saddle escaping performance on head classes compared with SAM. Empirical Hessian analysis offers more general and intuitive illustration of the degenerate saddle escaping ability of ImbSAM. See Fig.\ref{fig:hessian_analysis}, where a higher $\lratio$ indicates better saddle escaping ability. ImbSAM works well on tail classes (Fig.\ref{fig:hessian_analysis_c}), but has unsatisfactory results on head classes (Fig.\ref{fig:hessian_analysis_b}) and on a global scale (Fig.\ref{fig:hessian_analysis_a}). Finally, based on our theoretical analysis in Thm.\ref{thm:projection-onto-negative-curvature}, we offer a counter-example that ImbSAM shows degenerate saddle point escaping capabilities.

\begin{proposition}\label{prop:imbsam-limitation}
    Consider a special case where $\mathcal{H}_{\text{tail}} = r\mathcal{H}, \mathcal{H}_{\text{head}}=(1-r)\mathcal{H}$ with $r=\frac{\lvert S_{\text{tail}}\rvert}{\lvert S\rvert}\approx 0$. Let $\bw_t'$ denote the weight obtained after optimizing for $t$ steps by ImbSAM. Other variables are defined the same as in Thm.\ref{thm:projection-onto-negative-curvature}. For any $\mu\geq 1$, we have that:

    % (1) when $\lmin\leq-\frac{L}{\eta\rho r}\left(\frac{1+\mu}{2}\eta+\sqrt{\left(\frac{1+\mu}{2}\right)^2\eta^2+\frac{(2+\mu)\eta\rho r}{L}}\right)$, it holds that $\lvert\langle \vmin,\bw_t'-\bw_0\rangle\rvert \geq \mu\lvert\langle\vmin, \bw_t-\bw_0\rangle\rvert$;

    (1) when $\lvert\langle \vmin,\bw_t'-\bw_0\rangle\rvert < \mu\lvert\langle\vmin, \bw_t-\bw_0\rangle\rvert$, as long as $\lmin\leq-\frac{L}{\eta\rho}\bigg(\frac{1+\mu}{2}\eta+$ $\sqrt{\left(\frac{1+\mu}{2}\right)^2\eta^2+\frac{(2+\mu)\eta\rho}{L}}\bigg)$, it holds that $\lvert\langle\vmin, \btw_t-\bw_0\rangle\rvert\geq \mu\lvert\langle\vmin,\bw_t-\bw_0\rangle\rvert$;

    (2) when $\lvert\langle \vmin,\btw_t-\bw_0\rangle\rvert < \mu\lvert\langle\vmin, \bw_t-\bw_0\rangle\rvert$, it also holds that $\lvert\langle \vmin,\bw_t'-\bw_0\rangle\rvert < \mu\lvert\langle\vmin, \bw_t-\bw_0\rangle\rvert$.
\end{proposition}

Prop.\ref{prop:imbsam-limitation}(1) states that, even when ImbSAM has worse saddle escaping ability, \textit{i.e.}, $\lvert\langle \vmin,\bw_t'-\bw_0\rangle\rvert < \mu\lvert\langle\vmin, \bw_t-\bw_0\rangle\rvert$, under mild conditions, SAM still can have better saddle escaping ability. However, according to Prop.\ref{prop:imbsam-limitation}(2), when SAM has worse saddle escaping ability, ImbSAM cannot have better saddle escaping ability as well. Collectively, these observations indicate that ImbSAM has a reduced ability to escape from saddles compared to SAM.

% {\color{red} To see this, consider a special case when $\mathcal{H}_{\text{tail}} = r\mathcal{H}, \mathcal{H}_{\text{head}} = (1-r)\mathcal{H}$, where $r=\frac{\lvert S_{\text{tail}}\rvert}{\lvert S\rvert}\approx 0$. Following the proof of Theorem ~\ref{thm:projection-onto-negative-curvature} provided in Appendix ~\ref{app:projection-onto-negative-curvature}, it holds that $\lvert\langle\vmin, \btw_t-\bw_0\rangle\rvert\geq \lvert\langle\vmin, \bw_t-\bw_0\rangle\rvert$ when $\lmin\leq -\frac{L}{\eta\rho r}\left(\eta+\sqrt{\eta^2+\frac{3\eta\rho r}{L}}\right)$. Here $\btw_t$ is the weights of ImbSAM after training for $t$ steps. However, as $r\to 0$, the upper bound $\lmin\leq -\frac{L}{\eta\rho r}\left(\eta+\sqrt{\eta^2+\frac{3\eta\rho r}{L}}\right)\to -\infty$, which makes ImbSAM almost impossible to escape from saddles globally. }

% -------------------------------------------------------------------

\section{Stagewise Saddle Escaping SAM (SSE-SAM)}
\label{methodology}

\subsection{Saddle Point Escaping and Generalization}

Before introducing the proposed algorithm, we first analyze the connection between saddle point escaping and generalization. We have the following theorem.

\begin{theorem}\label{thm:saddle-generalization}
    Suppose loss function $\mathcal{L}$ is upper bounded by $M$. For any $\rho>0$ and any distribution $\mathcal{D}$, with probability at least $1-\delta$ over the choice of the training set $S\sim\mathcal{D}$, there exists a constant $0\leq c\leq 1$ such that
    \begin{equation}
    \begin{aligned}
        & \mathcal{L}_{\mathcal{D}}(\bw)\leq \mathcal{L}_S(\bw) + \rho^2\sqrt{\frac{d}{4\pi}}\max_{\lVert\be\rVert\leq \rho}\lmax\left(\nabla^2L\left(\bw+c\be\right)\right) \\
        & + \frac{M}{\sqrt{n}} + \left[\left(\frac{1}{4}d\log \left(1+\frac{\lVert\bw\rVert^2(\sqrt{d}+\sqrt{\log n})^2}{d\rho^2}\right)+\frac{1}{4}\right.\right. \\
        & \left.\left.+\log\frac{n}{\delta} +2\log(6n+3d)\right)/\left(n-1\right)\right]^{\frac{1}{2}},
    \end{aligned}
    \end{equation}
    where $d$ is the number of parameters and $n$ is the size of training set $S$.
\end{theorem}

Note that the upper bound can be written as
\begin{equation}
\begin{aligned}
    & \underbrace{\mathcal{L}_S(\bw)}_{(A)} + \underbrace{\rho^2\sqrt{\frac{d}{4\pi}}\max_{\lVert\be\rVert\leq \rho}\lmax\left(\nabla^2\mathcal{L}\left(\bw+c\be\right)\right)}_{(B)} \\
    & + r\left(\frac{\lVert\bw\rVert^2}{\rho^2}, \Tilde{O}\left(\frac{1}{\sqrt{n}}\right)\right).
\end{aligned}
\end{equation}

Here $r$ is an increasing function on $\lVert\bw\rVert^2/\rho^2$ and is of order $\Tilde{O}(1/\sqrt{n})$. The success of SAM can be viewed as a two-stage process. By escaping from saddle point along the most negative eigenvector $\vmin$, $(A)$ is consistently minimized while $(B)$ remains stable. After escaping from saddles, SAM seeks for smooth local/global minima. \cite{wen2023sharpness} proves that SAM minimizes $\lmax$, which corresponds to term $(B)$. In all, SAM escapes from saddle point and finds smooth local minima, minimizing both $(A)$ and $(B)$, which greatly improves generalization. This result bridges the gap between escaping from saddle points and improving model generalization, and dynamically analyzes the behavior of SAM algorithm, which lays foundation for the following algorithm.

\subsection{Stagewise Saddle Escaping SAM (SSE-SAM) Algorithm}

Previous sections analyze the complementary characteristics of SAM and ImbSAM. To sum up, SAM excels at saddle points escaping in the overall loss landscape, but fails in helping tail classes to do so. Conversely, while ImbSAM does focus on helping tail classes escape saddles, the elimination of  head classes perturbation also hurts the generalization ability of head classes.

Therefore, an important question arises: \textit{can we combine the advantages of both while avoiding their drawbacks?} Integrating these approaches seems challenging due to the apparent contradictions in saddle escaping abilities between SAM and ImbSAM. Nonetheless, recognizing that it is difficult to achieve it  under a single objective, we propose a two-stage strategy tailored to exploit the strengths of each method while compensating for their weaknesses: 

(1) \textbf{In the first stage, we apply SAM with a focus on head classes to ensure a rapid and effective escape from saddle points}. According to small loss assumption \cite{xia2021sample}, the\textbf{ head-classes}, often bear way smaller loss value, can naturally follow the trajectory to \textbf{escape the saddle points}. Meanwhile, increasing $\rho$ lowers the upper bound for $\lmin$ in Thm.\ref{thm:projection-onto-negative-curvature}, facilitating an easier saddle escape. Hence, we construct $\rho$ and $\epsilon$ separately for head- and tail-classes. By increasing $\rtail$, it is easier for tail-class to escape saddle points. Overall, the SSE-SAM objective for the first stage is:
\begin{equation}
    \mathcal{L}^{\text{SSE-SAM}}_1(\bw) = \Lhead(\bw+\ehead) + \Ltail(\bw+\etail)
\end{equation}
where $\ehead = \rhead\frac{\nabla\Lhead(\bw)}{\lVert\nabla L(\bw)\rVert}, \etail = \rhead\frac{\nabla\Ltail(\bw)}{\lVert\nabla L(\bw)\rVert}$.

(2) \textbf{In the second stage, we resort to the ImbSAM algorithm to further assist tail classes to escape saddles}. The SSE-SAM objective for the second stage is:
\begin{equation}
    \mathcal{L}^{\text{SSE-SAM}}_2(\bw) = \Lhead(\bw) + \Ltail(\bw+\etail)
\end{equation}
where $\etail = \rhead\frac{\nabla\Ltail(\bw)}{\lVert\nabla L(\bw)\rVert}$.

(3) \textbf{We introduce a hyperparameter $\gamma$ (where $0<\gamma<1$) to mark the transition point between the two stages}. Specifically, if the total training duration is $T$ epochs, the model will switch from the first stage objective to the second stage objective after $\gamma T$ epochs.

For further details and pseudo-code, please refer to Appendix. Note that when $\gamma=1$, SSE-SAM reverts to the SAM algorithm; when $\gamma=0$, SSE-SAM corresponds to the ImbSAM algorithm. This demonstrates that $\gamma$ serves as a pivotal factor in balancing the synergistic benefits of both SAM and ImbSAM.

% -------------------------------------------------------------------

\section{Experiments}
\label{experiments}

In this section, we first present our experimental results, demonstrating the remarkable performance of our method. Next, we provide a Hessian analysis to compare saddle escaping abilities, which aligns with our theoretical analysis. More additional experiment details are presented in Appendix.

\subsection{Results}

Our experiment results on CIFAR-100-LT and CIFAR-10-LT are presented in Tab.\ref{tab:cifar-results}. All hyperparameters are set according to the model with the highest \textit{overall accuracy}: $\rho=0.05$ for SAM, $\rho=0.10$ for ImbSAM, and $\rhead=0.05, \rtail=0.10, \gamma=0.70$ for SSE-SAM. SSE-SAM achieves the highest overall accuracy among all methods and baselines. We also run our model on large scale dataset ImageNet, where we set $\rho=0.05$ for SAM, $\rho=0.15$ for ImbSAM and $\rhead=0.05,\rtail=0.15,\gamma=0.80$ for SSE-SAM. The results can be found in Tab.\ref{tab:imagenet-results}.

To further demonstrate the improved trade-off effect between head and tail classes, we trained two groups of models. Each group consists of three models: SAM (ext.)\footnote{We use SAM (ext.) to denote models trained with the extended version of SAM, \textit{i.e.}, the first stage objective in SSE-SAM.}, ImbSAM, and our SSE-SAM. SSE-SAM and SAM share $\rhead$, and all three models share $\rtail$. The results are presented in Tab.\ref{tab:trade-off}. In each group, the evaluation results for Many class accuracy generally follow the pattern \textbf{``SAM $>$ SSE-SAM $>$ ImbSAM"}, indicating that SSE-SAM indeed reduces the insufficient training of head classes to some extent, aligning with our expectations. Meanwhile, Medium and Few class accuracy typically adhere to the pattern \textbf{``SSE-SAM $>$ ImbSAM $>$ SAM"}, suggesting that SSE-SAM more effectively escapes saddle points in tail classes. Additionally, it is notable that SSE-SAM generally outperforms ImbSAM in tail classes, exceeding our expectations.

\begin{table}
  \centering
  \setlength{\tabcolsep}{0.7mm} 
  \begin{tabular}{@{}lcccccccc@{}}
    \toprule
     & \multicolumn{4}{c}{CIFAR-100-LT} & \multicolumn{4}{c}{CIFAR-10-LT} \\ \cmidrule(l){2-5}\cmidrule(l){6-9} 
     & \multicolumn{1}{l}{Many} & \multicolumn{1}{l}{Med.} & \multicolumn{1}{l}{Few} & \multicolumn{1}{l}{Avg} & \multicolumn{1}{l}{Many} & \multicolumn{1}{l}{Med.} & \multicolumn{1}{l}{Few} & \multicolumn{1}{l}{Avg} \\ \midrule
    CE & 74.3 & 43.1 & 11.6 & 44.6 & 92.3 & 78.6 & 54.8 & 77.0 \\
    +SAM & \textbf{75.3} & 45.6 & 12.9 & 46.2 & \textbf{93.2} & 79.8 & 60.4 & 79.4 \\
    +ImbSAM & 72.3 & 50.4 & 16.6 & 47.9 & 92.2 & 79.2 & 65.8 & 80.4 \\
    % \rowcolor{rcolor!80}
    +SSE-SAM & 71.7 & \textbf{54.5} & \textbf{17.1} & \textbf{49.3} & 93.0 & \textbf{80.5} & \textbf{72.8} & \textbf{83.2} \\ \midrule
    LDAM & 74.3 & 48.4 & 17.7 & 48.3 & 90.8 & 76.0 & 61.5 & 77.6 \\
    +SAM & \textbf{76.3} & 49.6 & 19.4 & 49.9 & \textbf{91.2} & 77.4 & 62.6 & 78.5 \\
    +ImbSAM & 72.0 & \textbf{55.4} & \textbf{23.3} & \textbf{51.6} & 89.1 & 78.9 & \textbf{71.7} & 80.8 \\
    % \rowcolor{rcolor!80}
    +SSE-SAM & 72.3 & 54.7 & 21.5 & 50.9 & 91.1 & \textbf{81.8} & 69.7 & \textbf{81.9} \\ \midrule
    LA & 75.0 & 45.7 & 16.2 & 47.1 & 91.8 & 80.0 & 65.4 & 80.4 \\
    +SAM & \textbf{75.2} & 49.9 & 21.2 & 50.2 & 92.3 & 80.8 & 73.3 & 83.1 \\
    +ImbSAM & 68.7 & 54.1 & 23.5 & 50.0 & 91.1 & 80.5 & 74.0 & 82.8 \\
    % \rowcolor{rcolor!80}
    +SSE-SAM & 68.8 & \textbf{58.7} & \textbf{28.6} & \textbf{53.2} & \textbf{92.5} & \textbf{81.1} & \textbf{79.2} & \textbf{85.1} \\ \midrule
    VS & \textbf{73.5} & 46.7 & 19.3 & 47.8 & \textbf{92.0} & 80.1 & 65.1 & 80.4 \\
    +SAM & 73.0 & 53.9 & 24.6 & 51.8 & 89.9 & 81.2 & \textbf{81.9} & 84.9 \\
    +ImbSAM & 68.9 & 55.2 & 25.4 & 51.0 & 90.9 & 80.0 & 76.6 & 83.3 \\
    % \rowcolor{rcolor!80}
    +SSE-SAM & 69.7 & \textbf{57.2} & \textbf{30.5} & \textbf{53.5} & 91.6 & \textbf{82.6} & 79.0 & \textbf{85.1} \\ \bottomrule
  \end{tabular}
  \caption{\textbf{Comparison of overall and split accuracy (\%) on CIFAR-100-LT and CIFAR-10-LT.} `Med.' denotes Medium classes, `Avg' denotes overall accuracy. We set $\text{IF}=100$.}
  \label{tab:cifar-results}
\end{table}

\begin{table}
  \centering
  \setlength{\tabcolsep}{1mm} 
  \begin{tabular}{lllll}
    \hline
     & Many & Med. & Few & Avg. \\ \hline
    CB & 36.9 & 32.7 & 16.8 & 33.2 \\
    $\tau$-norm & 59.1 & 46.9 & 30.7 & 49.4 \\
    cRT & 62.5 & 47.4 & 29.5 & 50.3 \\
    LWS & 61.8 & 48.6 & 33.5 & 51.2 \\
    DisAlign & 61.3 & 52.2 & 31.4 & 52.9 \\
    DRO-LT & 64.0 & 49.8 & 33.1 & 53.5 \\ \hline
    CE & 54.8 & 43.8 & 25.7 & 45.1 \\
    CE+SAM & 54.9 & 44.8 & 28.4 & 46.0 \\
    CE+ImbSAM & 54.9 & 44.6 & 29.2 & 46.1 \\
    CE+SSE-SAM & \textbf{56.0} & \textbf{46.0} & \textbf{31.2} & \textbf{47.5} \\ \hline
  \end{tabular}
  % \caption{\textbf{Comparison of overall and split accuracy (\%) on ImageNet-LT.} We set $\rho=0.05$ for SAM, $\rho=0.15$ for ImbSAM and $\rhead=0.05,\rtail=0.15,\gamma=0.80$ for SSE-SAM.}
  \caption{\textbf{Comparison of overall and split accuracy (\%) on ImageNet-LT.}}
  \label{tab:imagenet-results}
\end{table}

Finally, we delve into the dynamic training process of SAM, ImbSAM, and SSE-SAM. The results are presented in Fig.\ref{fig:acc-100}. The data indicates that after transitional point, accuracy for medium and few classes in SSE-SAM exceeds that of ImbSAM, while accuracy for many classes remains between that observed with SAM and ImbSAM. Overall, all-class accuracy of SSE-SAM experiences further improvement upon entering the second stage of training, solidifying the superior performance of SSE-SAM relative to other methods.

\begin{table*}[!htbp]
  \centering
  \begin{tabular}{@{}lllcccccccc@{}}
    \toprule
    \multirow{2}{*}{Models} & \multirow{2}{*}{$\rhead$} & \multirow{2}{*}{$\rtail$} & \multicolumn{4}{c}{IF=100} & \multicolumn{4}{c}{IF=50} \\ \cmidrule(l){4-11} 
     &  &  & \multicolumn{1}{l}{Many} & \multicolumn{1}{l}{Med.} & \multicolumn{1}{l}{Few} & \multicolumn{1}{l}{Avg} & \multicolumn{1}{l}{Many} & \multicolumn{1}{l}{Med.} & \multicolumn{1}{l}{Few} & \multicolumn{1}{l}{Avg} \\ \midrule
    SGD & 0.00 & 0.00 & 74.3 & 43.1 & 11.6 & 44.6 & 73.6 & 41.4 & 15.4 & 49.9 \\ \midrule
    SAM & 0.05 & 0.05 & \textbf{75.3} & 45.6 & 12.9 & 46.2 & \textbf{74.7} & 43.9 & 19.1 & 52.1 \\
    ImbSAM & 0.00 & 0.05 & 73.0 & {\underline{50.1}} & {\underline{15.2}} & {\underline{47.6}} & 71.4 & {\underline{48.5}} & \textbf{23.2} & {\underline{53.3}} \\
    % \rowcolor{rcolor!80}
    SSE-SAM & 0.05 & 0.05 & {\underline{73.1}} & \textbf{51.3} & \textbf{15.8} & \textbf{48.3} & {\underline{73.4}} & \textbf{49.7} & {\underline{22.4}} & \textbf{54.5} \\ \midrule
    SAM (ext.) & 0.05 & 0.10 & \textbf{75.2} & 48.6 & 14.9 & 47.8 & \textbf{74.4} & 46.4 & 21.3 & 53.4 \\
    ImbSAM & 0.00 & 0.10 & {\underline{72.3}} & {\underline{50.4}} & {\underline{16.6}} & {\underline{47.9}} & 70.9 & {\underline{50.2}} & {\underline{24.1}} & {\underline{54.0}} \\
    % \rowcolor{rcolor!80}
    SSE-SAM & 0.05 & 0.10 & 71.7 & \textbf{54.5} & \textbf{17.1} & \textbf{49.3} & {\underline{72.2}} & \textbf{51.9} & \textbf{26.1} & \textbf{55.6} \\ \bottomrule
  \end{tabular}
  \caption{\textbf{Comparison of overall and split accuracy (\%) of different IFs and $\rho$'s.} We set $\gamma=0.70$ for SSE-SAM models. We use CE as loss function and train all models on CIFAR-100-LT. }
  \label{tab:trade-off}
\end{table*}

\begin{figure*}[htbp!]
    \centering
    \begin{subfigure}{0.20\textwidth}
        \centering
        \includegraphics[width=\textwidth]{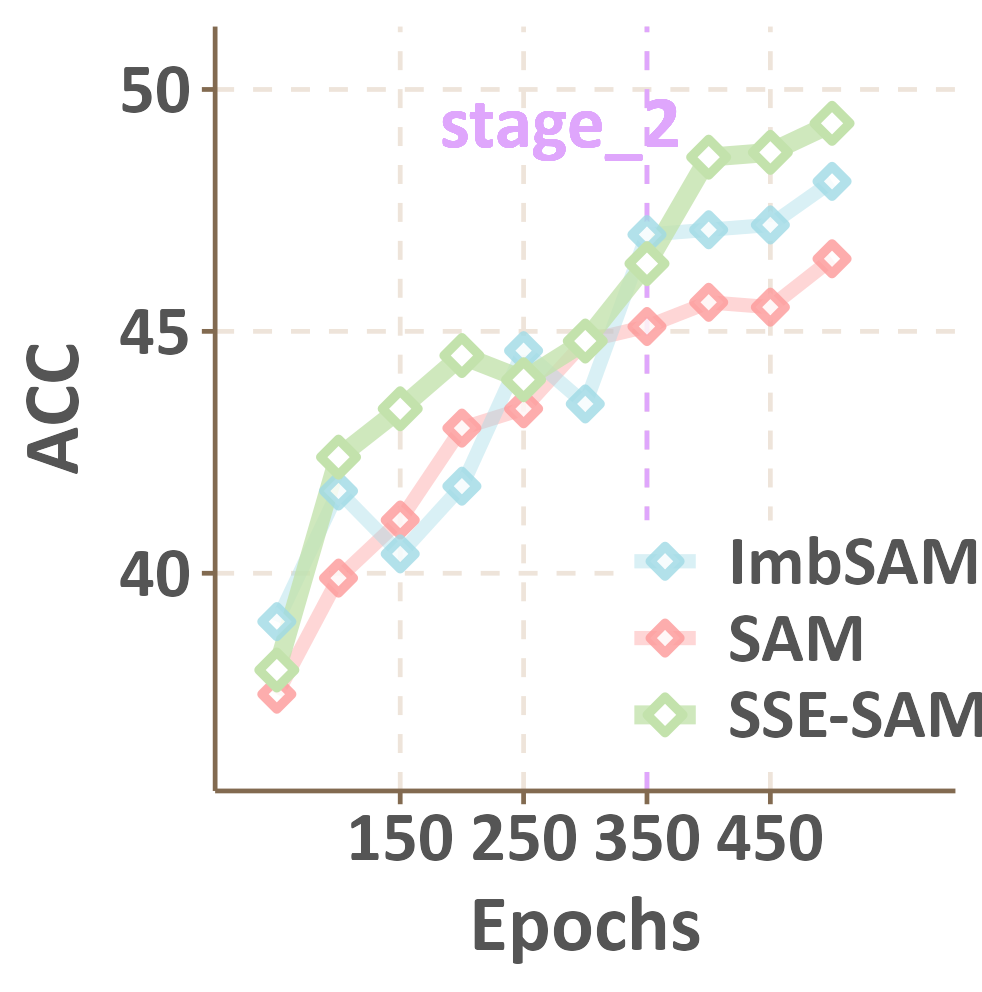}
        \caption{Overall}
        \label{fig:acc-all-100}
    \end{subfigure}
    \hfill
    \begin{subfigure}{0.20\textwidth}
        \centering
        \includegraphics[width=\textwidth]{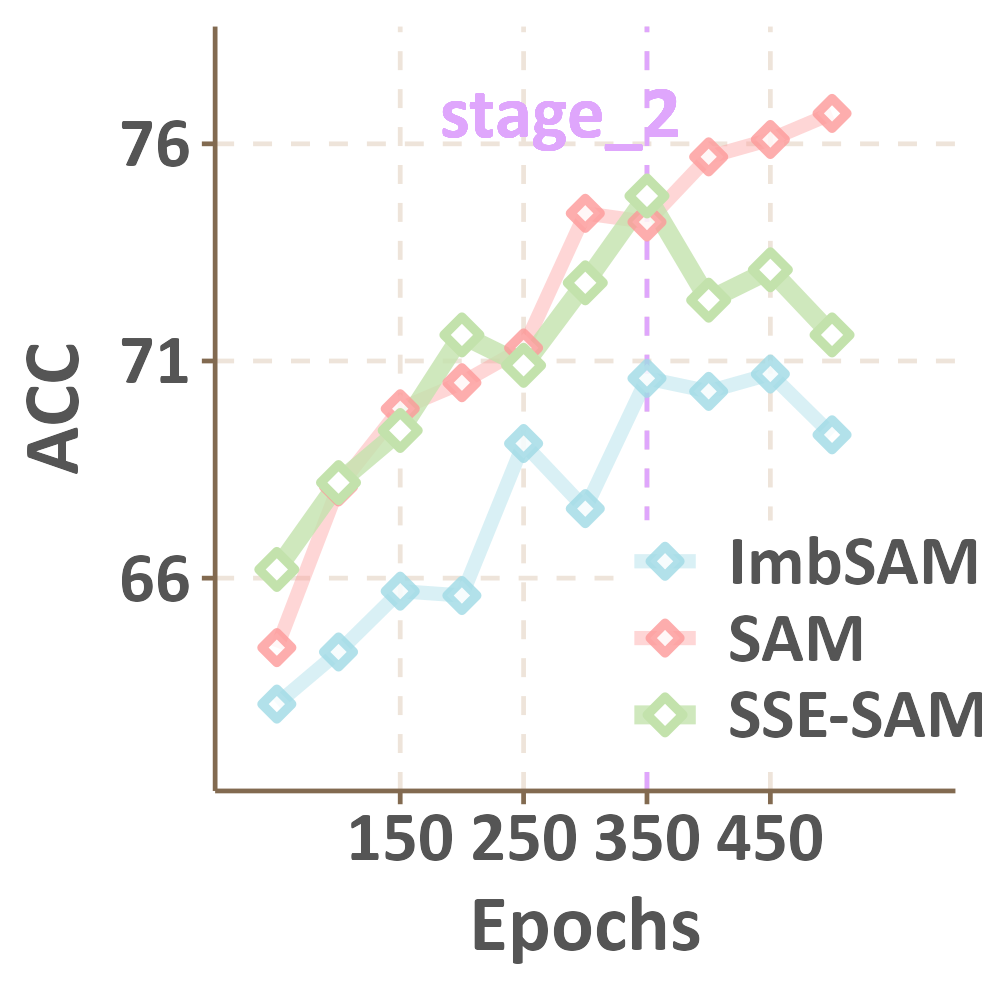}
        \caption{Many classes}
        \label{fig:acc-many-100}
    \end{subfigure}
    \hfill
    \begin{subfigure}{0.20\textwidth}
        \centering
        \includegraphics[width=\textwidth]{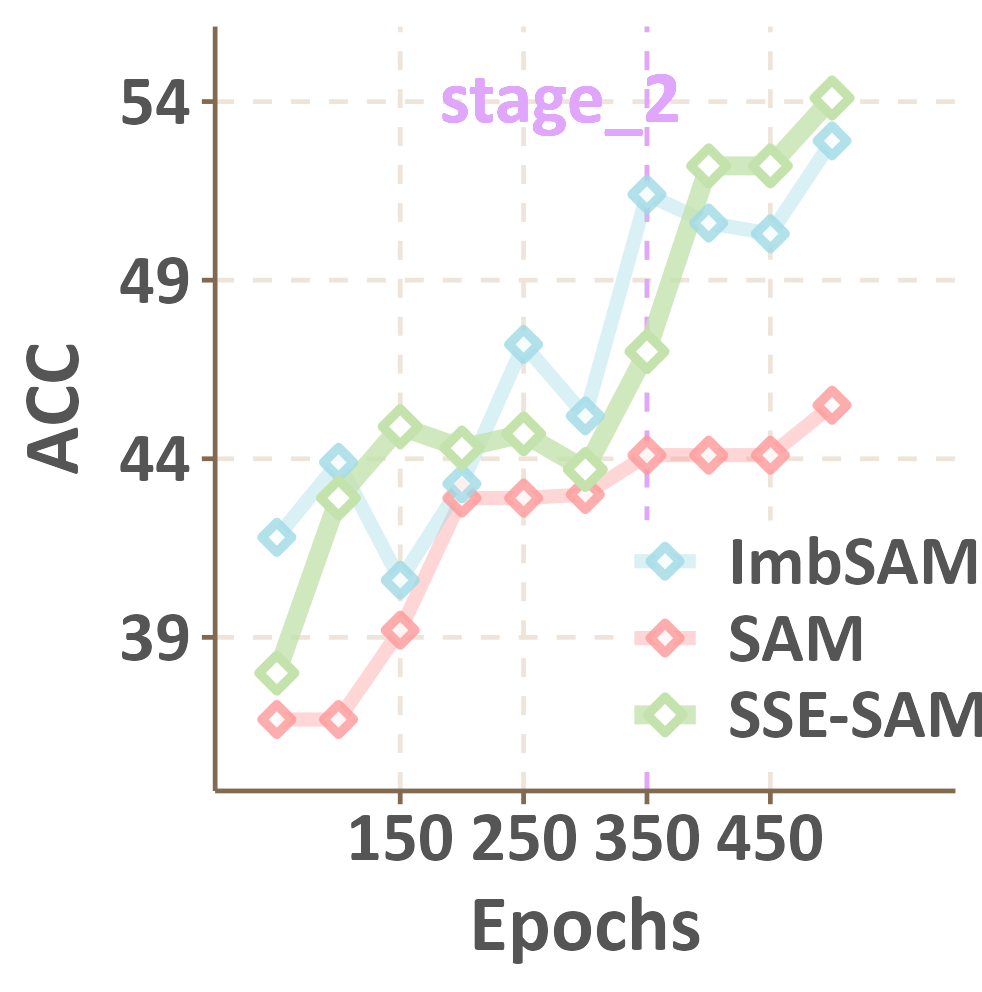}
        \caption{Medium classes}
        \label{fig:acc-med-100}
    \end{subfigure}
    \hfill
    \begin{subfigure}{0.20\textwidth}
        \centering
        \includegraphics[width=\textwidth]{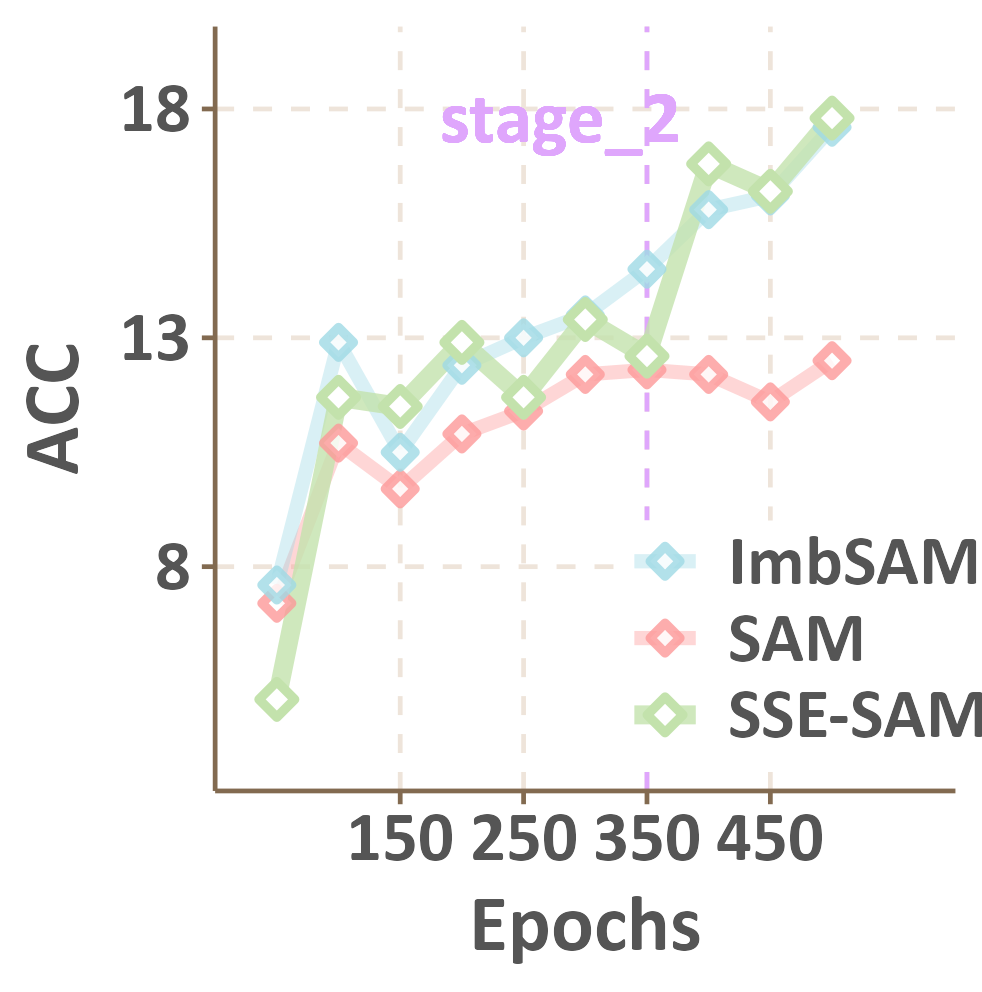}
        \caption{Few classes}
        \label{fig:acc-few-100}
    \end{subfigure}
    \caption{\textbf{Test accuracy on CIFAR-100-LT across different classes.} } % The results highlight the efficacy of our SSE-SAM method in enhancing accuracy for Medium and Few classes without significantly compromising performance in Many classes, thereby achieving optimal overall accuracy.
    \label{fig:acc-100}
\end{figure*}

\subsection{Hessian Analysis}

\begin{figure}[!t]
    \centering
    \begin{subfigure}{0.23 \textwidth}
      \includegraphics[width=\textwidth]{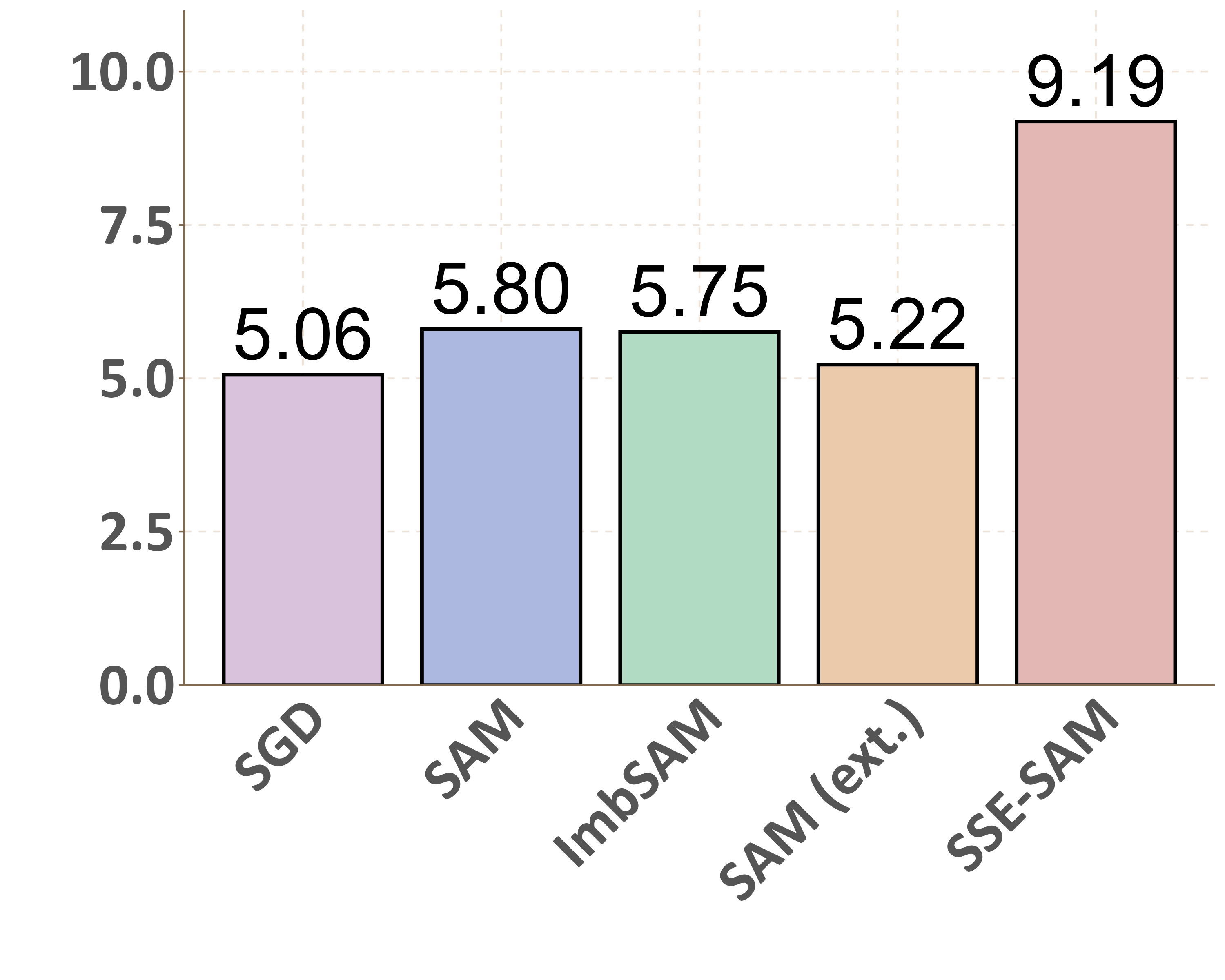}
      \caption{All classes}
      \label{fig:hessian_analysis_a}
    \end{subfigure}%
    % ~% add desired spacing
    \hfill
    \begin{subfigure}{0.23 \textwidth}
      \includegraphics[width=\textwidth]{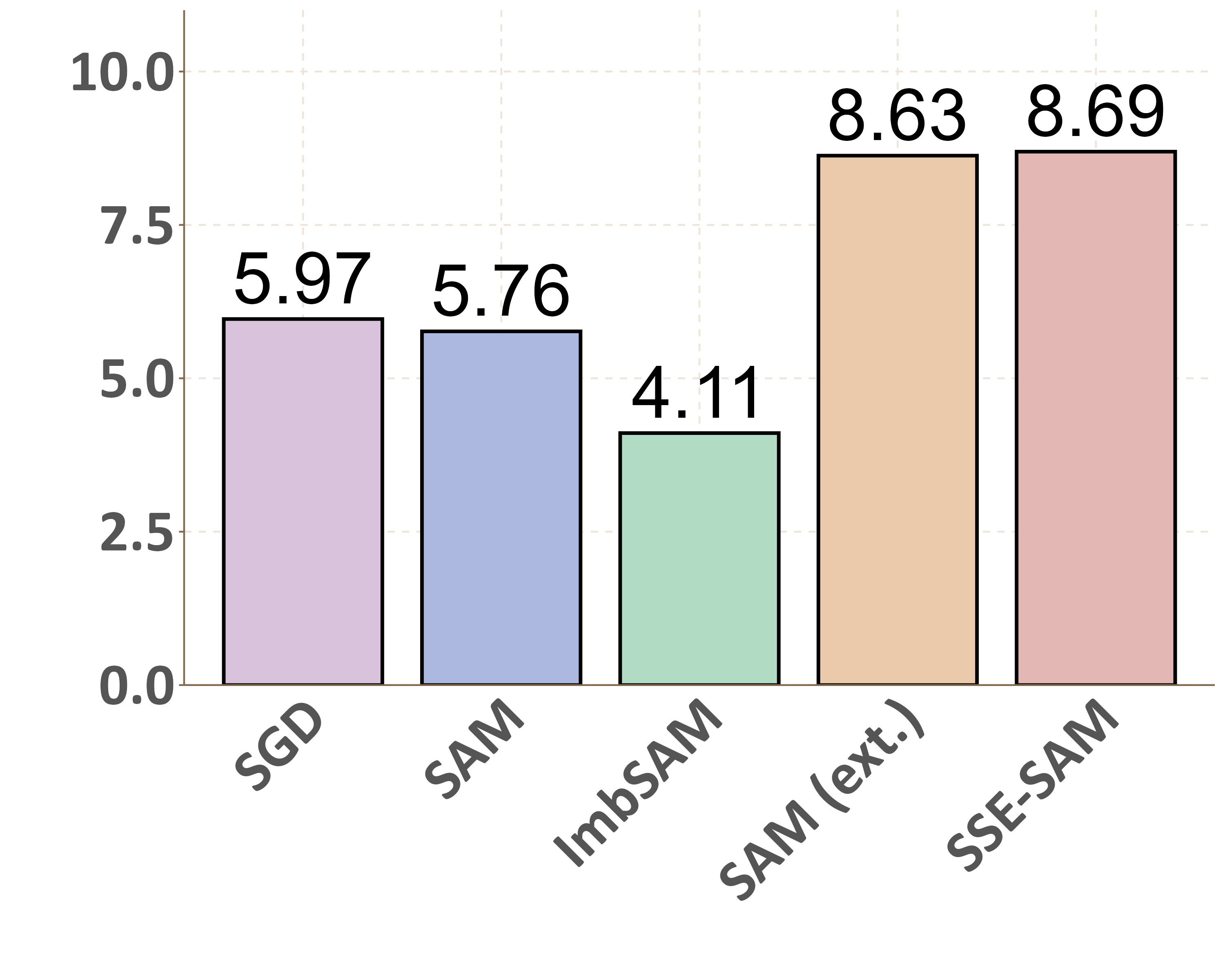}
      \caption{Class 0}
      \label{fig:hessian_analysis_b}
    \end{subfigure}%
    % ~% add desired spacing
    \hfill
    \begin{subfigure}{0.23 \textwidth}
      \includegraphics[width=\textwidth]{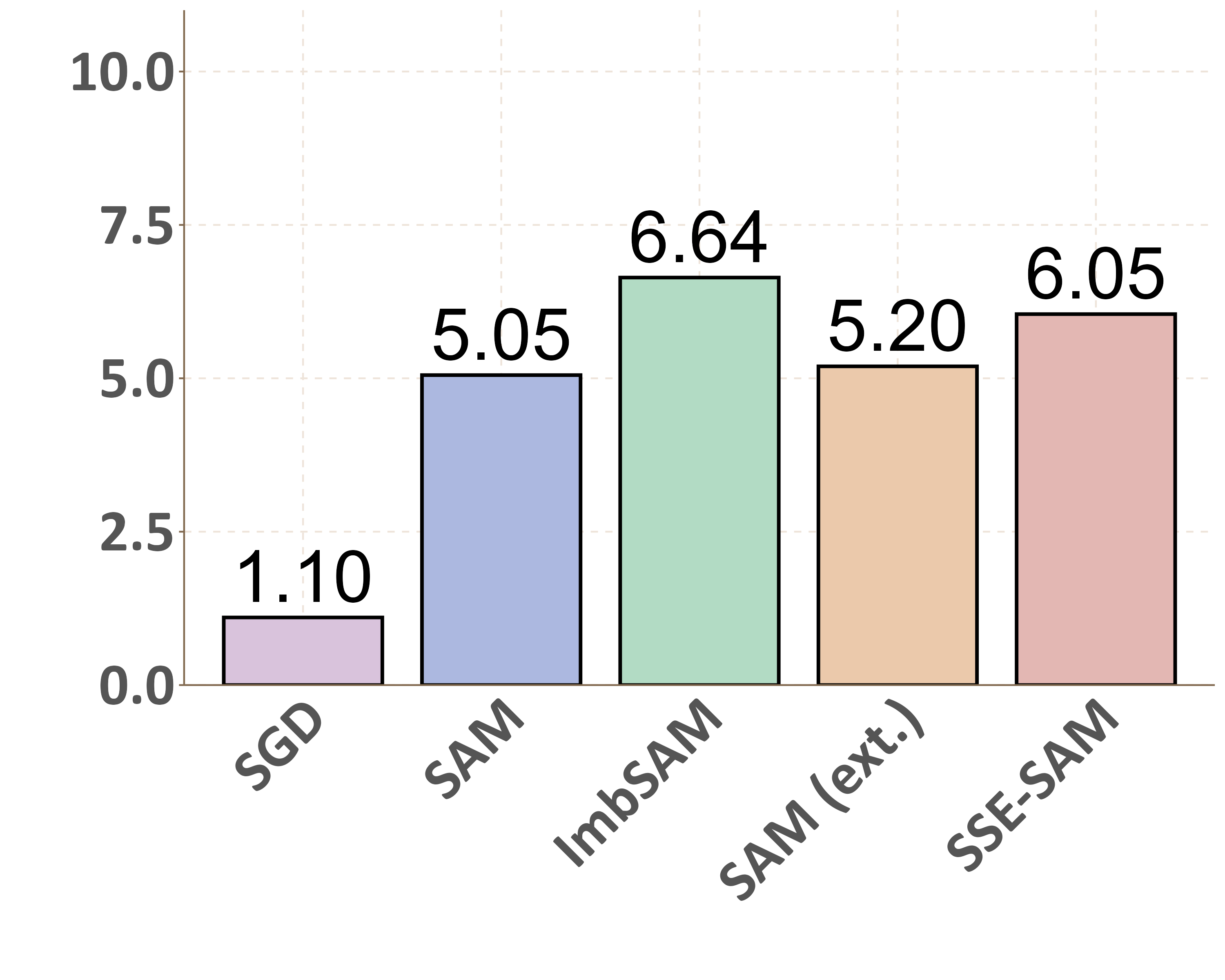}
      \caption{Class 80}
      \label{fig:hessian_analysis_c}
    \end{subfigure}%
    \caption{\textbf{Comparative Analysis of $\lratio$ for Different Models.} These results validate SSE-SAM's effective balance between head and tail classes compared with SAM and ImbSAM.}
    % Panel (a) demonstrates SSE-SAM's enhanced performance across all classes. Panel (b) \ref{fig:hessian_analysis_b} highlights its superiority in head classes (class 0), while Panel (c) \ref{fig:hessian_analysis_c} depicts its performance in tail classes (class 80), where it ranks between ImbSAM and SAM (ext.).
    \label{fig:hessian_analysis}
\end{figure}

We utilize $\lratio$, as described in Preliminary section, to evaluate convexity. We analyze the $\lratio$ for all classes, head classes, and tail classes across models trained using SGD, SAM ($\rho = 0.05$), ImbSAM ($\rho=0.10$), SAM (ext.) ($\rhead=0.05, \rtail=0.10$), and our SSE-SAM ($\rhead=0.05, \rtail=0.10, \gamma=0.7$). The outcomes are illustrated in Fig.\ref{fig:hessian_analysis}. Fig.\ref{fig:hessian_analysis_a} distinctly demonstrates that SSE-SAM shows superior overall performance compared to the other methods. Specifically, for head classes, we analyze $\lratio$ for class 0, where Fig.\ref{fig:hessian_analysis_b} shows that SSE-SAM outperforms all other methods. In the case of class 80, which represents tail classes, Fig.\ref{fig:hessian_analysis_c} reveals that SSE-SAM's capacity to escape saddle points lies between ImbSAM and SAM (ext.), all sharing the same $\rtail$. Concerning $\rhead$, our analysis indicates that the predominance of head classes results in suboptimal escape from saddles in tail classes. Here, the $\rhead$ of ImbSAM is $0$, indicating an absence of suppressive effect from head classes; the $\rhead$ of SAM (ext.) is $0.05$, displaying the most significant suppressive effect; the $\rhead$ of SSE-SAM transitions from $0.05$ to $0$, reflecting a suppressive impact that lies between the other two methods. These findings corroborate our theoretical expectations.

% -------------------------------------------------------------------

\section{Conclusion}
\label{conclusion}

In this work, we theoretically show that SAM has better ability in escaping saddle points compared with ImbSAM. Under long-tail distribution, we analyze that SAM is good at escaping saddles for head classes but fails to efficiently escape saddles for tail classes. ImbSAM, on the contrary, focuses on escaping saddles for tail classes while suffers from insufficient training of head classes. Leveraging this complementary effect, we propose a two-stage saddle escaping algorithm SSE-SAM. At the first stage, SSE-SAM uses SAM to help head classes escape saddles. At the next stage, SSE-SAM sets $\rhead=0$ to focus on helping tail classes in escaping saddles. Experiment results show that SSE-SAM escapes saddles efficiently both on head and tail classes, achieving outstanding performance among all compared baselines.

% -------------------------------------------------------------------

\appendix

% -------------------------------------------------------------------

\section{Limitations}
\label{app:limitations}

This paper conducts a theoretical analysis of the ability of the SAM and ImbSAM algorithms to escape saddle points and introduces the SSE-SAM algorithm, which performs better in long-tail scenarios. Although progress has been made, there are still areas for improvement. Possible future research directions include:

\begin{itemize}
    \item Theoretically investigating the reasons why SSE-SAM performs better than expected. In Tab.~\ref{tab:trade-off}, SSE-SAM outperforms ImbSAM for both middle and tail classes. However, the misleading influence of head classes on tail classes is only mitigated in later stages. This unexpectedly superior performance warrants further theoretical exploration.
    \item Exploring the application of more smoothness regularization methods in long-tail distributions. This paper primarily focuses on the state-of-the-art smoothness regularization methods SAM and ImbSAM. Future work could involve introducing additional smoothness regularization techniques to improve model performance in handling long-tail distributions.
\end{itemize}

% -------------------------------------------------------------------

\section{Related Work}
\label{app:related_work}

In this section, we begin by describing relevant metrics and reviewing re-weighting methods in long-tail learning. We then explore popular smoothness regularization techniques. Next, we present conclusions regarding saddle points in the loss landscape. Finally, we explain the notation conventions used throughout this paper.

\subsection{Long-tail Learning}

Class re-balancing is a prevalent method for training on long-tailed datasets \cite{chawla2002smote, lin2017focal, yang2022learning, yang2022optimizing, hou2022adauc, han2024aucseg}. This approach seeks to mitigate the adverse effects of tail classes through three main strategies: re-sampling, class-sensitive learning, and logit adjustment. Re-sampling involves either oversampling tail classes \cite{chawla2002smote} or undersampling head classes \cite{buda2018systematic} to achieve a more balanced dataset. Class-sensitive learning \cite{lin2017focal, cao2019learning, kini2021label, gao2023feature, li2024size} modifies the learning objectives for different classes to equalize training impact. Logit adjustment \cite{menon2020long, wang2024unified} addresses class imbalance by modifying the prediction logits of neural networks.

\subsection{Smoothness Regularization Methods}

Smoothness regularization algorithms aim to reduce both the loss value and sharpness, guiding models toward flat minima in the loss landscape, which is known to enhance performance \cite{keskar2016large, jiang2019fantastic}. Sharpness-Aware Minimization (SAM) \cite{foret2020sharpness}, a notable method in this area, has gained significant attention recently for its solid theoretical foundation and impressive results.

For theory, \cite{woodworth2020kernel} demonstrate how initial scale impacts generalization in multilayer homogeneous models using gradient flow. In terms of theoretical research,  \cite{woodworth2020kernel} demonstrate how initial scale impacts generalization in multilayer homogeneous models using gradient flow. \cite{andriushchenko2022towards} further apply the gradient flow method to compare two SAM variants, m-SAM and n-SAM, showing that m-SAM's effects are similar to $\ell_1$ regularization, which improves its efficacy in sparse scenarios. \cite{wen2022does} prove that n-SAM actually minimizes maximum eigenvalue of Hessian, while m-SAM minimizes its trace. Moreover, \cite{wen2023sharpness} observe that although SAM may not always reach the flattest minima, it generally performs well, suggesting that minimizing sharpness is not its only pathway to enhanced performance. 

For application, \cite{du2021efficient} developed the Efficient Sharpness-Aware Minimization (ESAM) algorithm, which maintains SAM’s efficiency and generalization capabilities. \cite{kwon2021asam} introduce Adaptive Sharpness-Aware Minimization (ASAM), which uses adaptive sharpness that is scale-invariant, leading to improved model performance. \cite{mi2022make} propose Sparse SAM (SSAM), which applies a binary mask to achieve sparse perturbations. SSAM matches SAM in convergence speed and effectively smooths the loss landscape. Lastly, \cite{zhou2023imbsam} introduce Imbalanced SAM (ImbSAM), integrating class priors into SAM to tackle overfitting in tail classes within long-tail learning contexts.

\subsection{Escaping Saddle Points}

Saddle points are regions in the loss landscape that are flat but have both positive and negative curvatures. They occur where the function's gradient is zero and its Hessian matrix is indefinite, featuring both positive and negative eigenvalues. 

Research has demonstrated that models converging to saddle points typically exhibit poor generalization. By updating weights in the direction of negative curvature, models can significantly improve their generalization abilities. Thus, developing algorithms that effectively bypass or escape from saddle points is critical in machine learning \cite{dauphin2014identifying}. There are extensive theoretical studies on saddle point problem \cite{daneshmand2018escaping, jin2021nonconvex, rangwani2022escaping}, and many algorithms have been developped to efficiently escape from saddle points \cite{palaniappan2016stochastic, jin2017escape, staib2019escaping, criscitiello2019efficiently, zhang2021escape, huang2022efficiently}. \citet{rangwani2022escaping} found that the SGD method escapes saddle points in head classes but not in tail classes. Compared to SGD, SAM better assists both head and tail classes in escaping saddle points, thereby boosting generalization capabilities. However, SAM still faces challenges in aiding tail classes due to the overpowering influence of head classes.

% -------------------------------------------------------------------

\section{Computation of Weingarten Matrix}
\label{app:weingarten_computation}

For $\boldsymbol{w}=\left[w^{(1)}, \cdots, w^{(d)}\right]\in\mathbb{R}^d$, to simplify notation, the first-order partial derivative of the multivariate function $f(\boldsymbol{w})$ w.r.t. $w^{(i)}$ will be denoted as $\partial_if(\boldsymbol{w})$ or $\partial_if$, and the second-order partial derivative will be denoted as $\partial_{ij}f := \partial_{ij}f(\boldsymbol{w}):=\frac{\partial^2 f(\boldsymbol{w})}{\partial w^{(i)}\partial w^{(j)}}$.

We aim to compute $\mathbf{W}$. Drawing on theories from differential geometry\cite{lee2018introduction}, it is established that $\mathbf{W}=\mathcal{I}^{-1}\cdot\mathcal{II}$, where $\mathcal{I}$ represents the first fundamental form of the hypersurface, and $\mathcal{II}$ denotes the second fundamental form.

In a Cartesian coordinate system, the loss function hypersurface is parameterized as follows:
\begin{equation}
    \boldsymbol{r}\left(w^{(1)}, \cdots, w^{(d)}\right) = \left[w^{(1)}, \cdots, w^{(d)}, \mathcal{L}(\boldsymbol{w})\right]^{\top}.
\end{equation}

The computations yield:
\begin{equation}
\begin{aligned}
    \partial_i\boldsymbol{r} &= \left[0,\cdots,0,\underbrace{1}_{i\text{-th position}}, 0, \cdots, 0, \partial_i\mathcal{L}\right]^{\top}, \\
    \partial_{ij}\boldsymbol{r} &= \left[0, \cdots, 0, \partial_{ij}\mathcal{L}\right]^{\top}.
\end{aligned}
\end{equation}

The $ij$-th element of the first fundamental form of the hypersurface is:
\begin{equation}
\begin{aligned}
    \mathcal{I}_{ij} &= \partial_i\boldsymbol{r}^{\top}\partial_j\boldsymbol{r} \\
    &= \left\{\begin{aligned}
        & \partial_i\mathcal{L}\partial_j\mathcal{L}    &, i\neq j,\\
        & \partial_i\mathcal{L}^2 + 1 &, i=j.
    \end{aligned}
    \right.
\end{aligned}\end{equation}
Hence, $\mathcal{I} = \mathbf{I} + \nabla \mathcal{L}\cdot \nabla \mathcal{L}^{\top}$ represents a typical rank-one update matrix. The inverse can be derived using the Sherman-Morrison formula:
\begin{equation}
    \mathcal{I}^{-1} = \mathbf{I} - \frac{\nabla \mathcal{L}\cdot \nabla \mathcal{L}^{\top}}{\alpha^2},
\end{equation}
where $\alpha = \sqrt{1 + \lVert\nabla \mathcal{L}\rVert^2}$ simplifies the denominator.

In a $d+1$-dimensional space, the normal vector to the hypersurface defined by $F(x_1, ..., x_d, x_{d+1})=0$ is $n=\nabla F$. Considering $\mathcal{L} - \mathcal{L}(\boldsymbol{w}) = 0$ as the equation of the hypersurface in $d+1$ dimensions, with $L$ representing the $d+1$-th component. The normal vector at the point $(\boldsymbol{w}, \mathcal{L}(\boldsymbol{w}))$ is given by:
\begin{equation}
    \boldsymbol{n} = \left[-\partial_1\mathcal{L}, \cdots, -\partial_d\mathcal{L}, 1\right]^{\top} / \alpha,
\end{equation}
where the denominator $\alpha$ is the normalization factor. Therefore, the $ij$-th element of the second fundamental form of the hypersurface is:
\begin{equation}
\begin{aligned}
    \mathcal{II}_{ij} &= \partial_{ij}\boldsymbol{r}^{\top}\boldsymbol{n} \\
    &= \partial_{ij}\mathcal{L} / \alpha.
\end{aligned}
\end{equation}

Thus, the Weingarten matrix, which describes how the hypersurface bends in different directions, can be computed as:
\begin{equation}
    \mathbf{W} = \left(\mathbf{I} - \frac{\nabla \mathcal{L}\cdot\nabla \mathcal{L}^{\top}}{\alpha^2}\right)\frac{\nabla ^2\mathcal{L}}{\alpha}.
\end{equation}

% -------------------------------------------------------------------

\section{Omitted Proofs}
\label{app:omitted_proofs}

We provide the omitted proofs of propositions and theorems.

\subsection{Proof of Thm.\ref{thm:projection-onto-negative-curvature}}
\label{app:projection-onto-negative-curvature}

We first re-state Thm.~\ref{thm:projection-onto-negative-curvature} here.

\begin{theorem}
    Let $\mathcal{W} := \left\{\bw_0, \bw_1, \cdots, \bw_t, \btw_1, \cdots, \btw_t\right\}$, where $\bw_0$ is the starting weight, $\bw_t$ denotes the weight obtained after optimizing for $t$ steps using SGD, $\btw_t$ denotes the weight obtained after optimizing for $t$ steps using SAM. All weights in $\mathcal{W}$ are around saddle point. Let $L := \max\limits_{\bw \in \mathcal{W}}\lVert \nabla L(\bw)\rVert>0$. Given $\mu\geq 1$, when $\lmin\leq -\frac{L}{\eta\rho}\left(\frac{1+\mu}{2}\eta+\sqrt{\left(\frac{1+\mu}{2}\right)^2\eta^2+\frac{\left(2+\mu\right)\eta\rho}{L}}\right)$, under second order Taylor approximation, we have
    \begin{equation}
    \begin{aligned}
        \lvert\langle \boldsymbol{v}_{\text{min}}, \widetilde{\boldsymbol{w}}_t-\boldsymbol{w}_0\rangle\rvert\geq\mu\lvert\langle \boldsymbol{v}_{\text{min}}, \boldsymbol{w}_t-\boldsymbol{w}_0\rangle\rvert\\
        +\left((\mu^2+\mu)t-2\mu^2-\mu\right)\frac{\lvert\langle\vmin, \nabla \mathcal{L}(\bw_0)\rangle\rvert}{\lvert\lmin\rvert}.
    \end{aligned}
    \end{equation}
\end{theorem}

According to the theorem's assumptions, the weights on the optimization trajectory are all near a saddle point. Thus, the original loss function can be approximated using a second-order Taylor expansion around $\bw_0$:
\begin{equation}\label{eq:taylor-expansion-l}
\begin{aligned}
    \mathcal{L}(\boldsymbol{w}_i)=\mathcal{L}(\boldsymbol{w}_0) + (\boldsymbol{w}-\boldsymbol{w}_0)^{\top}\nabla \mathcal{L}(\boldsymbol{w}_0) \\
    + \frac{1}{2}(\boldsymbol{w}-\boldsymbol{w}_0)^{\top}\mathcal{H}(\boldsymbol{w}-\boldsymbol{w}_0).
\end{aligned}
\end{equation}
where $\mathcal{H}$ denotes the Hessian matrix at $\bw_0$.

To prove Thm.~\ref{thm:projection-onto-negative-curvature}, we first introduce two lemmas that measure the magnitude of the inner product between $\boldsymbol{v}_{\text{min}}$ and the gradient of the loss function $\nabla L$ after $t$ steps of optimization using SGD and SAM respectively.

\begin{lemma}\label{lem:sgd-vmin-multiply-nabla-l}
    Under the conditions of Thm.~\ref{thm:projection-onto-negative-curvature}, it holds that
    \begin{equation}\label{eq:sgd-vmin-multiply-nabla-l}
        \langle \boldsymbol{v}_{\text{min}}, \nabla \mathcal{L}(\boldsymbol{w}_t)\rangle = (1-\eta \lambda_{\text{min}})^t\langle \boldsymbol{v}_{\text{min}}, \nabla \mathcal{L}(\boldsymbol{w}_0)\rangle.
    \end{equation}
\end{lemma}

\begin{proof}
    Starting from the Taylor expansion given in Eq. (~\ref{eq:taylor-expansion-l}), we have:
    \begin{equation}\label{eq:taylor-expansion-nabla-l}
        \nabla \mathcal{L}(\boldsymbol{w}) = \nabla \mathcal{L}(\boldsymbol{w}_0) + \mathcal{H}(\boldsymbol{w}-\boldsymbol{w}_0).
    \end{equation}

    Substituting $\boldsymbol{w}_t$ for $\boldsymbol{w}$ in the above equation yields:
    \begin{equation}\label{eq:sgd-taylor-expansion-nabla-l}
        \nabla \mathcal{L}(\boldsymbol{w}_t) = \nabla \mathcal{L}(\boldsymbol{w}_0) + \mathcal{H}(\boldsymbol{w}_t-\boldsymbol{w}_0).
    \end{equation}
    
    The displacement from the initial point can be expressed as:
    \begin{equation}
    \begin{aligned}
        \boldsymbol{w}_t - \boldsymbol{w}_0 &= \boldsymbol{w}_{t-1} - \eta\nabla \mathcal{L}(\boldsymbol{w}_{t-1}) - \boldsymbol{w}_0 \\
                  &= \boldsymbol{w}_{t-1} - \boldsymbol{w}_0 - \eta\left(\nabla \mathcal{L}(\boldsymbol{w}_0)+\mathcal{H}\left(\boldsymbol{w}_{t-1}-\boldsymbol{w}_0\right)\right) \\
                  &= (\mathbf{I}-\eta \mathcal{H})(\boldsymbol{w}_{t-1}-\boldsymbol{w}_0) - \eta\nabla \mathcal{L}(\boldsymbol{w}_0).
    \end{aligned}
    \end{equation}

    Iteratively applying this relationship gives:
    \begin{equation}\label{eq:sam-w_t-w_0}
    \begin{aligned}
        \boldsymbol{w}_t-\boldsymbol{w}_0 =& (\mathbf{I}-\eta \mathcal{H})^{t-1}(\boldsymbol{w}_1-\boldsymbol{w}_0) - \\
                & \eta\sum_{i=0}^{t-2}(\mathbf{I}-\eta\mathcal{H})^i\nabla \mathcal{L}(\boldsymbol{w}_0) \\
                =& -\eta\sum_{i=0}^{t-1}(\mathbf{I}-\eta\mathcal{H})^i\nabla \mathcal{L}(\boldsymbol{w}_0) \\
                =& -\left[\mathbf{I}-\left(\mathbf{I}-\eta\mathcal{H}\right)^t\right]\mathcal{H}^{-1}\nabla \mathcal{L}(\boldsymbol{w}_0).
    \end{aligned}
    \end{equation}

    Consequently, the gradient at $\bw_t$ can be rewritten as:
    \begin{equation}
    \begin{aligned}
        \nabla \mathcal{L}(\boldsymbol{w}_t) &= \nabla \mathcal{L}(\boldsymbol{w}_0) + \mathcal{H}(\boldsymbol{w}_t-\boldsymbol{w}_0) \\
                      &= \nabla \mathcal{L}(\boldsymbol{w}_0) + \mathcal{H}\left[\mathbf{I}-\left(\mathbf{I}-\eta \mathcal{H}\right)^t\right]\mathcal{H}^{-1}\nabla \mathcal{L}(\boldsymbol{w}_0) \\
                      &= (\mathbf{I}-\eta \mathcal{H})^t\nabla \mathcal{L}(\boldsymbol{w}_0).
    \end{aligned}
    \end{equation}

    Thus, the inner product of interest becomes:
    \begin{equation}
    \begin{aligned}
        \langle \boldsymbol{v}_{\text{min}}, \nabla \mathcal{L}(\boldsymbol{w}_t)\rangle &= \nabla \mathcal{L}(\boldsymbol{w}_0)^{\top}(\mathbf{I}-\eta\mathcal{H})^t\boldsymbol{v}_{\text{min}} \\
        &= \nabla \mathcal{L}(\boldsymbol{w}_0)^{\top}\sum_{i=0}^{t}\binom{t}{i}(-1)^i\eta^i\mathcal{H}^i\boldsymbol{v}_{\text{min}} \\
        &= \nabla \mathcal{L}(\boldsymbol{w}_0)^{\top}\sum_{i=0}^{t}\binom{t}{i}(-1)^i\eta^i\lambda_{\text{min}}^i\boldsymbol{v}_{\text{min}} \\
        &= (1-\eta\lambda_{\text{min}})^t\nabla \mathcal{L}(\boldsymbol{w}_0)^{\top}\boldsymbol{v}_{\text{min}} \\
        &= (1-\eta\lambda_{\text{min}})^t\langle \boldsymbol{v}_{\text{min}}, \nabla \mathcal{L}(\boldsymbol{w}_0)\rangle.
    \end{aligned}
    \end{equation}
\end{proof}

\begin{lemma}\label{lem:sam-vmin-multiply-nabla-l}
    Under the conditions of Thm.~\ref{thm:projection-onto-negative-curvature}, it holds that
    \begin{equation}\label{eq:sam-vmin-multiply-nabla-l}
    \begin{aligned}
        \lvert\langle \boldsymbol{v}_{\text{min}}, \nabla \mathcal{L}(\widetilde{\boldsymbol{w}}_t)\rangle\rvert \geq& \left(\eta\rho\lambda_{\text{min}}^2/L + \eta \lambda_{\text{min}} - 1\right) ^t\cdot\\
        &\lvert \langle \boldsymbol{v}_{\text{min}}, \nabla \mathcal{L}(\boldsymbol{w}_0)\rangle\rvert.
    \end{aligned}
    \end{equation}
\end{lemma}

\begin{proof}
    Starting with the Taylor expansion from Eq. (~\ref{eq:taylor-expansion-l}), we have:
    \begin{equation}
        \nabla \mathcal{L}(\boldsymbol{w}) = \nabla \mathcal{L}(\boldsymbol{w}_0) + \mathcal{H}(\boldsymbol{w}-\boldsymbol{w}_0).
    \end{equation}

    Substituting $\widetilde{\boldsymbol{w}}_t+\boldsymbol{\epsilon}_t$ for $\boldsymbol{w}$ in the above equation, we obtain:
    \begin{equation}\label{eq:sam-taylor-expansion-nabla-l}
        \nabla \mathcal{L}(\widetilde{\boldsymbol{w}}_t+\boldsymbol{\epsilon}_t) = \nabla \mathcal{L}(\boldsymbol{w}_0) + \mathcal{H}(\widetilde{\boldsymbol{w}}_t+\boldsymbol{\epsilon}_t-\boldsymbol{w}_0).
    \end{equation}

    For SAM, considering iterative updates:
    \begin{equation}
    \begin{aligned}
        \widetilde{\boldsymbol{w}}_t-\boldsymbol{w}_0 =& \widetilde{\boldsymbol{w}}_{t-1} - \eta\nabla \mathcal{L}(\widetilde{\boldsymbol{w}}_{t-1} + \boldsymbol{\epsilon}_{t-1}) - \boldsymbol{w}_0 \\
        =& \widetilde{\boldsymbol{w}}_{t-1} - \boldsymbol{w}_0 - \eta(\nabla \mathcal{L}(\boldsymbol{w}_0) +\\
        & \mathcal{H}\left(\widetilde{\boldsymbol{w}}_{t-1} - \boldsymbol{w}_0 + \boldsymbol{\epsilon}_{t-1}\right)) \\
        =& (\mathbf{I}-\eta\mathcal{H})(\widetilde{\boldsymbol{w}}_{t-1}-\boldsymbol{w}_0) - \\
        & \eta\left(\nabla \mathcal{L}(\boldsymbol{w}_0)+\mathcal{H}\boldsymbol{\epsilon}_{t-1}\right) \\
        =& -\eta\sum_{i=0}^{t-1}(\mathbf{I}-\eta\mathcal{H})^i\left(\nabla \mathcal{L}(\boldsymbol{w}_0)+\mathcal{H}\boldsymbol{\epsilon}_{t-1-i}\right) \\
        =& -\left[\mathbf{I}-\left(\mathbf{I}-\eta\mathcal{H}\right)^t\right]\mathcal{H}^{-1}\nabla \mathcal{L}(\boldsymbol{w}_0) - \\
        & \eta\sum_{i=0}^{t-1}(\mathbf{I}-\eta\mathcal{H})^i\mathcal{H}\boldsymbol{\epsilon}_{t-1-i}.
    \end{aligned}
    \end{equation}

    This leads to
    \begin{equation}
    \begin{aligned}
        \nabla \mathcal{L}(\widetilde{\boldsymbol{w}}_t) =& (\mathbf{I}-\eta\mathcal{H})^t\nabla \mathcal{L}(\boldsymbol{w}_0) - \\
        & \eta\sum_{i=0}^{t-1}(\mathbf{I}-\eta\mathcal{H})^i\mathcal{H}^2\boldsymbol{\epsilon}_{t-1-i},
    \end{aligned}
    \end{equation}
    and
    \begin{equation}
    \begin{aligned}
        \nabla \mathcal{L}(\widetilde{\boldsymbol{w}}_{t-1}) =& (\mathbf{I}-\eta\mathcal{H})^{t-1}\nabla \mathcal{L}(\boldsymbol{w}_0) - \\
        & \eta\sum_{i=0}^{t-1}(\mathbf{I}-\eta\mathcal{H})^{i-1}\mathcal{H}^2\boldsymbol{\epsilon}_{t-1-i}.
    \end{aligned}
    \end{equation}

    Therefore, the recursive relation is established by combining the above two formulas:
    \begin{equation}
        \nabla \mathcal{L}(\widetilde{\boldsymbol{w}}_t) = (\mathbf{I}-\eta \mathcal{H})\nabla \mathcal{L}(\widetilde{\boldsymbol{w}}_{t-1}) - \eta\mathcal{H}^2\boldsymbol{\epsilon}_{t-1}.
    \end{equation}

    This recursion yields:
    \begin{equation}
    \begin{aligned}
        \langle \boldsymbol{v}_{\text{min}}, \nabla \mathcal{L}(\widetilde{\boldsymbol{w}}_t)\rangle =& \left(1-\eta\lambda_{\text{min}}-\eta\rho\lambda_{\text{min}}^2 / \lVert\nabla \mathcal{L}(\widetilde{\boldsymbol{w}}_{t-1})\rVert\right)\cdot\\
        &\langle \boldsymbol{v}_{\text{min}}, \nabla \mathcal{L}(\widetilde{\boldsymbol{w}}_{t-1})\rangle.
    \end{aligned}
    \end{equation}

    We know that 
    \begin{equation}
        \lmin\leq -\frac{L}{\eta\rho}\left(\frac{1+\mu}{2}\eta+\sqrt{\left(\frac{1+\mu}{2}\right)^2\eta^2+\frac{\left(2+\mu\right)\eta\rho}{L}}\right)
    \end{equation}
    therefore, because $\mu\geq 1$, it holds that
    \begin{equation}
        \lambda_{\text{min}}\leq -\frac{L}{\eta \rho}\left(\eta + \sqrt{\eta^2+\frac{3\eta \rho}{L}}\right).
    \end{equation}
    It follows that 
    \begin{equation}
    \begin{aligned}
        &\eta\rho\lambda_{\text{min}}^2 / \lVert\nabla \mathcal{L}(\widetilde{\boldsymbol{w}}_{t-1})\rVert + \eta\lambda_{\text{min}} - 1 \\
        \geq& \eta\rho\lambda_{\text{min}}^2 / L + \eta\lambda_{\text{min}} - 1 \\
        \geq& 2 - \eta\lambda_{\text{min}} > 0.
    \end{aligned}
    \end{equation}
    Therefore
    \begin{equation}
    \begin{aligned}
        \lvert\langle \boldsymbol{v}_{\text{min}}, \nabla \mathcal{L}(\widetilde{\boldsymbol{w}}_t)\rangle\rvert
        =& \left(\eta\rho\lambda_{\text{min}}^2 / \lVert\nabla \mathcal{L}(\widetilde{\boldsymbol{w}}_{t-1})\rVert + \eta\lambda_{\text{min}} - 1\right)\\
        & \lvert\langle \boldsymbol{v}_{\text{min}}, \nabla \mathcal{L}(\widetilde{\boldsymbol{w}}_{t-1})\rangle\rvert \\
        \geq& \left(\eta\rho\lambda_{\text{min}}^2 / L + \eta\lambda_{\text{min}} - 1\right) \\
        & \lvert\langle \boldsymbol{v}_{\text{min}}, \nabla \mathcal{L}(\widetilde{\boldsymbol{w}}_{t-1})\rangle\rvert.
    \end{aligned}
    \end{equation}
\end{proof}

Now that we have the above two lemmas, we can proceed to prove Thm.~\ref{thm:projection-onto-negative-curvature}.

\begin{proof}
    From the Taylor expansion in Eq. (~\ref{eq:taylor-expansion-nabla-l}), we have:
    \begin{equation}
        \langle \boldsymbol{v}_{\text{min}}, \nabla \mathcal{L}(\boldsymbol{w}) \rangle = \langle \boldsymbol{v}_{\text{min}}, \nabla \mathcal{L}(\boldsymbol{w}_0) \rangle + \lambda_{\text{min}}\langle \boldsymbol{v}_{\text{min}}, \boldsymbol{w} - \boldsymbol{w}_0\rangle.
    \end{equation}

    This leads to:
    \begin{equation}
        \langle \boldsymbol{v}_{\text{min}}, \boldsymbol{w} - \boldsymbol{w}_0\rangle = \frac{\langle \boldsymbol{v}_{\text{min}}, \nabla \mathcal{L}(\boldsymbol{w}) \rangle - \langle \boldsymbol{v}_{\text{min}}, \nabla \mathcal{L}(\boldsymbol{w}_0) \rangle}{\lambda_{\text{min}}}.
    \end{equation}

    Thus, we derive:
    \begin{equation}
        \langle \boldsymbol{v}_{\text{min}}, \boldsymbol{w}_t - \boldsymbol{w}_0\rangle = \frac{\langle \boldsymbol{v}_{\text{min}}, \nabla \mathcal{L}(\boldsymbol{w}_t) \rangle - \langle \boldsymbol{v}_{\text{min}}, \nabla \mathcal{L}(\boldsymbol{w}_0) \rangle}{\lambda_{\text{min}}},
    \end{equation}
    \begin{equation}
        \langle \boldsymbol{v}_{\text{min}}, \widetilde{\boldsymbol{w}}_t - \boldsymbol{w}_0\rangle = \frac{\langle \boldsymbol{v}_{\text{min}}, \nabla \mathcal{L}(\widetilde{\boldsymbol{w}}_t) \rangle - \langle \boldsymbol{v}_{\text{min}}, \nabla \mathcal{L}(\boldsymbol{w}_0) \rangle}{\lambda_{\text{min}}}.
    \end{equation}

    Given $\mu\geq 1$, to prove that 
    \begin{equation}
    \begin{aligned}
        \lvert\langle \boldsymbol{v}_{\text{min}}, \widetilde{\boldsymbol{w}}_t-\boldsymbol{w}_0\rangle\rvert\geq\mu\lvert\langle \boldsymbol{v}_{\text{min}}, \boldsymbol{w}_t-\boldsymbol{w}_0\rangle\rvert \\ +\left((\mu^2+\mu)t-2\mu^2-\mu\right)\frac{\lvert\langle\vmin, \nabla \mathcal{L}(\bw_0)\rangle\rvert}{\lvert\lmin\rvert},
    \end{aligned}
    \end{equation}
    we only have to prove that 
    \begin{equation}
    \begin{aligned}
        & \lvert \langle \boldsymbol{v}_{\text{min}}, \nabla \mathcal{L}(\widetilde{\boldsymbol{w}}_t) \rangle - \langle \boldsymbol{v}_{\text{min}}, \nabla \mathcal{L}(\boldsymbol{w}_0) \rangle\rvert - \\
        & \mu\lvert \langle \boldsymbol{v}_{\text{min}}, \nabla \mathcal{L}(\boldsymbol{w}_t) \rangle - \langle \boldsymbol{v}_{\text{min}}, \nabla \mathcal{L}(\boldsymbol{w}_0) \rangle\rvert \\ 
        \geq&  \left((\mu^2+\mu)t-2\mu^2-\mu\right)\lvert\langle\vmin, \nabla \mathcal{L}(\bw_0)\rangle\rvert.
    \end{aligned}
    \end{equation}
    
    By triangle inequality, it suffices to show that
    \begin{equation}
    \begin{aligned}
        \lvert \langle \boldsymbol{v}_{\text{min}}, \nabla \mathcal{L}(\widetilde{\boldsymbol{w}}_t) \rangle\rvert - \mu\lvert \langle \boldsymbol{v}_{\text{min}}, \nabla \mathcal{L}(\boldsymbol{w}_t) \rangle\rvert -\\
        \left((\mu^2+\mu)t-2\mu^2+1\right)\lvert\langle \boldsymbol{v}_{\text{min}}, \nabla \mathcal{L}(\boldsymbol{w}_0) \rangle\rvert \geq 0.
    \end{aligned}
    \end{equation}

    Using Lem.~\ref{lem:sgd-vmin-multiply-nabla-l} and Lem.\ref{lem:sam-vmin-multiply-nabla-l}, we have:
    \begin{equation}
    \begin{aligned}
        & \lvert \langle \boldsymbol{v}_{\text{min}}, \nabla \mathcal{L}(\widetilde{\boldsymbol{w}}_t) \rangle\rvert - \mu\lvert \langle \boldsymbol{v}_{\text{min}}, \nabla \mathcal{L}(\boldsymbol{w}_t) \rangle\rvert -\\
        & \left((\mu^2+\mu)t-2\mu^2+1\right)\lvert\langle \boldsymbol{v}_{\text{min}}, \nabla \mathcal{L}(\boldsymbol{w}_0) \rangle\rvert \\
        \geq & \Big[\left(\eta\rho\lambda_{\text{min}}^2/L + \eta \lambda_{\text{min}} - 1\right) ^t - \mu\left(1-\eta\lambda_{\text{min}}\right)^t -\\
        & \left((\mu^2+\mu)t-2\mu^2+1\right)\Big]\lvert\langle \boldsymbol{v}_{\text{min}}, \nabla \mathcal{L}(\boldsymbol{w}_0) \rangle\rvert \\
        \geq & \Big[\left(\eta\rho\lambda_{\text{min}}^2/L + \eta \lambda_{\text{min}} - 1\right) ^t - \mu^t\left(1-\eta\lambda_{\text{min}}\right)^t -\\
        & \left((\mu^2+\mu)t-2\mu^2+1\right)\Big]\lvert\langle \boldsymbol{v}_{\text{min}}, \nabla \mathcal{L}(\boldsymbol{w}_0) \rangle\rvert.
    \end{aligned}
    \end{equation}
    Therefore, it suffices to show 
    \begin{equation}
    \begin{aligned}
        & \left(\eta\rho\lambda_{\text{min}}^2/L + \eta \lambda_{\text{min}} - 1\right)^t - \mu^t\left(1-\eta\lambda_{\text{min}}\right)^t \\
        \geq& (\mu^2+\mu )t-2\mu^2+1.
    \end{aligned}
    \end{equation}

    To establish the desired inequality, apply the binomial theorem for difference of powers:
    \begin{equation}
    \begin{aligned}
        & \left(\eta\rho\lambda_{\text{min}}^2/L + \eta \lambda_{\text{min}} - 1\right) ^t - \mu^t\left(1-\eta\lambda_{\text{min}}\right)^t \\
        = & \left(\eta\rho\lambda_{\text{min}}^2/L + (1+\mu)\eta \lambda_{\text{min}} - (1+\mu)\right)\cdot \\
        & \left(\sum_{i=0}^{t-1}\underbrace{\left(\eta\rho\lambda_{\text{min}}^2/L + \eta \lambda_{\text{min}} - 1\right)^{t-1-i}\left(\mu-\mu\eta\lambda_{\text{min}}\right)^i}_{T_i}\right).
    \end{aligned}
    \end{equation}

    Setting $\eta\rho\lambda_{\text{min}}^2/L + (1+\mu)\eta \lambda_{\text{min}} - (1+\mu) \geq 1$, it holds that 
    \begin{equation}
        \eta\rho\lambda_{\text{min}}^2/L + \eta \lambda_{\text{min}} - 1 \geq 1 + \mu - \mu\eta\lambda_{\text{min}}.
    \end{equation}
    
    Further set $\mu - \mu\eta\lambda_{\text{min}} \geq \mu$, i.e., $\lambda_{\text{min}} \leq 0$, we have that 
    \begin{equation}
        \eta\rho\lambda_{\text{min}}^2/L + \eta \lambda_{\text{min}} - 1 \geq 1+\mu.
    \end{equation}
    
    Now, for $i=1,\cdots, t-2$, $T_i\geq \mu(1+\mu), T_0 \geq 1+\mu, T_{t-1} \geq \mu$. Thus, 
    \begin{equation}
    \begin{aligned}
        & \left(\eta\rho\lambda_{\text{min}}^2/L + \eta \lambda_{\text{min}} - 1\right) ^t - \mu^t\left(1-\eta\lambda_{\text{min}}\right)^t \\
        \geq& (\mu^2+\mu)t-2\mu^2+1.
    \end{aligned}
    \end{equation} % When $t\geq 2$, $2t-1>2$, this inequality is clearly satisfied.

    Solving the system
    \begin{equation}
    \left\{
    \begin{aligned}
        \eta\rho\lambda_{\text{min}}^2/L + (1+\mu)\eta \lambda_{\text{min}} - (1+\mu) &\geq 1, \\
        \lambda_{\text{min}} &\leq 0.
    \end{aligned}
    \right.
    \end{equation}
    yields $\lmin\leq -\frac{L}{\eta\rho}\left(\frac{1+\mu}{2}\eta+\sqrt{\left(\frac{1+\mu}{2}\right)^2\eta^2+\frac{\left(2+\mu\right)\eta\rho}{L}}\right)$, validating the theorem.
\end{proof}

\subsection{Proof of Prop.\ref{prop:scale-of-ehead-and-etail} and Prop.\ref{prop:direction-of-ehead-and-etail}}
\label{app:ehaed-etail-direction}

Here we re-state Prop.\ref{prop:scale-of-ehead-and-etail} and Prop.\ref{prop:direction-of-ehead-and-etail}, and provide their proofs respectively.

\begin{proposition} % \label{prop:scale-of-ehead-and-etail}
    Assume that the samples are drawn from i.i.d distribution, let $m=\lvert S_{\text{head}}\rvert, n=\lvert S_{\text{tail}}\rvert$
    then only if $\frac{\xi_{head}}{\xi_{tail}} = \Theta(\frac{n^2}{m^2})$, we have 
    \begin{equation}
        \frac{\mathbb{E}_{head}[\lVert\nabla \mathcal{L}(\bw)\rVert^2]}{\mathbb{E}_{tail}[\lVert\nabla \mathcal{L}(\bw)\rVert^2]} = \Theta(1).
    \end{equation}
    where 
    \begin{align*}
    &\xi_{head} = \|\bm{\mu}^{(1)}\|^2 + \sum_{j=1}^d \sigma_j^{(1)^2},~\xi_{tail} = \|\bm{\mu}^{(2)}\|^2 + \sum_{j=1}^d \sigma_j^{(2)^2}
    \end{align*}
$\bm{\mu}^{(1)},\bm{\mu}^{(2)}$ are the mean of instance gradient $\nabla \li$ for head/tail class, $\sigma_j^{(1)^2},\sigma_j^{(2)^2} $ are the variance of the j-th dimension of the instance gradient $\nabla \li$ for head/tail class.
\end{proposition}

\begin{proof}
    Since $(\boldsymbol{x}_i, y_i) \in S$ are sampled i.i.d, so will $\nabla \ell_i(\bm{w})$. For the head class distribution, we have
    \begin{align}
        \expe_{head}\left[ \nabla \li \right] = \bm{\mu}^{(1)}, ~~\bm{Cov}_{head}(\nabla \li) = \bm{\Sigma}^{(1)}
    \end{align}
    Similarly, for the tail distribution, we have:
    \begin{align}
        \expe_{tail}\left[ \nabla \li \right] = \bm{\mu}^{(2)}, ~~\bm{Cov}_{tail}(\nabla \li) = \bm{\Sigma}^{(2)}
    \end{align}
    For the sake of simplicity, denote $\bm{K} = \nabla \li$, we have:
    \begin{align*}
        \expe_{head}\left[ \bm{K}^\top \bm{K} \right] &=  \expe_{head}\left[ tr\left(\bm{K}^\top \bm{K} \right) \right] \\
        & = \expe_{head}\left[ tr\left(\bm{K} \bm{K}^\top \right) \right] \\
        & = tr\left(\expe_{head}\left[ \bm{K} \bm{K}^\top \right]\right) \\ 
        & = tr\left(  \bm{\mu}^{(1)} \bm{\mu}^{(1)^\top} \right) +  tr\left(  \bm{\Sigma}^{(1)} \right) \\
        & = \|\bm{\mu}^{(1)}\|^2 + \sum_{j=1}^d \sigma_j^{(1)^2} = \xi_{head}
    \end{align*}
    
Similarly, we have:
\begin{align*}
    \expe_{tail}\left[ \bm{K}^\top \bm{K} \right] =  \|\bm{\mu}^{(2)}\|^2 + \sum_{j=1}^d \sigma_j^{(2)^2} = \xi_{tail}
\end{align*}

The result follows from the fact that:
\begin{align*}
    \frac{\mathbb{E}_{head}[\lVert\nabla \mathcal{L}(\bw)\rVert^2]}{\mathbb{E}_{tail}[\lVert\nabla \mathcal{L}(\bw)\rVert^2]} = \frac{m^2\cdot \xi_{head}}{n^2 \cdot \xi_{tail}}.
\end{align*}

\end{proof}

\begin{proposition} % \label{prop:direction-of-ehead-and-etail}
    Denote by $\theta_{\text{head}}$ the angle between $\ehead$ and $\boldsymbol{\epsilon}$, $\theta_{\text{tail}}$ the angle between $\etail$ and $\boldsymbol{\epsilon}$, $\psi$ the angle between $\ehead$ and $\etail$. When $\lVert\ehead\rVert\gg\lVert\etail\rVert$, we have
    \begin{equation}
    \begin{aligned}
        &\cos\theta_{\text{head}} \approx 1, \\
        &\cos\theta_{\text{tail}} \approx \cos\psi.
    \end{aligned}
    \end{equation}
\end{proposition}

\begin{proof}
    We have that 
    \begin{equation}
    \begin{aligned}
        \cos{\theta_{\text{head}}} &= \frac{\boldsymbol{\epsilon}_{\text{head}}^{\top}\boldsymbol{\epsilon}}{\lVert\boldsymbol{\epsilon}_{\text{head}}\rVert\lVert\boldsymbol{\epsilon}\rVert} \\
        &= \frac{\lVert\boldsymbol{\epsilon}_{\text{head}}\rVert}{\lVert\boldsymbol{\epsilon}\rVert} + \frac{\boldsymbol{\epsilon}_{\text{head}}^{\top}\boldsymbol{\epsilon}_{\text{tail}}}{\lVert\boldsymbol{\epsilon}_{\text{head}}\rVert\lVert\boldsymbol{\epsilon}\rVert} \\
        &= \frac{\lVert\boldsymbol{\epsilon}_{\text{head}}\rVert}{\lVert\boldsymbol{\epsilon}\rVert} + \frac{\lVert\boldsymbol{\epsilon}_{\text{tail}}\rVert}{\lVert\boldsymbol{\epsilon}\rVert}\cos{\psi}\approx 1.
    \end{aligned}
    \end{equation}
    
    Similarly, 
    \begin{equation}
        \cos{\theta_{\text{tail}}} = \frac{\lVert\boldsymbol{\epsilon}_{\text{tail}}\rVert}{\lVert\boldsymbol{\epsilon}\rVert} + \frac{\lVert\boldsymbol{\epsilon}_{\text{head}}\rVert}{\lVert\boldsymbol{\epsilon}\rVert}\cos{\psi}\approx \cos{\psi}.
    \end{equation}
\end{proof}

\subsection{Proof of Prop.\ref{prop:imbsam-limitation}}
\label{app:proof-of-imbsam-limitation}

Proof of Prop.\ref{prop:imbsam-limitation} is similar to Thm.\ref{thm:projection-onto-negative-curvature}. We first re-state Prop.\ref{prop:imbsam-limitation} here.

\begin{proposition}
    Consider a special case where $\mathcal{H}_{\text{tail}} = r\mathcal{H}, \mathcal{H}_{\text{head}}=(1-r)\mathcal{H}$ with $r=\frac{\lvert S_{\text{tail}}\rvert}{\lvert S\rvert}\approx 0$. Let $\bw_t'$ denote the weight obtained after optimizing for $t$ steps by ImbSAM. For any $\mu\geq 1$ and $t\geq 2$, we have that:

    (1) when $\lvert\langle \vmin,\bw_t'-\bw_0\rangle\rvert < \mu\lvert\langle\vmin, \bw_t-\bw_0\rangle\rvert$, as long as $\lmin\leq-\frac{L}{\eta\rho}\bigg(\frac{1+\mu}{2}\eta+$ $\sqrt{\left(\frac{1+\mu}{2}\right)^2\eta^2+\frac{(2+\mu)\eta\rho}{L}}\bigg)$, it holds that $\lvert\langle\vmin, \btw_t-\bw_0\rangle\rvert\geq \mu\lvert\langle\vmin,\bw_t-\bw_0\rangle\rvert$;

    (2) when $\lvert\langle \vmin,\btw_t-\bw_0\rangle\rvert < \mu\lvert\langle\vmin, \bw_t-\bw_0\rangle\rvert$, it also holds that $\lvert\langle \vmin,\bw_t'-\bw_0\rangle\rvert < \mu\lvert\langle\vmin, \bw_t-\bw_0\rangle\rvert$.
\end{proposition}

Before proving Prop.\ref{prop:imbsam-limitation}, we first prove the following lemma:

\begin{lemma}\label{lem:imbsam-limitation}
    Under the conditions of Prop.\ref{prop:imbsam-limitation}, when $\lmin\leq-\frac{L}{\eta\rho r}\left(\frac{1+\mu}{2}\eta+\sqrt{\left(\frac{1+\mu}{2}\right)^2\eta^2+\frac{(2+\mu)\eta\rho r}{L}}\right)$, it holds that 
    \begin{equation}
        \lvert\langle \vmin,\bw_t'-\bw_0\rangle\rvert \geq \mu\lvert\langle\vmin, \bw_t-\bw_0\rangle\rvert.
    \end{equation}
\end{lemma}

\begin{proof}
    Following the proof of Lem.\ref{lem:sam-vmin-multiply-nabla-l}, we have that 
    \begin{equation}
    \begin{aligned}
        \nabla \mathcal{L}^{\text{ImbSAM}}(\bw_t') =& \nabla \Lhead(\bw_t') + \nabla \Ltail(\bw_t'+\be_{\text{tail},t}) \\
        =& \nabla \Lhead(\bw_0') + \mathcal{H}_{\text{head}}(\bw_t'-\bw_0) +\\
        &\nabla \Ltail(\bw_0') + \mathcal{H}_{\text{tail}}(\bw_t'+\be_{\text{tail},t}-\bw_0)\\
        =& \nabla \mathcal{L}(\bw_0') + \mathcal{H}(\bw_t'-\bw_0) + \mathcal{H}_{\text{tail}}\be_{\text{tail},t}
    \end{aligned}
    \end{equation}

    Consider iterative updates:
    \begin{equation}
    \begin{aligned}
        \bw_t'-\bw_0 =& \bw_{t-1}' - \eta\nabla\mathcal{L}^{\text{ImbSAM}}(\bw_{t-1}') - \boldsymbol{w}_0 \\
        =& \bw_{t-1}' - \boldsymbol{w}_0 - \eta(\nabla \mathcal{L}(\bw_0') + \mathcal{H}(\bw_{t-1}'-\bw_0) + \\
        & \mathcal{H}_{\text{tail}}\be_{\text{tail},t-1}) \\
        =& (\mathbf{I}-\eta\mathcal{H})(\bw_{t-1}'-\boldsymbol{w}_0) - \\
        & \eta\left(\nabla \mathcal{L}(\boldsymbol{w}_0)+\mathcal{H}_{\text{tail}}\be_{\text{tail},t-1}\right) \\
        =& -\eta\sum_{i=0}^{t-1}(\mathbf{I}-\eta\mathcal{H})^i\left(\nabla \mathcal{L}(\boldsymbol{w}_0)+\mathcal{H}_{\text{tail}}\be_{\text{tail},t-1-i}\right) \\
        =& -\left[\mathbf{I}-\left(\mathbf{I}-\eta\mathcal{H}\right)^t\right]\mathcal{H}^{-1}\nabla \mathcal{L}(\boldsymbol{w}_0) - \\
        & \eta\sum_{i=0}^{t-1}(\mathbf{I}-\eta\mathcal{H})^i\mathcal{H}_{\text{tail}}\be_{\text{tail},t-1-i}.
    \end{aligned}
    \end{equation}

    This leads to
    \begin{equation}
    \begin{aligned}
        \nabla L(\bw_t') =& (\mathbf{I}-\eta\mathcal{H})^t\nabla L(\boldsymbol{w}_0) - \eta\sum_{i=0}^{t-1}(\mathbf{I}-\\
        & \eta\mathcal{H})^i\mathcal{H}\mathcal{H}_{\text{tail}}\be_{\text{tail},t-1-i},
    \end{aligned}
    \end{equation}
    and
    \begin{equation}
    \begin{aligned}
        \nabla \mathcal{L}(\bw_{t-1}') =& (\mathbf{I}-\eta\mathcal{H})^{t-1}\nabla \mathcal{L}(\boldsymbol{w}_0) - \\
        & \eta\sum_{i=0}^{t-1}(\mathbf{I}-\eta\mathcal{H})^{i-1}\mathcal{H}\mathcal{H}_{\text{tail}}\be_{\text{tail},t-1-i}.
    \end{aligned}
    \end{equation}

    Therefore, the recursive relation is established by combining the above two formulas:
    \begin{equation}
    \begin{aligned}
        \nabla \mathcal{L}(\bw_t') & = (\mathbf{I}-\eta \mathcal{H})\nabla \mathcal{L}(\bw_{t-1}') - \eta\mathcal{H}\mathcal{H}_{\text{tail}}\be_{\text{tail},t-1} \\
        & = (\mathbf{I}-\eta \mathcal{H})\nabla \mathcal{L}(\bw_{t-1}') - \eta r\mathcal{H}^2\be_{\text{tail},t-1}.
    \end{aligned}
    \end{equation}

    This recursion yields:
    \begin{equation}
    \begin{aligned}
        \langle \boldsymbol{v}_{\text{min}}, \nabla \mathcal{L}(\bw_t')\rangle =& \left(1-\eta\lambda_{\text{min}}-\frac{\eta r^2 \rho\lambda_{\text{min}}^2}{\lVert\nabla \Ltail(\bw_{t-1}')\rVert}\right)\cdot \\
        &\langle \boldsymbol{v}_{\text{min}}, \nabla \mathcal{L}(\bw_{t-1}')\rangle.
    \end{aligned}
    \end{equation}

    We know that 
    \begin{equation}
    \begin{aligned}
        \lmin\leq& -\frac{L}{\eta\rho r}\left(\frac{1+\mu}{2}\eta+\right.\\
        &\left.\sqrt{\left(\frac{1+\mu}{2}\right)^2\eta^2+\frac{\left(2+\mu\right)\eta\rho r}{L}}\right),
    \end{aligned}
    \end{equation}
    therefore, because $\mu\geq 1$, it holds that 
    \begin{equation}
        \lambda_{\text{min}}\leq -\frac{L}{\eta \rho r}\left(\eta + \sqrt{\eta^2+\frac{3\eta \rho r}{L}}\right).
    \end{equation}
    It follows that 
    \begin{equation}
    \begin{aligned}
        & \eta\rho r^2 \lambda_{\text{min}}^2 / \lVert\nabla \Ltail(\bw_{t-1}')\rVert + \eta\lambda_{\text{min}} - 1 \\
        \geq& \eta\rho r \lambda_{\text{min}}^2 / L + \eta\lambda_{\text{min}} - 1 \\
        \geq& 2 - \eta\lambda_{\text{min}} > 0.
    \end{aligned}
    \end{equation}
    Therefore
    \begin{equation}
    \begin{aligned}
        \lvert\langle \boldsymbol{v}_{\text{min}}, \nabla \mathcal{L}(\bw_t')\rangle\rvert
        &= \left(\frac{\eta\rho r^2 \lambda_{\text{min}}^2}{\lVert\nabla \Ltail(\bw_{t-1}')\rVert} + \eta\lambda_{\text{min}} - 1\right) \\
        & \lvert\langle \boldsymbol{v}_{\text{min}}, \nabla \mathcal{L}(\bw_{t-1}')\rangle\rvert \\
        &\geq \left(\eta\rho r \lambda_{\text{min}}^2 / L + \eta\lambda_{\text{min}} - 1\right) \\
        &\lvert\langle \boldsymbol{v}_{\text{min}}, \nabla \mathcal{L}(\bw_{t-1}')\rangle\rvert.
    \end{aligned}
    \end{equation}

    Subsequently, replacing all $L$ with $L/r$ in the proof of Thm.\ref{thm:projection-onto-negative-curvature}, we come to the following conclusion
    \begin{equation}
    \begin{aligned}
        & \lvert\langle \boldsymbol{v}_{\text{min}}, \widetilde{\boldsymbol{w}}_t-\boldsymbol{w}_0\rangle\rvert\\
        \geq& \mu\lvert\langle \boldsymbol{v}_{\text{min}}, \boldsymbol{w}_t-\boldsymbol{w}_0\rangle\rvert + \\
        & \left((\mu^2+\mu)t-2\mu^2-\mu\right)\frac{\lvert\langle\vmin, \nabla \mathcal{L}(\bw_0)\rangle\rvert}{\lvert\lmin\rvert}.
    \end{aligned}
    \end{equation}

    For $t\geq 2$, we have that $(\mu^2+\mu)t-2\mu^2-\mu \geq \mu > 0$, therefore we have $\lvert\langle \boldsymbol{v}_{\text{min}}, \widetilde{\boldsymbol{w}}_t-\boldsymbol{w}_0\rangle\rvert\geq\mu\lvert\langle \boldsymbol{v}_{\text{min}}, \boldsymbol{w}_t-\boldsymbol{w}_0\rangle\rvert$, concluding our proof.
\end{proof}

Now we are ready to prove Prop.\ref{prop:imbsam-limitation}.

\begin{proof}
    Note that the contra-positive of a statement is equivalent to the original statement. Therefore, according to Lem.\ref{lem:imbsam-limitation}, when $\lvert\langle \vmin,\bw_t'-\bw_0\rangle\rvert < \mu\lvert\langle\vmin, \bw_t-\bw_0\rangle\rvert$, we have that 
    \begin{equation}
    \begin{aligned}
        \lmin >& -\frac{L}{\eta\rho r}\left(\frac{1+\mu}{2}\eta+\right.\\
        &\left.\sqrt{\left(\frac{1+\mu}{2}\right)^2\eta^2+\frac{(2+\mu)\eta\rho r}{L}}\right).
    \end{aligned}
    \end{equation}

    Define 
    \begin{equation}
    \begin{aligned}
        \lambda(r) :=& -\frac{L}{\eta\rho r}\left(\frac{1+\mu}{2}\eta+\right.\\
        &\left.\sqrt{\left(\frac{1+\mu}{2}\right)^2\eta^2+\frac{(2+\mu)\eta\rho r}{L}}\right).
    \end{aligned}
    \end{equation}
    To prove (2), we only have to show that 
    \begin{equation}
    \begin{aligned}
        &\lambda(r) \\
        <& -\frac{L}{\eta\rho}\left(\frac{1+\mu}{2}\eta+\sqrt{\left(\frac{1+\mu}{2}\right)^2\eta^2+\frac{(2+\mu)\eta\rho}{L}}\right) \\
        =& \lambda(1),
    \end{aligned}
    \end{equation}
    so that the condition $\lmin \leq \lambda(1)$ is reachable. 

    Because $0<r<1$, we only have to show that $\frac{\partial\lambda(r)}{\partial r} > 0$, i.e., we have to show that
    \begin{equation}
    \begin{aligned}
        & \frac{L}{\eta\rho}\left(\frac{1}{r^2}\left(\frac{1+\mu}{2}\eta+\sqrt{\left(\frac{1+\mu}{2}\right)^2\eta^2+\frac{(2+\mu)\eta\rho r}{L}}\right) - \right.\\
        & \left.\frac{1}{r}\cdot \frac{(2+\mu)\eta\rho / L}{2 \sqrt{\left(\frac{1+\mu}{2}\right)^2\eta^2+\frac{(2+\mu)\eta\rho r}{L}}}\right) > 0.
    \end{aligned}
    \end{equation}

    Multiplying each side with $\frac{\eta\rho}{L}r^2\sqrt{\left(\frac{1+\mu}{2}\right)^2\eta^2+\frac{(2+\mu)\eta\rho r}{L}}\\>0$, what we have to show becomes
    \begin{equation}
    \begin{aligned}
        & \frac{1+\mu}{2}\eta\sqrt{\left(\frac{1+\mu}{2}\right)^2\eta^2+\frac{(2+\mu)\eta\rho r}{L}} +\\
        & \left(\frac{1+\mu}{2}\right)^2\eta^2+\frac{(2+\mu)\eta\rho r}{L} - \frac{(2+\mu)\eta\rho r}{2L} > 0,
    \end{aligned}
    \end{equation}
    which is equivalent to
    \begin{equation}
    \begin{aligned}
        & \frac{1+\mu}{2}\eta\sqrt{\left(\frac{1+\mu}{2}\right)^2\eta^2+\frac{(2+\mu)\eta\rho r}{L}} +\\
        & \left(\frac{1+\mu}{2}\right)^2\eta^2+\frac{(2+\mu)\eta\rho r}{2L} > 0.
    \end{aligned}
    \end{equation}

    Clearly last inequality is true. Therefore we prove (1).

    For (2), when $\lvert\langle \vmin,\btw_t-\bw_0\rangle\rvert < \mu\lvert\langle\vmin, \bw_t-\bw_0\rangle\rvert$, it holds that 
    \begin{equation}
    \begin{aligned}
        \lmin >& -\frac{L}{\eta\rho}\\
        & \left(\frac{1+\mu}{2}\eta+\sqrt{\left(\frac{1+\mu}{2}\right)^2\eta^2+\frac{(2+\mu)\eta\rho}{L}}\right) \\
        =& \lambda(1) > \lambda(r) \\
        =& -\frac{L}{\eta\rho r}\\
        & \left(\frac{1+\mu}{2}\eta+\sqrt{\left(\frac{1+\mu}{2}\right)^2\eta^2+\frac{(2+\mu)\eta\rho r}{L}}\right).
    \end{aligned}
    \end{equation}

    Then by Lem.\ref{lem:imbsam-limitation}, it also holds that $\lvert\langle \vmin,\bw_t'-\bw_0\rangle\rvert < \mu\lvert\langle\vmin, \bw_t-\bw_0\rangle\rvert$.
\end{proof}

\subsection{Proof of Thm.\ref{thm:saddle-generalization}}

We first re-state Thm.\ref{thm:saddle-generalization} here, then provide its proof.

\begin{theorem}
    Suppose loss function $\mathcal{L}$ is upper bounded by $M$. For any $\rho>0$ and any distribution $\mathcal{D}$, with probability at least $1-\delta$ over the choice of the training set $S\sim\mathcal{D}$, there exists a constant $0\leq c\leq 1$ such that
    \begin{equation}
    \begin{aligned}
        & \mathcal{L}_{\mathcal{D}}(\bw)\leq \mathcal{L}_S(\bw) + \rho^2\sqrt{\frac{d}{4\pi}}\max_{\lVert\be\rVert\leq \rho}\lmax\left(\nabla^2L\left(\bw+c\be\right)\right) \\
        & + \frac{M}{\sqrt{n}} + \left[\left(\frac{1}{4}d\log \left(1+\frac{\lVert\bw\rVert^2(\sqrt{d}+\sqrt{\log n})^2}{d\rho^2}\right)+\frac{1}{4}\right.\right. \\
        & \left.\left.+\log\frac{n}{\delta} +2\log(6n+3d)\right)/\left(n-1\right)\right]^{\frac{1}{2}},
    \end{aligned}
    \end{equation}
    where $d$ is the number of parameters and $n$ is the size of training set $S$.
\end{theorem}

\begin{proof}
    Fix $\sigma = \rho/(\sqrt{d}+\sqrt{\log n})$. Following the proof of Thm.1 in \cite{foret2020sharpness}, we have that
    \begin{equation}\label{eq:saddle-generl-1}
    \begin{aligned}
        &\mathbb{E}_{\be\sim\mathcal{N}(\boldsymbol{0}, \sigma^2\mathbf{I})}\left[\mathcal{L}_{\mathcal{D}}(\bw+\be)\right]\leq \mathbb{E}_{\be\sim\mathcal{N}(\boldsymbol{0}, \sigma^2\mathbf{I})}\left[\mathcal{L}_S(\bw+\be)\right] +\\
        & \left[\left(\frac{1}{4}d\log \left(1+\frac{\lVert\bw\rVert^2(\sqrt{d}+\sqrt{\log n})^2}{d\rho^2}\right)+\frac{1}{4}\right.\right. \\
        & \left.\left.+\log\frac{n}{\delta} +2\log(6n+3d)\right)/\left(n-1\right)\right]^{\frac{1}{2}}
    \end{aligned}
    \end{equation}
    
    Since $\be\sim\mathcal{N}(\boldsymbol{0}, \sigma^2\mathbf{I})$, $\lVert\be\rVert^2/\sigma^2$ follows chi-squared distribution. By Lem.1 in \cite{laurent2000adaptive}, it holds that for any $t>0$, 
    \begin{equation}
        \mathbb{P}\left(\lVert\be\rVert^2/\sigma^2 - d\geq 2\sqrt{dt}+2t\right)\leq \exp(-t).
    \end{equation}
    
    Then with probability at least $1-1/\sqrt{n}$,
    \begin{equation}
    \begin{aligned}
        \lVert\be\rVert^2 \leq& \sigma^2(d + \sqrt{2d\log n} + \log n)\\
        \leq& \sigma^2\left(\sqrt{d}+\sqrt{\log n}\right)^2 \\
        =& \rho^2.
    \end{aligned}
    \end{equation}
    
    Therefore, 
    \begin{equation}\label{eq:saddle-generl-2}
    \begin{aligned}
        & \mathbb{E}_{\be\sim\mathcal{N}(\boldsymbol{0}, \sigma^2\mathbf{I})}\left[\mathcal{L}_S(\bw+\be)\right] \\
        \leq& \mathbb{E}_{\be\sim\mathcal{N}(\boldsymbol{0}, \sigma^2\mathbf{I})}\left[\mathcal{L}_S(\bw+\be) \mid \lVert\be\rVert\leq\rho\right]\mathbb{P}\left(\lVert\be\rVert\leq\rho\right) + \\
        & \mathbb{E}_{\be\sim\mathcal{N}(\boldsymbol{0}, \sigma^2\mathbf{I})}\left[\mathcal{L}_S(\bw+\be) \mid \lVert\be\rVert>\rho\right]\mathbb{P}\left(\lVert\be\rVert>\rho\right) \\
        \leq& \mathbb{E}_{\be\sim\mathcal{N}(\boldsymbol{0}, \sigma^2\mathbf{I})}\left[\mathcal{L}_S(\bw+\be) \mid \lVert\be\rVert\leq\rho\right] + \frac{M}{\sqrt{n}}.
    \end{aligned}
    \end{equation}
    
    According to mean value theorem, there exist constants $0\leq c'\leq 1$ and $0\leq c''\leq 1$, such that
    \begin{equation}
        \mathcal{L}_S(\bw+\be) = \mathcal{L}_S(\bw) + \be^{\top}\nabla\mathcal{L}_S(\bw+c'\be)
    \end{equation}
    \begin{equation}
        \nabla\mathcal{L}_S(\bw+c'\be) = \nabla\mathcal{L}_S(\bw) + c'\cdot\nabla^2\mathcal{L}_S(\bw+c''c'\be)\be.
    \end{equation}
    
    Combining the above equations, we have that
    \begin{equation}
        \mathcal{L}_S(\bw+\be) = \mathcal{L}_S(\bw) + \be^{\top}\nabla\mathcal{L}_S(\bw)+c'\be^{\top}\nabla^2\mathcal{L}_S(\bw+c\be)\be,
    \end{equation}
    where $c=c''c'$.
    
    Substituting this result back, we have that
    \begin{equation}\label{eq:saddle-generl-3}
    \begin{aligned}
        & \mathbb{E}_{\be\sim\mathcal{N}(\boldsymbol{0}, \sigma^2\mathbf{I})}\left[\mathcal{L}_S(\bw+\be) \mid \lVert\be\rVert\leq\rho\right] \\
        \leq& \mathcal{L}_S(\bw) + \mathbb{E}_{\be\sim\mathcal{N}(\boldsymbol{0}, \sigma^2\mathbf{I})}\left[\be^{\top}\nabla\mathcal{L}_S(\bw) \mid \lVert\be\rVert\leq\rho\right] \\
        & +  \mathbb{E}_{\be\sim\mathcal{N}(\boldsymbol{0}, \sigma^2\mathbf{I})}\left[\be^{\top}\nabla^2\mathcal{L}_S(\bw+c\be)\be \mid \lVert\be\rVert\leq\rho\right]
    \end{aligned}
    \end{equation}
    
    By symmetry, 
    \begin{equation}\label{eq:saddle-generl-4}
        \mathbb{E}_{\be\sim\mathcal{N}(\boldsymbol{0}, \sigma^2\mathbf{I})}\left[\be^{\top}\nabla\mathcal{L}_S(\bw) \mid \lVert\be\rVert\leq\rho\right] = 0.
    \end{equation}
    
    By Rayleigh-Ritz theorem, we have that
    \begin{equation}
    \begin{aligned}
        & \mathbb{E}_{\be\sim\mathcal{N}(\boldsymbol{0}, \sigma^2\mathbf{I})}\left[\be^{\top}\nabla^2\mathcal{L}_S(\bw+c\be)\be \mid \lVert\be\rVert\leq\rho\right] \\
        \leq& \mathbb{E}_{\be\sim\mathcal{N}(\boldsymbol{0}, \sigma^2\mathbf{I})}\left[\lmax\left(\nabla^2\mathcal{L}_S(\bw+c\be)\right)\be^{\top}\be \mid \lVert\be\rVert\leq\rho\right] \\
        \leq& \max_{\lVert\be\rVert\leq \rho}\lmax\left(\nabla^2\mathcal{L}_S(\bw+c\be)\right)\sigma^2\\
        & \mathbb{E}_{\be\sim\mathcal{N}(\boldsymbol{0}, \sigma^2\mathbf{I})}\left[\lVert\be\rVert^2/\sigma^2 \mid \lVert\be\rVert\leq\rho\right] \\
        \leq& \max_{\lVert\be\rVert\leq \rho}\lmax\left(\nabla^2\mathcal{L}_S(\bw+c\be)\right)\sigma^2\mathbb{E}_{z}\left[z \mid z\leq\rho^2/\sigma^2\right], \\
    \end{aligned}
    \end{equation}
    where $z=\lVert\be\rVert^2/\sigma^2$, $z\sim\chi^2(d)$.
    
    Denote by $p(z)$ the probability density function of $z$. By properties of chi-squared distribution, we have that
    \begin{equation}
    \begin{aligned}
        & \mathbb{E}_{z}\left[z \mid z\leq\rho^2/\sigma^2\right] \\
        =& \int_{0}^{\rho^2/\sigma^2}zp(z)dz \\
        =& \frac{1}{2^{d/2}\Gamma(d/2)}\int_{0}^{\rho^2/\sigma^2}z^{d/2}e^{-z/2}dz
    \end{aligned}
    \end{equation}
    
    Solving $\frac{d}{dz}\left(z^{d/2}e^{-z/2}\right)=0$ gives $z=d$. Thus, $z^{d/2}e^{-z/2}$ reaches maximum at $z=d$. Because $\rho^2/\sigma^2 = d+\log n + 2\sqrt{d\log n}>d$, we have that
    \begin{equation}
        \mathbb{E}_{z}\left[z \mid z\leq\rho^2/\sigma^2\right]\leq \frac{\rho^2}{\sigma^2\Gamma(d/2)}\left(\frac{d}{2e}\right)^{\frac{d}{2}}.
    \end{equation}
    
    Using Stirling's formula, which is a rigorous lower bound of Gamma function, we have that
    \begin{equation}
    \begin{aligned}
        \Gamma\left(\frac{d}{2}\right) =& \frac{2}{d}\Gamma\left(\frac{d}{2}+1\right) \\
        \geq& \frac{2\sqrt{\pi d}}{d}\left(\frac{d}{2e}\right)^{\frac{d}{2}} \\
        =& \sqrt{\frac{4\pi}{d}}\left(\frac{d}{2e}\right)^{\frac{d}{2}}.
    \end{aligned}
    \end{equation}
    
    Substituting this result back, we have that
    \begin{equation}
        \mathbb{E}_{z}\left[z \mid z\leq\rho^2/\sigma^2\right]\leq \frac{\rho^2}{\sigma^2}\sqrt{\frac{d}{4\pi}}.
    \end{equation}
    
    Finally,
    \begin{equation}\label{eq:saddle-generl-5}
    \begin{aligned}
        & \mathbb{E}_{\be\sim\mathcal{N}(\boldsymbol{0}, \sigma^2\mathbf{I})}\left[\be^{\top}\nabla^2\mathcal{L}_S(\bw+c\be)\be \mid \lVert\be\rVert\leq\rho\right] \\
        \leq& \rho^2\sqrt{\frac{d}{4\pi}}\max_{\lVert\be\rVert\leq \rho}\lmax\left(\nabla^2\mathcal{L}_S(\bw+c\be)\right).
    \end{aligned}
    \end{equation}
    
    Combining (\ref{eq:saddle-generl-1}), (\ref{eq:saddle-generl-2}), (\ref{eq:saddle-generl-3}), (\ref{eq:saddle-generl-4}) and (\ref{eq:saddle-generl-5}) gives the final bound.
\end{proof}

% -------------------------------------------------------------------

\section{Pseudo-code for SSE-SAM Algorithm}
\label{app:algorithm}

We provide the pseudo-code for our SSE-SAM here in Alg.\ref{alg:sse-sam}.

\begin{algorithm}
    \caption{SSE-SAM Algorithm}
    \label{alg:sse-sam}
    \textbf{Input}: Training set $S=\left\{(\boldsymbol{x}_i,y_i)\right\}_{i=1}^n$, loss function $l$, batch size $b$, learning rate $\eta$, total training epochs $T$, initial value of head class neighborhood size $\rhead$, tail class neighborhood size $\rtail$, stage transitional factor $\gamma$ (where $0<\gamma<1$).\\
    \textbf{Output}: Model weights $\bw^*$ returned by SSE-SAM algorithm.
    \begin{algorithmic}[1]
        \STATE Initialize weights $\bw_0, t=0$;
        \WHILE{$t<T$}
        \IF {$t\geq \gamma T$}
        \STATE $\rhead=0$;
        \ENDIF
        \STATE Sample batch $\mathcal{B}=\left\{(\boldsymbol{x}_i,y_i)\right\}_{i=1}^n$;
        \STATE Split $\mathcal{B}$ into $\mathcal{B}_{\text{head}}$ and $\mathcal{B}_{\text{tail}}$;
        \STATE Compute $\nabla_{\bw}\Lhead(\be_t)$ and $\nabla_{\bw}\Ltail(\bw_t)$;
        \STATE Compute perturbation for head classes $\ehead(\bw_t)=\rhead\nabla_{\bw}\Lhead(\bw_t)/\lVert \nabla_{\bw}L(\bw_t)\rVert$;
        \STATE Compute perturbation for tail classes $\etail(\bw_t)=\rtail\nabla_{\bw}\Ltail(\bw_t)/\lVert \nabla_{\bw}L(\bw_t)\rVert$;
        \STATE Compute head class gradient $\boldsymbol{g}_{\text{head}}=\nabla_{\bw}\Lhead(\bw_t+\ehead(\bw_t))$;
        \STATE Compute tail class gradient $\boldsymbol{g}_{\text{tail}}=\nabla_{\bw}\Ltail(\bw_t+\etail(\bw_t))$;
        \STATE Update weights $\bw_{t+1}=\bw_t-\eta(\boldsymbol{g}_{\text{head}}+\boldsymbol{g}_{\text{tail}})$;
        \STATE $t=t+1$;
        \ENDWHILE
        \STATE \textbf{return} $\bw_t$;
    \end{algorithmic}
\end{algorithm}

% -------------------------------------------------------------------

\section{Experiment Settings}
\label{app:experiment-setting}

\subsection{Baseline Methods}

we use cross-entropy (CE) loss with fine-tuned weight decay\cite{alshammari2022long}, LDAM\cite{cao2019learning}, LA\cite{menon2020long} and VS Loss\cite{kini2021label} as baselines for comparison. Here we briefly describe them.

\textbf{LDAM\cite{cao2019learning}}. LDAM introduces a class-dependent margin for multiple classes to regularize the classes with low frequency more. It results in the following loss function
\begin{equation}
    L_{\text{LDAM}}\left((x,y),f\right) = -\log\frac{e^{f(x)_y-\Delta_y}}{e^{f(x)_y-\Delta_y}+\sum\limits_{j\neq y}e^{f(x)_j}},
\end{equation}
where $\Delta_j=\frac{C}{\lvert S_j\rvert^{1/4}}$ for $j\in \left\{1,\cdots,k\right\}$.

\textbf{LA\cite{menon2020long}}. LA adjusts the predicted logit to achieve Fisher consistent loss, resulting in the following objective
\begin{equation}
    L_{\text{LA}}\left((x,y),f\right) = -w_y\log\frac{e^{f(x)_y+\tau\cdot\log\pi_y}}{\sum\limits_{j=1}^ke^{f(x)_j+\tau\cdot\log\pi_j}},
\end{equation}
where $\pi_j$ are estimates of the class priors $\mathbb{P}(j)$, e.g., the empirical class frequencies on training sample.

\textbf{VS\cite{kini2021label}}. VS loss proposes to unify the multiplicative loss, additive shift and loss re-weighting, and has the following form
\begin{equation}
    L_{\text{VS}}\left((x,y),f\right) = -w_y\log\frac{e^{\gamma_yf(x)_y+\Delta_y}}{\sum\limits_{j=1}^ke^{\gamma_jf(x)_j+\Delta_j}},
\end{equation}
where $\gamma_j$ and $\Delta_j$ denote the multiplicative and additive logit hyperparameters.

\subsection{Experiment Settings}

We train our model on CIFAR-10-LT and CIFAR-100-LT datasets\cite{krizhevsky2009learning}. CIFAR-10-LT and CIFAR-100-LT are artificially truncated datasets from balanced datasets CIFAR-10 and CIFAR-100, with class number from first to last class following exponential distribution. 

When splitting head classes and tail classes, we set $\eta_{\text{thres}}=100$ on CIFAR-10-LT and $1000$ on CIFAR-10-LT according to the following equation:
\begin{equation}\label{eq:eta-thres}
\left\{
\begin{aligned}
    & (\boldsymbol{x}_i, y_i)\in S_{\text{head}}, & \text{ if } \lvert S_{y_i}\rvert >\eta_{\text{thres}}, \\
    & (\boldsymbol{x}_i, y_i)\in S_{\text{tail}}, & \text{ if } \lvert S_{y_i}\rvert \leq\eta_{\text{thres}}. \\
\end{aligned}
\right.
\end{equation}

Model evaluation is performed on balanced test sets from CIFAR-10 and CIFAR-100 directly. Following \cite{liu2019large}, Evaluation results are presented using four metrics: \textit{overall accuracy} ($=\frac{1}{K}\sum \text{acc}_k$, where $\text{acc}_k$ is model accuracy on k-th class), \textit{Many classes} (number of training data $>1000$ in CIFAR-10, $>100$ in CIFAR-100), \textit{Medium classes} ($200\sim 1000$ in CIFAR-10, $20\sim 100$ in CIFAR-100) and \textit{Few classes} ($<200$ in CIFAR-10, $<20$ in CIFAR-100).

We use ResNet32\cite{cui2019class} as our model structure. We set SGD optimizer with momentum $0.9$ as the base optimizer and train models for $500$ epochs with batch size $64$. As introduced in last subsection, we use cross-entropy (CE) with fine-tuned weight decay\cite{alshammari2022long}, LDAM\cite{cao2019learning}, LA\cite{menon2020long} and VS Loss\cite{kini2021label} as baselines for comparison. All experiments are conducted on NVIDIA GeForce RTX 3090 with 16GB of memory. The operating system on which we perform our experiments is Ubuntu 20.04.3 LTS (GNU/Linux 5.15.0-107-generic x86\_64). We run each experiment for three times under different random seeds due to limited resources. For experiment runtime, our comparative experiments all run for 290 to 300 minutes, Hessian analysis experiments for overall classes run for around 14 hours, Hessian analysis experiments for head and tail classes run for 30 to 50 minutes.

% -------------------------------------------------------------------

\section{Additional Experiments}
\label{app:additional-experiments}

\subsection{Test Accuracy Dynamics on CIFAR-10-LT}
\label{app:test-acc-cifar10}

Despite test accuracy dynamics on CIFAR-100-LT, we also run experiments on CIFAR-10-LT. Fig.\ref{fig:acc-10} shows our experiment results. We note that these results exhibit similar trends demonstrated in Fig.\ref{fig:acc-100}, aligning with our expectation.

\begin{figure}[!htbp]
    \centering
    \begin{subfigure}{0.23\textwidth}
        \centering
        \includegraphics[width=\textwidth]{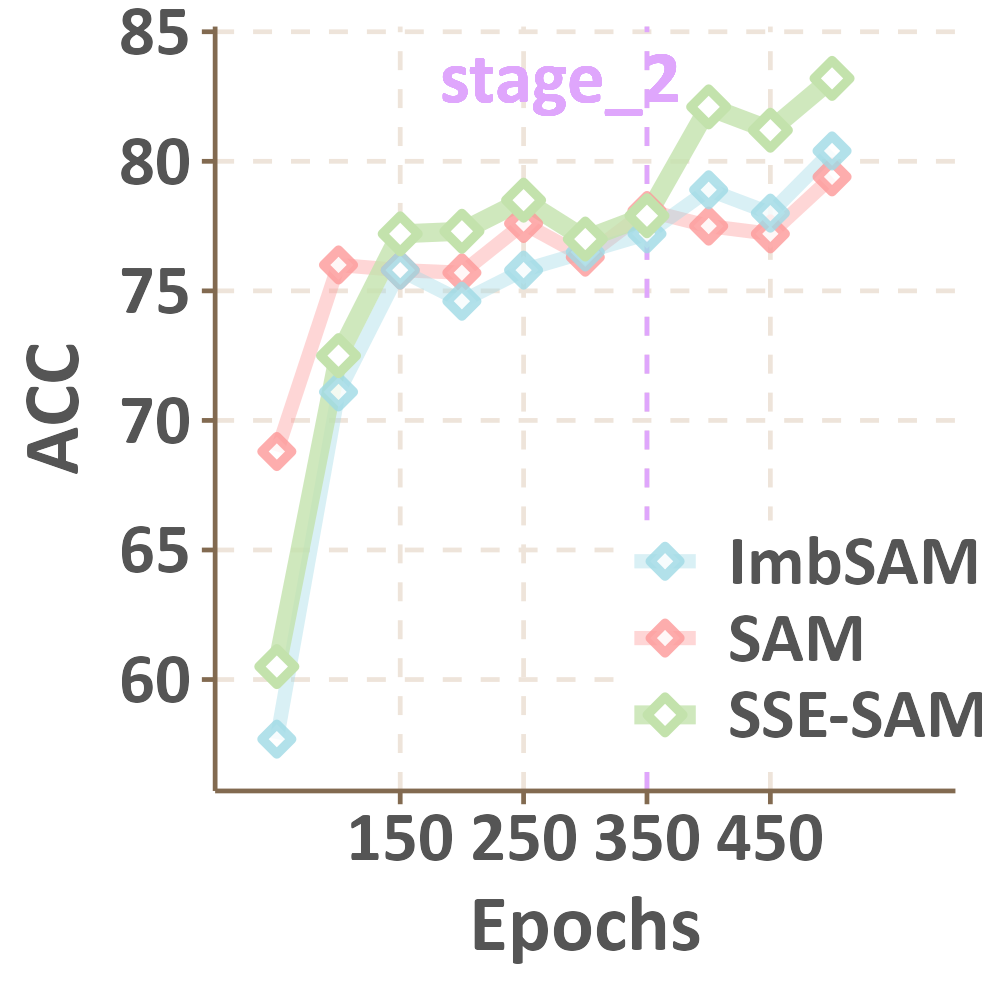}
        \caption{Overall}
        \label{fig:acc-all-10}
    \end{subfigure}
    \hfill
    \begin{subfigure}{0.23\textwidth}
        \centering
        \includegraphics[width=\textwidth]{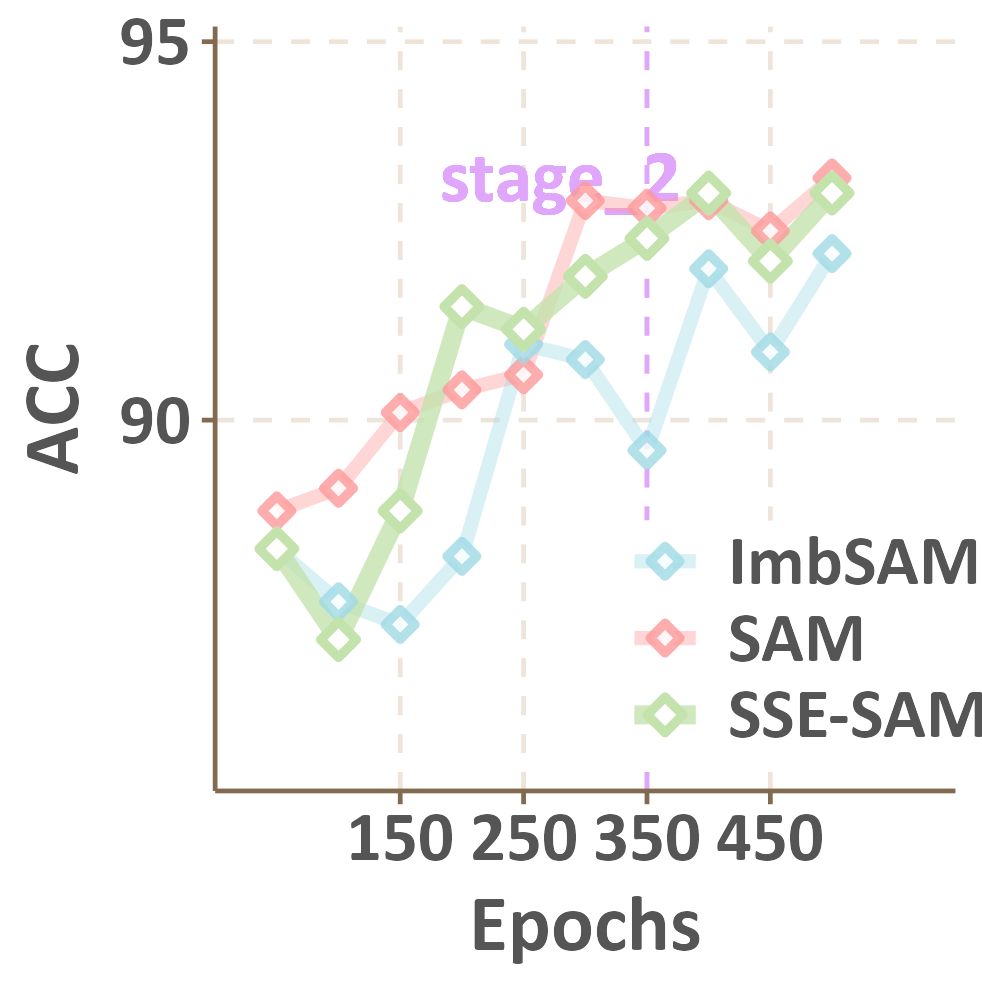}
        \caption{Head classes}
        \label{fig:acc-many-10}
    \end{subfigure}
    \hfill
    \begin{subfigure}{0.23\textwidth}
        \centering
        \includegraphics[width=\textwidth]{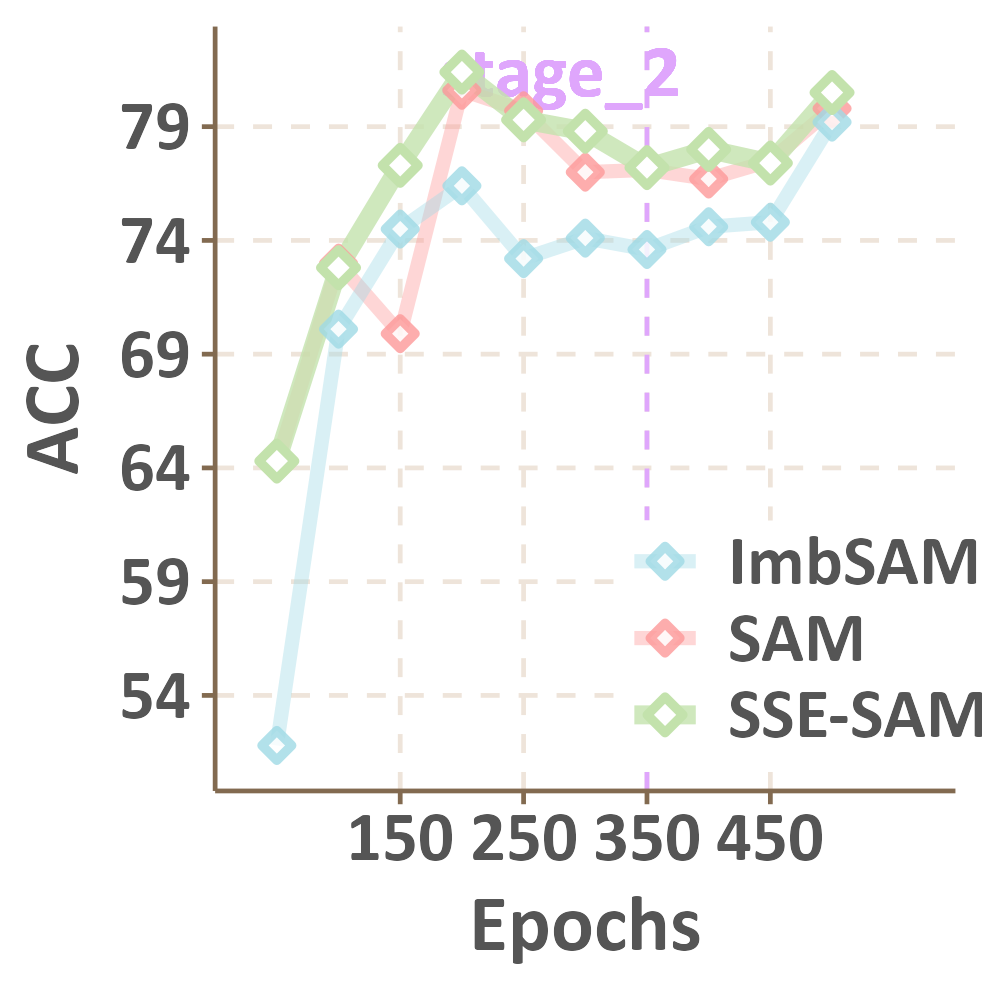}
        \caption{Medium classes}
        \label{fig:acc-med-10}
    \end{subfigure}
    \hfill
    \begin{subfigure}{0.23\textwidth}
        \centering
        \includegraphics[width=\textwidth]{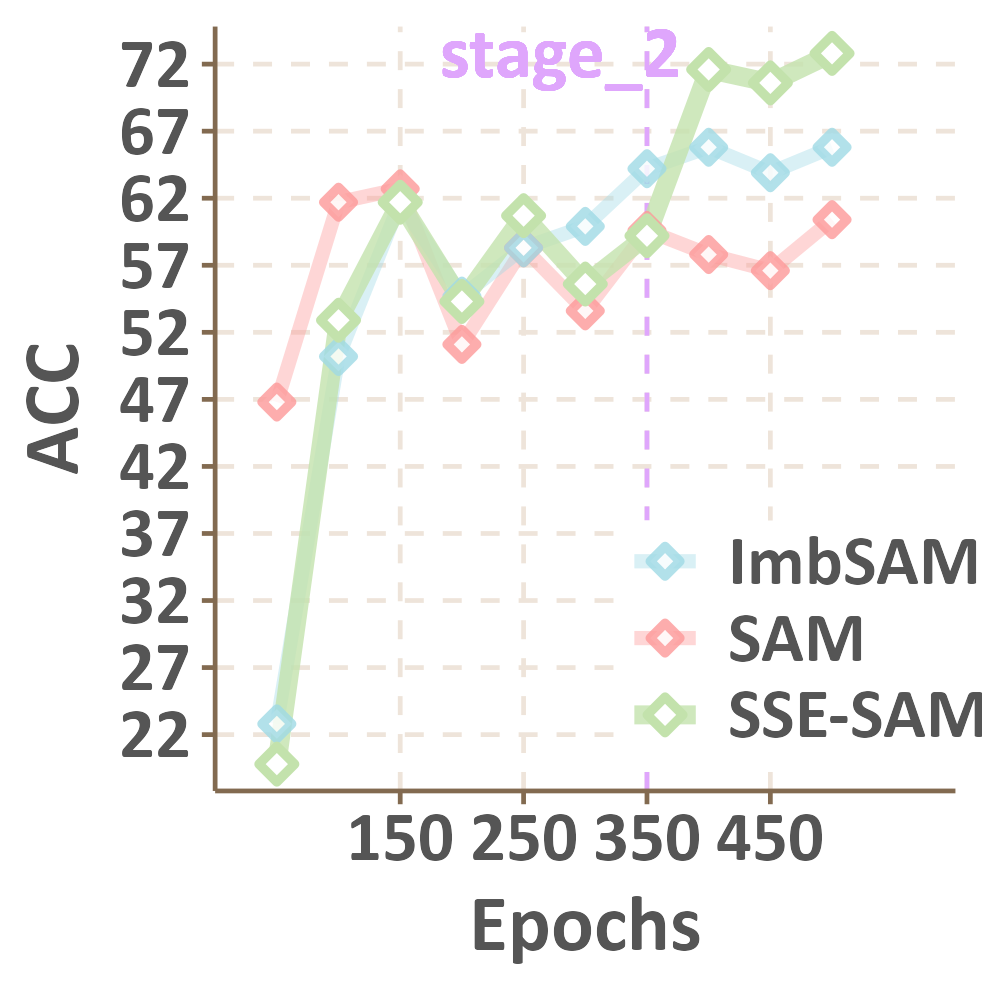}
        \caption{Tail classes}
        \label{fig:acc-few-10}
    \end{subfigure}
    \caption{\textbf{Test accuracy on CIFAR-10-LT across different classes.}}
    \label{fig:acc-10}
\end{figure}

\subsection{Ablation Studies}
\label{app:ablation-studies}

We also perform some ablation studies on hyperparameters. 

First, the SSE-SAM objective for the first stage is
\begin{equation}
    \mathcal{L}^{\text{SSE-SAM}}_1(\bw) = \Lhead(\bw+\ehead) + \Ltail(\bw+\etail).
\end{equation}

Following \cite{foret2020sharpness, zhou2023imbsam}, we set $\rhead = 0.05$. As we have mentioned in previous section discussing SSE-SAM algorithm, it is natural for head classes to escape from saddles while difficult for tail classes. From Thm.\ref{thm:projection-onto-negative-curvature}, we know that an increased $\rho$ will lead to a lower upper bound on $\lmin$, therefore we fine-tune $\rtail$ by increasing its value from $0.05$ to $0.15$. The results are presented in Fig.\ref{fig:ablt-rtail}. As illustrated in Fig.\ref{fig:ablt-rtail-overall}, the model's overall performance peaks at $\rtail=0.10$ and then declines. A similar trend is observed for the medium and tail classes (Figs.\ref{fig:ablt-rtail-med} and \ref{fig:ablt-rtail-tail}), while the head classes remain nearly unaffected (Fig.\ref{fig:ablt-rtail-head}). This confirms our hypothesis that increasing $\rtail$ within a certain range can enhance performance on tail classes.

\begin{figure}[htbp!]
    \centering
    \begin{subfigure}{0.23\textwidth}
        \centering
        \includegraphics[width=\textwidth]{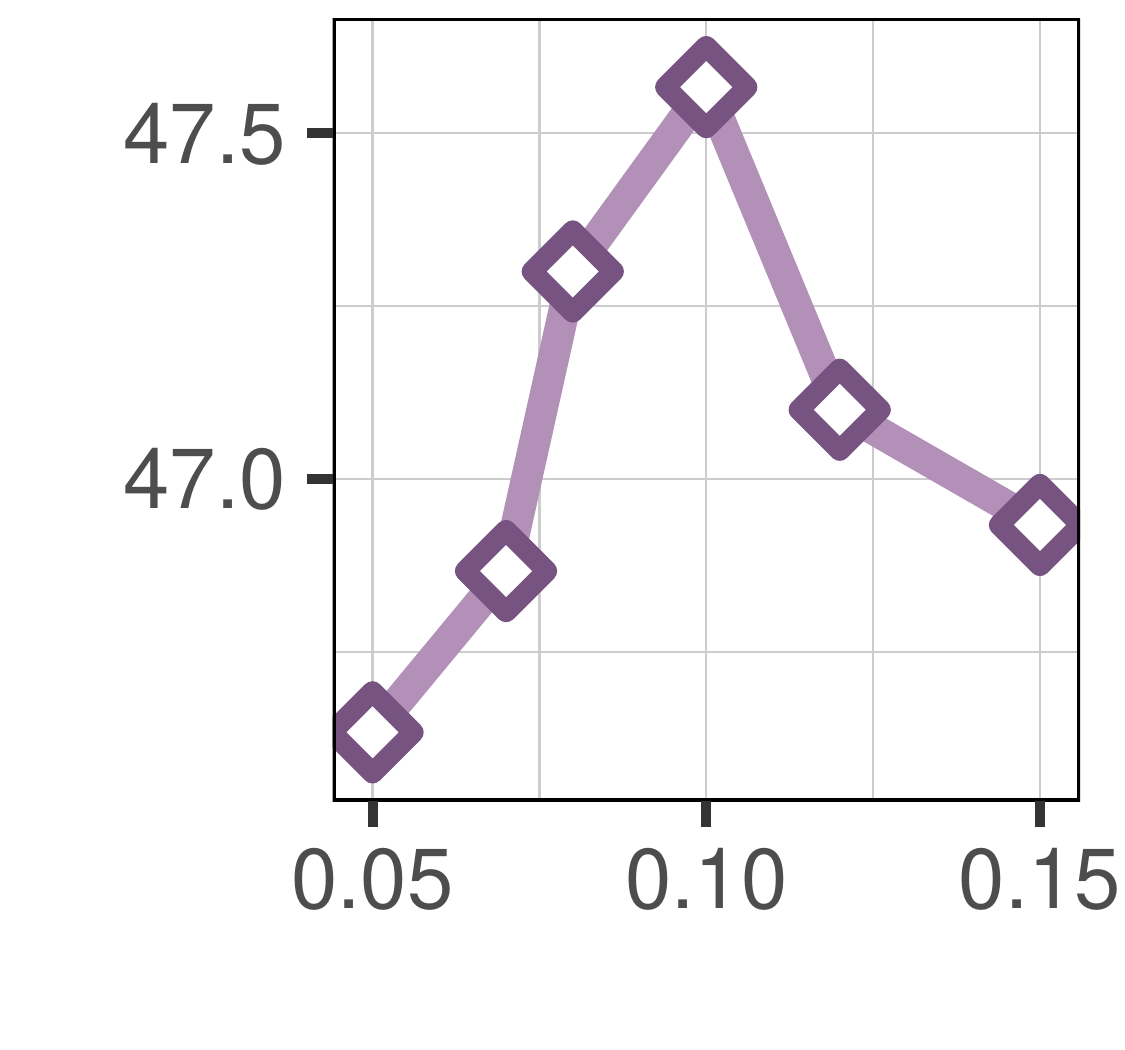}
        \caption{Overall}
        \label{fig:ablt-rtail-overall}
    \end{subfigure}
    \hfill
    \begin{subfigure}{0.22\textwidth}
        \centering
        \includegraphics[width=\textwidth]{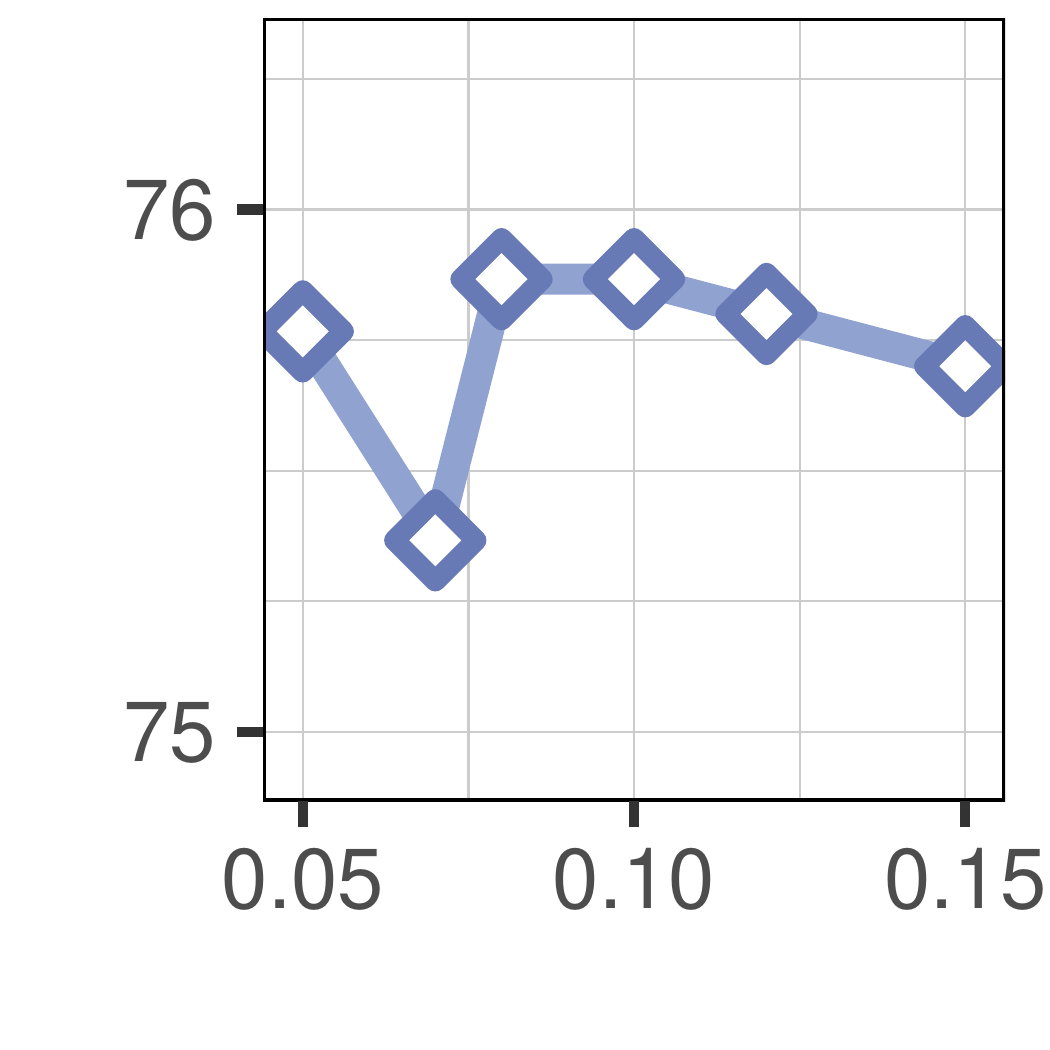}
        \caption{Head classes}
        \label{fig:ablt-rtail-head}
    \end{subfigure}
    \hfill
    \begin{subfigure}{0.22\textwidth}
        \centering
        \includegraphics[width=\textwidth]{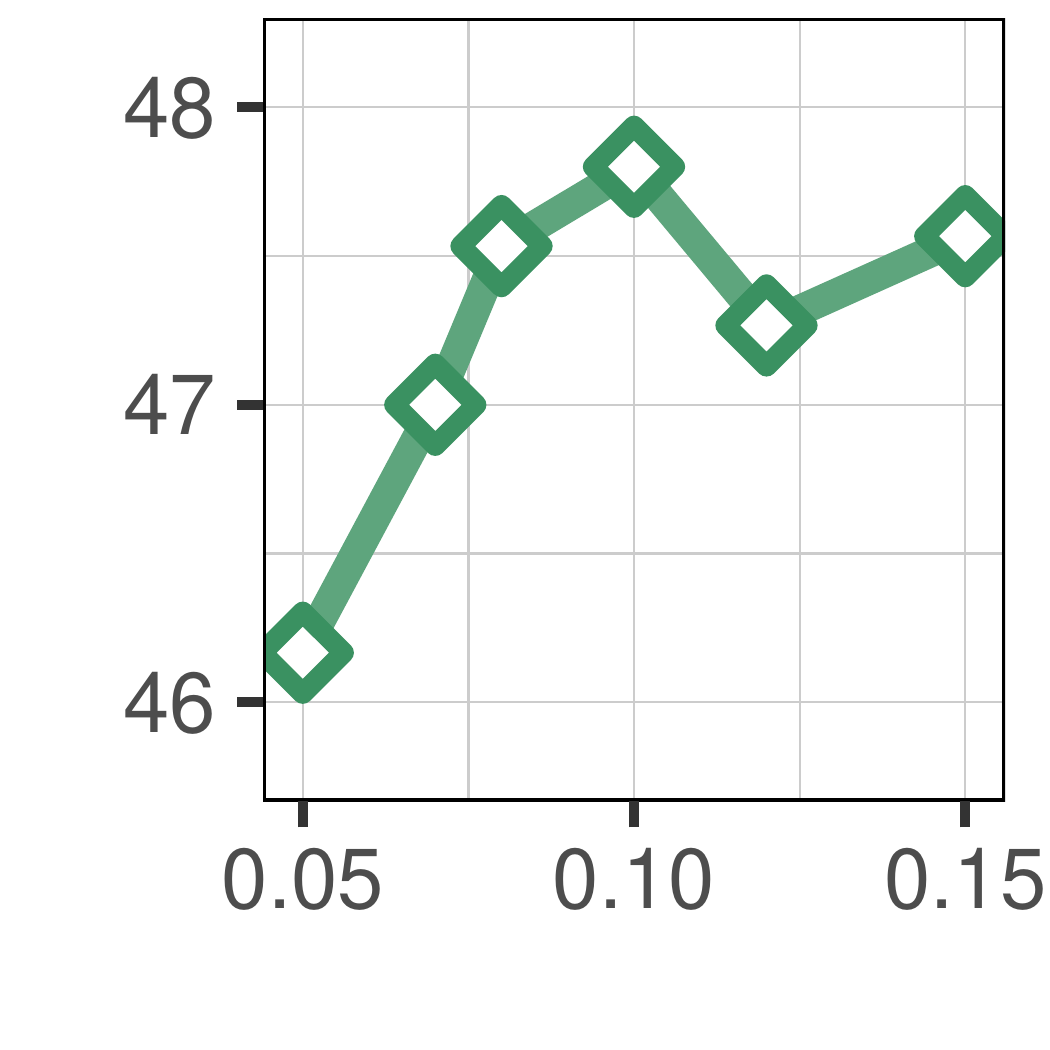}
        \caption{Medium classes}
        \label{fig:ablt-rtail-med}
    \end{subfigure}
    \hfill
    \begin{subfigure}{0.22\textwidth}
        \centering
        \includegraphics[width=\textwidth]{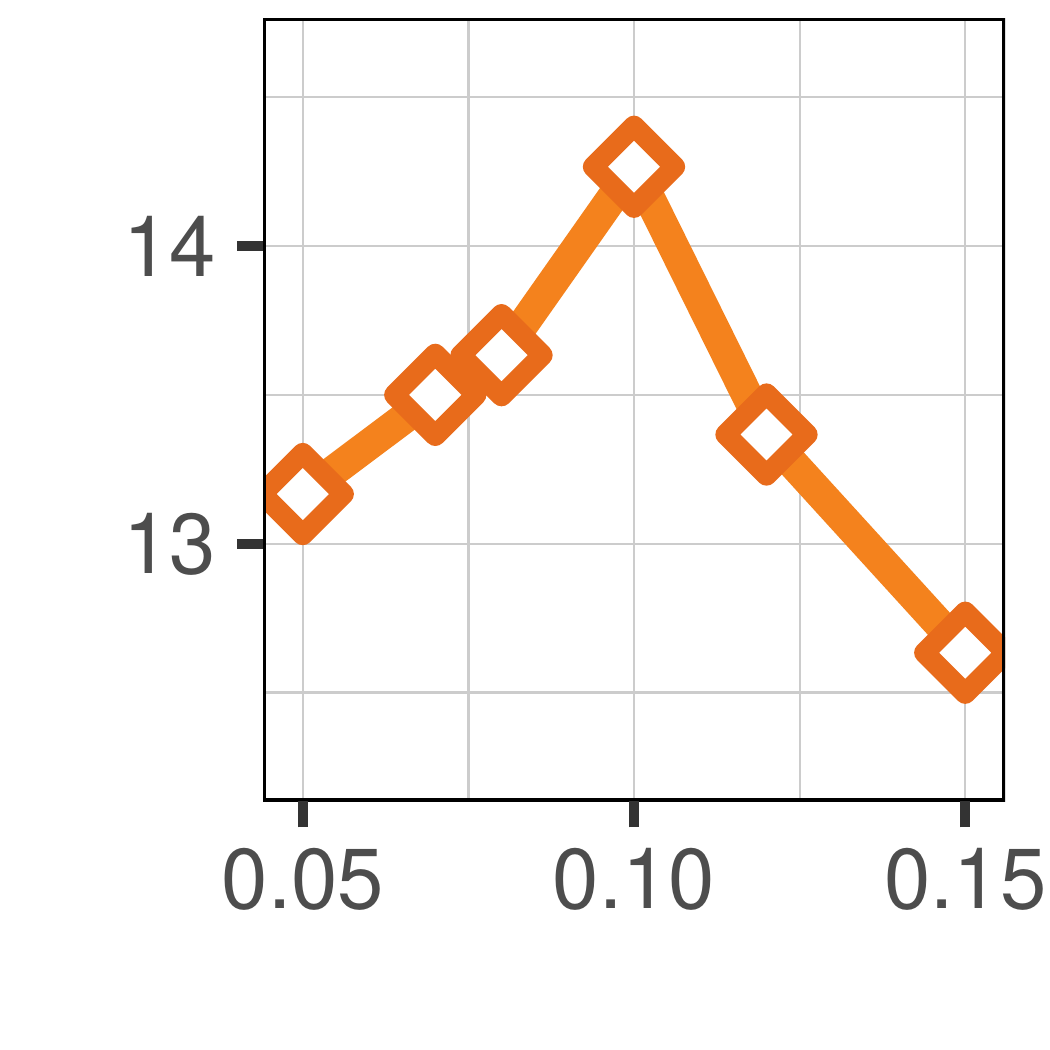}
        \caption{Tail classes}
        \label{fig:ablt-rtail-tail}
    \end{subfigure}
    
    \caption{\textbf{Test accuracy of different classes on varying $\rtail$}. The horizontal axis represents the values of $\rtail$, and the vertical axis represents Accuracy (\%).}
    \label{fig:ablt-rtail}
\end{figure}

\begin{figure}[htbp!]
    \centering
    \begin{subfigure}{0.23\textwidth}
        \centering
        \includegraphics[width=\textwidth]{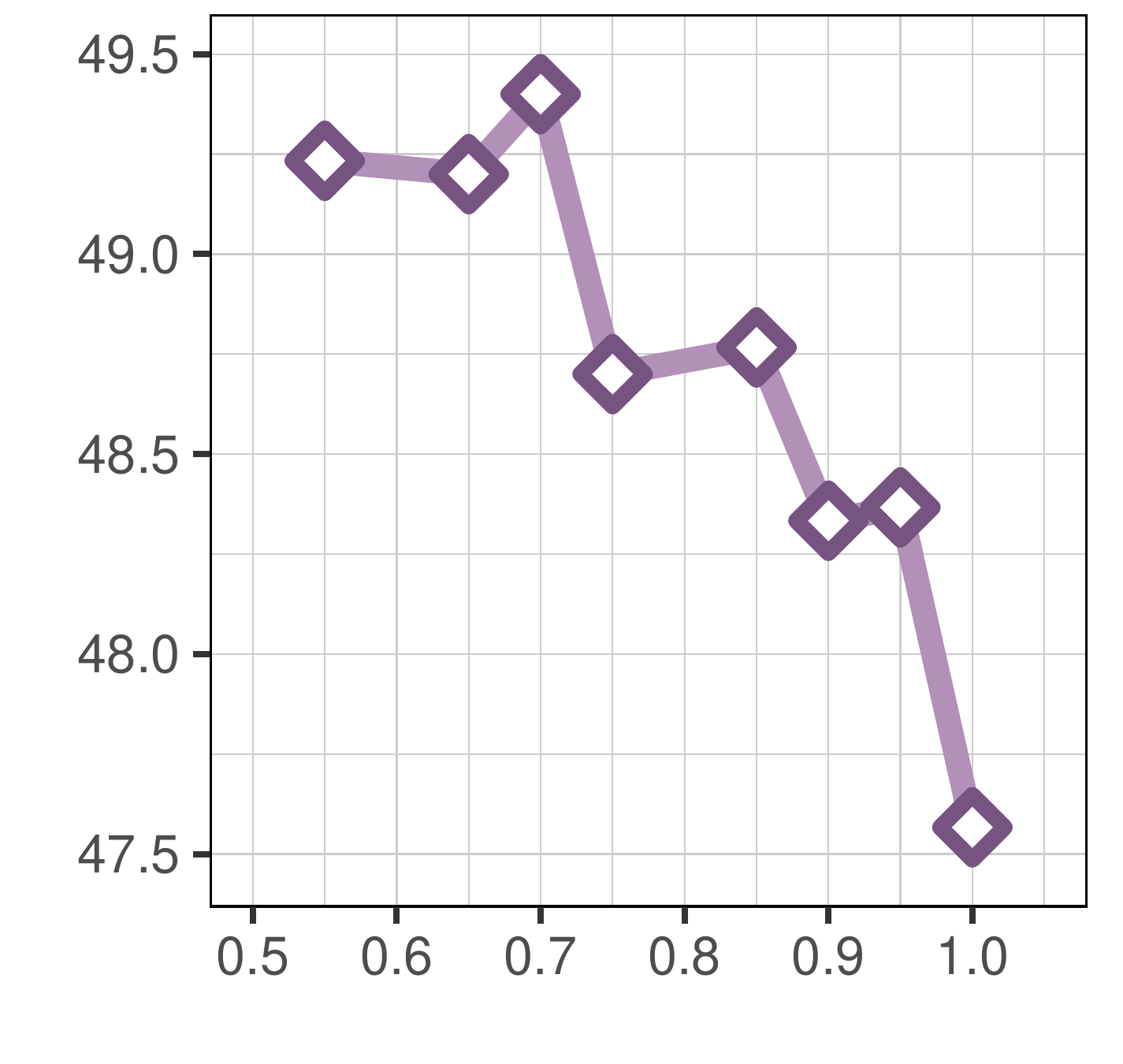}
        \caption{Overall}
        \label{fig:ablt-gamma-overall}
    \end{subfigure}
    \hfill
    \begin{subfigure}{0.23\textwidth}
        \centering
        \includegraphics[width=\textwidth]{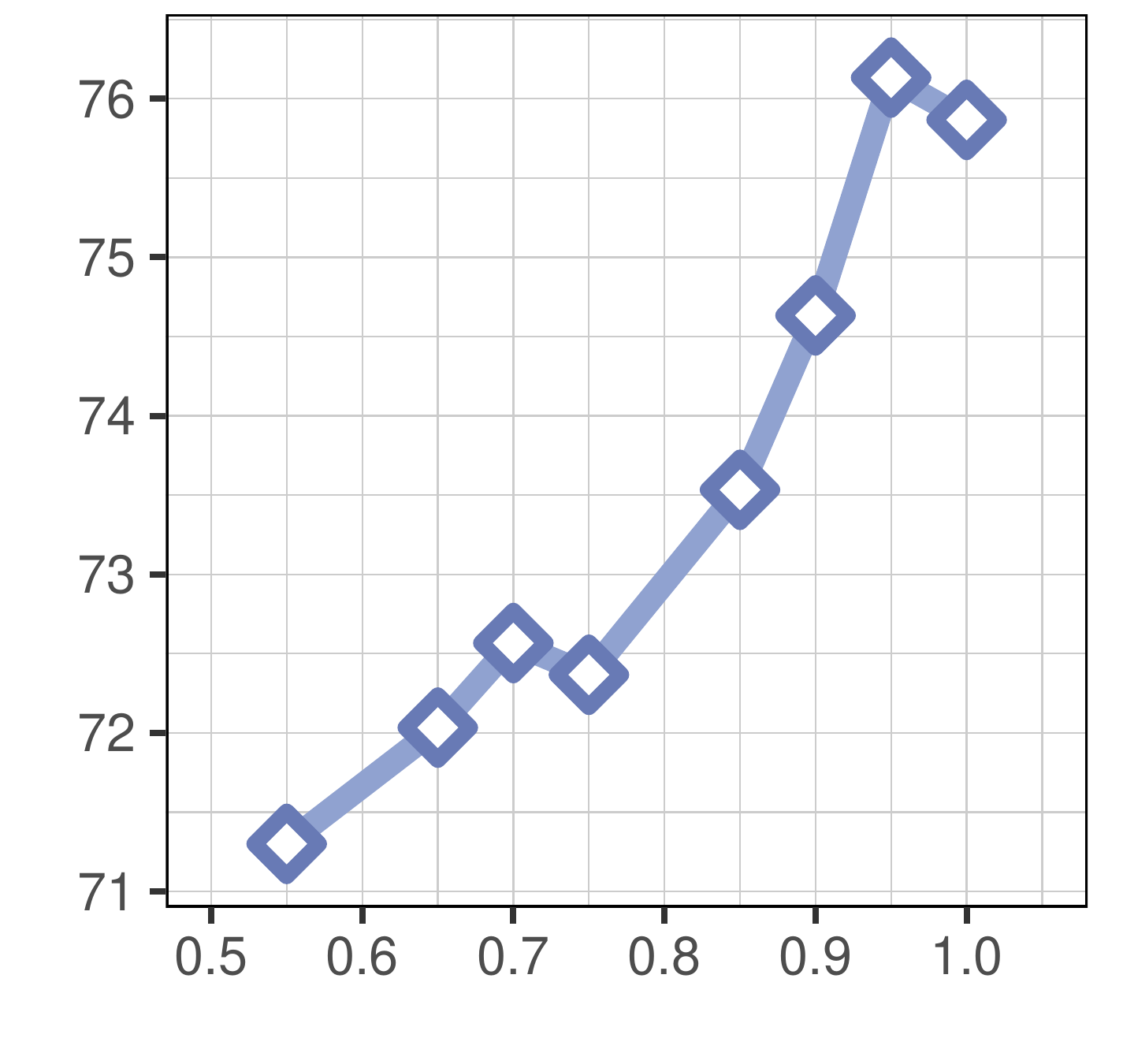}
        \caption{Head classes}
        \label{fig:ablt-gamma-head}
    \end{subfigure}
    \hfill
    \begin{subfigure}{0.23\textwidth}
        \centering
        \includegraphics[width=\textwidth]{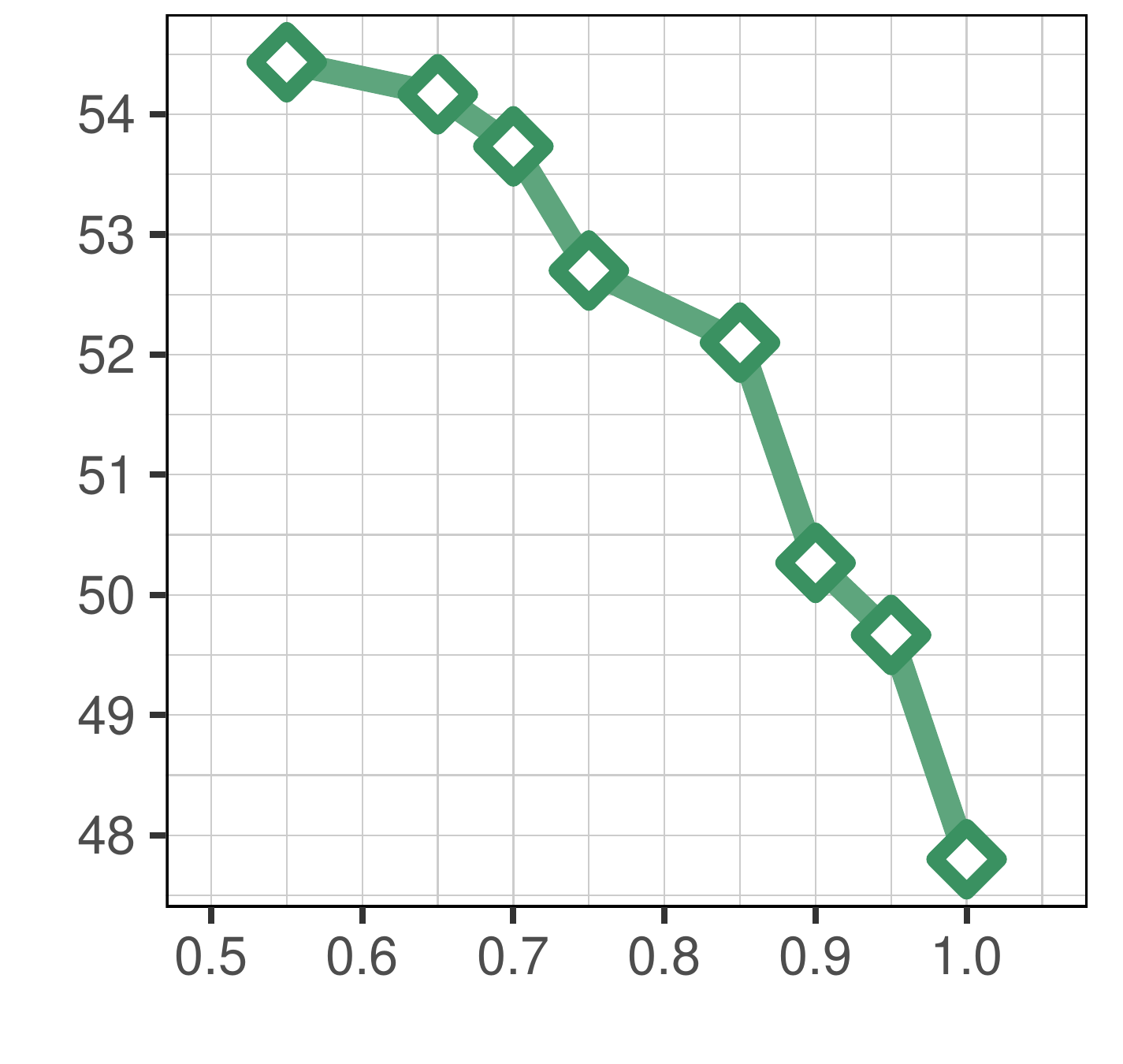}
        \caption{Medium classes}
        \label{fig:ablt-gamma-med}
    \end{subfigure}
    \hfill
    \begin{subfigure}{0.23\textwidth}
        \centering
        \includegraphics[width=\textwidth]{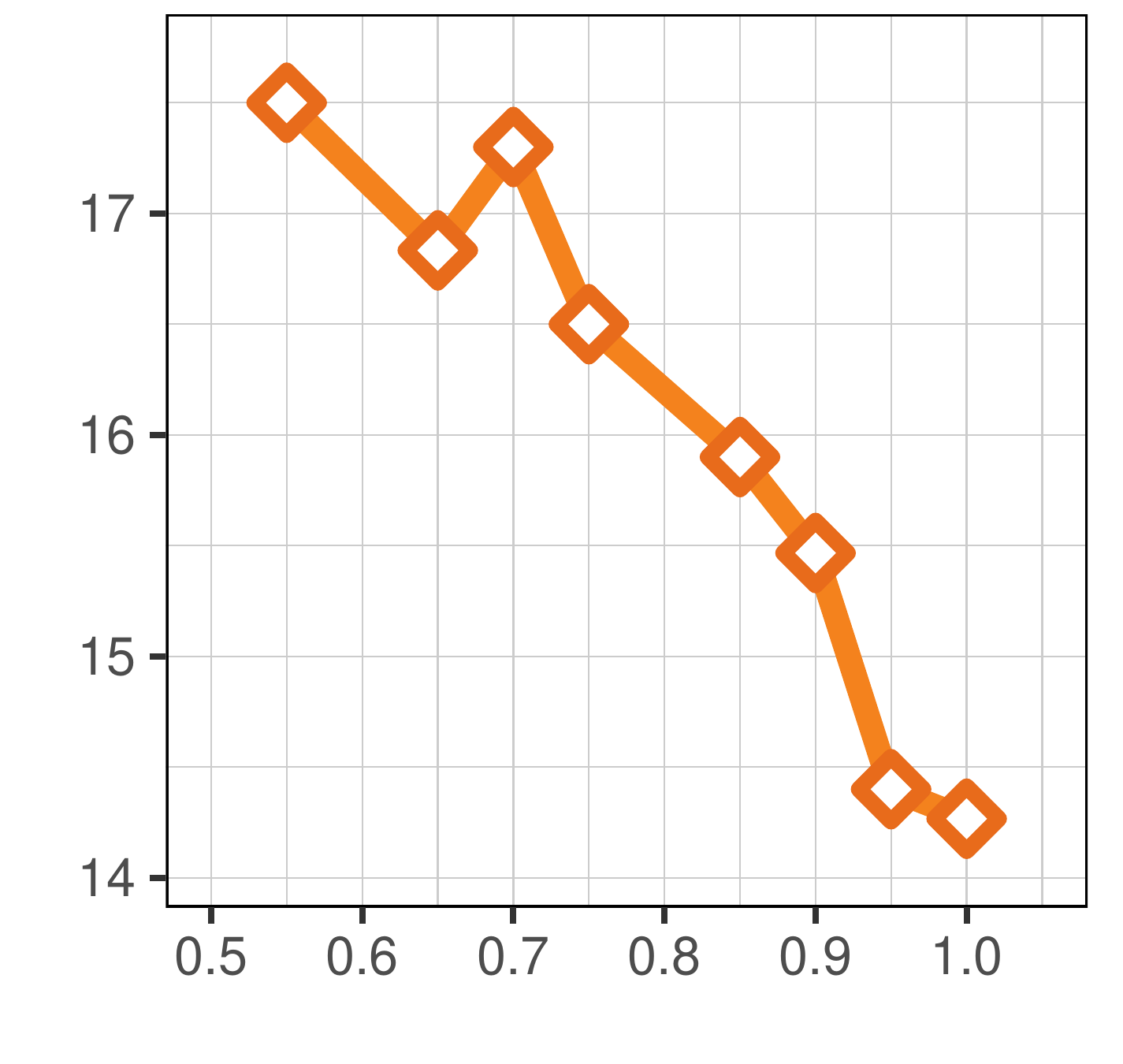}
        \caption{Tail classes}
        \label{fig:ablt-gamma-tail}
    \end{subfigure}
    
    \caption{\textbf{Test accuracy of different classes on varying $\gamma$}. The horizontal axis represents the values of $\gamma$, and the vertical axis represents Accuracy (\%).}
    \label{fig:ablt-gamma}
\end{figure}

Subsequently, we fine-tune hyperparameter $\gamma$ to see its effect on model performance. Results are shown in Fig.\ref{fig:ablt-gamma}. As we increase $\gamma$, i.e., postponing transitional point, performance on head classes increases. However, model performance on medium and tail classes is greatly hurt, rendering the overall performance decreases. This shows that at final stages, model performance is highly influenced by tail classes.

% -------------------------------------------------------------------

\section{Broader Impacts}
\label{app:broader-impacts}

In our study, we analyze the effectiveness of SAM and ImbSAM in escaping saddle points and introduce the SSE-SAM algorithm, which exhibits enhanced performance in long-tail learning scenarios. The application of such long-tail learning algorithms can facilitate the development of more inclusive and representative AI systems. This is particularly crucial in fields such as medical diagnosis, where underrepresented classes often correspond to rare diseases that conventional models frequently overlook. By ensuring these tail classes receive adequate attention, our method has the potential to save lives by improving the recognition of lesser-known medical conditions.

% -------------------------------------------------------------------

\section{Acknowledgements}

This work was supported in part by the National Key R\&D Program of China under Grant 2018AAA0102000, in part by National Natural Science Foundation of China: 62236008, U21B2038, U23B2051, 61931008, 62122075, 62206264, and 92370102, in part by Youth Innovation Promotion Association CAS, in part by the Strategic Priority Research Program of the Chinese Academy of Sciences, Grant No. XDB0680000, in part by the Innovation Funding of ICT, CAS under Grant No.E000000, in part by the Tencent Marketing Solution Rhino-Bird Focused Research Program.

\bibliography{aaai25}

\end{document}